\documentclass{article}

\PassOptionsToPackage{numbers, compress}{natbib}

\usepackage[final]{neurips_2024}

\usepackage[utf8]{inputenc} 
\usepackage[T1]{fontenc}    
\usepackage{hyperref}       
\usepackage{nccmath}
\hypersetup{
     colorlinks=true,
     linkcolor=blue,
     filecolor=blue,
     citecolor=purple,
     urlcolor=purple,
     }
\usepackage{url}
\usepackage{wrapfig}
\usepackage{booktabs}       
\usepackage{amsfonts}       
\usepackage{nicefrac}       
\usepackage{microtype}      
\usepackage{xcolor}         
\usepackage{multicol}
\usepackage{graphicx}
\usepackage{amsthm}
\usepackage{bbm}
\usepackage{comment}
\usepackage{enumitem}
\usepackage{bm}
\usepackage{booktabs}
\usepackage{subfigure}
\usepackage{mathtools}
\usepackage[export]{adjustbox}
\usepackage{stmaryrd}
\usepackage{booktabs} 
\usepackage{xifthen}
\usepackage{xargs}
\usepackage{mathtools}
\usepackage{amsmath,amssymb}
\usepackage{algorithm,algorithmic}
\usepackage{hyperref}
\usepackage{caption}
\usepackage{cleveref}
\usepackage{titlesec}
\usepackage{multirow}
\titleformat*{\subparagraph}{\itshape}
\newtheorem{thm}{Theorem}[section]

\newtheorem{proposition}[thm]{Proposition}

\numberwithin{equation}{section}

\usepackage{amsmath}
\usepackage{amssymb}
\usepackage{mathtools}
\usepackage{amsthm}
\newcommand\blfootnote[1]{%
  \begingroup
  \renewcommand\thefootnote{}\footnote{#1}%
  \addtocounter{footnote}{-1}%
  \endgroup
}

\usepackage[textsize=tiny]{todonotes}

\def\algo{{\sc{DCPS}}}
\def\algoname{{\sc{Divide-and-Conquer Posterior Sampler}}}
\def\tauks{L}
\def\PGDM{$\Pi${\sc{GDM}}}
\def\reddiff{{\sc{RedDiff}}}
\def\DPS{{\sc{DPS}}}
\def\DiffPIR{{\sc{DiffPIR}}}
\def\DDNM{{\sc{DDNM}}}
\def\SDA{{\sc{SDA}}}
\def\FPS{{\sc{FPS} }}

\def\ffhq{\texttt{FFHQ}}
\def\imagenet{\texttt{ImageNet}}
\def\eqdef{\vcentcolon=}

\def\mcgdiff{{\sc{MCGDiff}}}
\def\target{\pi}

\newcommandx{\normconst}[1][1=]{\mathcal{Z}^{#1}}

\def\data{Y}

\def\datadistr{p_{\tiny{\mbox{data}}}}
\def\rset{\mathbb{R}}
\def\nset{\mathbb{N}}
\def\rmd{\mathrm{d}}
\def\latent{X}

\def\eqsp{}

\def\pE{\mathbb{E}}
\def\gauss{\mathcal{N}}
\def\N{\mathbb{N}}

\def\param{\theta}

\def\dimobs{{d_\obs}}

\def\1{\mathbbm{1}}

\newcommandx{\generator}[1][1=]{G^{\genparam _{#1}}}
\def\genparam{\varphi}

\newcommandx{\bwreparam}[1][1=]{\mathcal{G}^{\genparam_{#1}}}

\def\1{\mathbbm{1}}

\def\pE{\mathbb{E}}

\def\N{\mathbb{N}}

\def\gauss{\mathcal{N}}

\def\rmd{\mathrm{d}}

\def\rhs{r.h.s.}

\def\zero{0}

\def\stdobs{\sigma_\obs}

\def\last{n}

\def\vmu{\hat{\mu}}
\def\vstd{\hat{\sigma}}
\def\vlogstd{\hat{\upsilon}}

\newcommand{\intset}[2]{\llbracket #1, #2 \rrbracket}
\newcommand{\lklhd}[2]{\ifthenelse{\equal{#2}{}}{g^\obs _{#1}}{g^\obs _{#1}(#2)}}

\newcommand{\kldivergence}[2]{\mathsf{KL}(#1 \parallel #2)}

\newcommand{\Qdistr}[4]{\ifthenelse{\equal{#4}{}}{q^{#1} _{#2}}{q^{#1} _{#2}(#4 |#3)}}
\newcommand{\mudistr}[4]{\ifthenelse{\equal{#4}{}}{\mu^{#1} _{#2}}{\mu^{#1} _{#2}(#4 |#3)}}
\newcommand{\Qpost}[3]{\ifthenelse{\equal{#3}{}}{q^{#1} _{\latent|\data}}{q^{#1} _{\latent|\data}(#3 |#2)}}

\newcommandx{\predx}[2][2=]{\hat{x}^{#2} _{0|#1}}
\newcommandx{\predxs}[3][2=0,3=\param]{\hat{x}^{#3^{\star}} _{#2|#1}}
\newcommandx{\prednoise}[2][2=\param]{\smash{\hat{\epsilon}^{#2} _{#1}}}

\newcommand{\transition}[1]{\ifthenelse{\equal{#1}{0}}{M_0}{M_{#1}}}

\newcommand{\asymptvarest}[2]{\ifthenelse{\equal{#2}{}}{\ensuremath{\Gamma _{0: #1 \mid #1}}}{\ensuremath{\Gamma^{#2}_{0: #1 \mid #1}}}}
\newcommand{\asymptvarffbs}[2]{\ifthenelse{\equal{#2}{}}{\ensuremath{\Gamma^{\textrm{FFBS}} _{0: #1 \mid #1}}}{\ensuremath{\Gamma^{#2, \textrm{FFBS}}_{0: #1 \mid #1}}}}
\newcommand{\transitiondens}[1]{\ifthenelse{\equal{#1}{0}}{m_0}{m_{#1}}}

\newcommand{\bounded}[1]{\ifthenelse{\equal{#1}{}}{\ensuremath{\mathsf{F} (\mathcal{X})}}{\ensuremath{\mathsf{F} (\mathcal{X}^{\otimes {#1}})}}}
\newcommand{\measurable}[1]{\ifthenelse{\equal{#1}{}}{\ensuremath{\mathsf{F}(\mathcal{X})}}{\ensuremath{\mathsf{F}(\mathcal{X}^{\otimes {#1}})}}}

\newcommand{\As}[2]{\ifthenelse{\equal{#2}{}}{(\mathbf{A}{\mathbf{\ref*{#1})}}}{(\mathbf{A}\mathbf{\ref*{#1}:\ref*{#2}})}}

\newcommand{\muDDIM}[1]{\mu_{#1}}

\def\param{\theta}
\def\gauss{\Normal}

\def\rmd{\operatorname{d}\hspace{-2pt}}
\def\Id{\operatorname{Id}}

\def\rset{\mathbb{R}}

\def\nset{\mathbb{N}}

\newcommandx{\Exp}[2][1=]{\ensuremath{\pE_{#1}\left[ #2\right]}}

\newcommandx{\marginal}[2][1=]{\ensuremath{\xi^{#2}_{#1}}}

\newcommandx{\margindensm}[2]{m_{#1}^{#2}}
\newcommandx{\margindensmu}[4]{m_{#1}^{#2}(#3|#4)}

\newcommandx{\margindens}[4][4=]{\ifthenelse{\equal{#3}{}}{q_{#1}^{#4}(#2)}{q_{#1,#3}^{#4}(#2)}}
\newcommandx{\margindensw}[3][3=]{\ifthenelse{\equal{#3}{}}{q_{#1}^{#2}}{q_{#1,#3}^{#2}}}

\def\rmd{\mathrm{d}}

\def\eqsp{\,}

\def\fwdtransfo{T}
\newcommandx{\fwdtransfoparam}[2]{\fwdtransfo_{#1,#2}}

\def\eg{\text{e.g.}}

\def\wrt{w.r.t.}

\newcommandx{\Vnorm}[2][1=V]{\| #2 \|_{#1}}
\newcommandx{\VnormEq}[2][1=V]{\ensuremath{\left\| #2 \right\|_{#1}}}

\newcommandx\probaMarkovTilde[2][2=]
{\ifthenelse{\equal{#2}{}}{{\widetilde{\mathbb{P}}_{#1}}}{\widetilde{\mathbb{P}}_{#1}\left[ #2\right]}}

\def\eqsp{\;}

\renewcommand{\iint}[2]{\llbracket #1, #2 \rrbracket}

\newcommandx\sequence[3][2=,3=]
{\ifthenelse{\equal{#3}{}}{\ensuremath{\{ #1_{#2}\}}}{\ensuremath{\{ #1_{#2}, \eqsp #2 \in #3 \}}}}
\newcommandx\sequenceD[3][2=,3=]
{\ifthenelse{\equal{#3}{}}{\ensuremath{\{ #1_{#2}\}}}{\ensuremath{( #1)_{ #2 \in #3} }}}

\newcommandx{\sequencen}[2][2=n\in\N]{\ensuremath{\{ #1_n, \eqsp #2 \}}}
\newcommandx\sequenceDouble[4][3=,4=]
{\ifthenelse{\equal{#3}{}}{\ensuremath{\{ (#1_{#3},#2_{#3}) \}}}{\ensuremath{\{  (#1_{#3},#2_{#3}), \eqsp #3 \in #4 \}}}}
\newcommandx{\sequencenDouble}[3][3=n\in\N]{\ensuremath{\{ (#1_{n},#2_{n}), \eqsp #3 \}}}

\def\iid{i.i.d.}

\newcommand{\opnorm}[1]{{\left\vert\kern-0.25ex\left\vert\kern-0.25ex\left\vert #1
    \right\vert\kern-0.25ex\right\vert\kern-0.25ex\right\vert}}

\def\Id{\operatorname{Id}}

\newcommandx{\CPE}[3][1=]{{\mathbb E}_{#1}\left[\left. #2 \middle \vert #3 \right. \right]} 
\newcommandx{\CPVar}[3][1=]{\mathrm{Var}^{#3}_{#1}\left\{ #2 \right\}}
\newcommand{\CPP}[3][]
{\ifthenelse{\equal{#1}{}}{{\mathbb P}\left(\left. #2 \, \right| #3 \right)}{{\mathbb P}_{#1}\left(\left. #2 \, \right | #3 \right)}}

\newcommandx{\osc}[2][1=]{\mathrm{osc}_{#1}(#2)}

\def\Id{\operatorname{Id}}

\def\target{\pi}

\newcommand\coupling[2]{\Gamma(\mu,\nu)}

\def\rset{\mathbb{R}}

\def\nset{\mathbb{N}}

\def\rmd{\mathrm{d}}

\def\rme{\mathrm{e}}

\newcommandx{\estpiNaive}[3][1=g,2=N,3=f]{\hat{\pi}_{#1}^{#2}(#3)}

\newcommandx{\norm}[2][1=]{\ifthenelse{\equal{#1}{}}{\left\Vert #2 \right\Vert}{\left\Vert #2 \right\Vert^{#1}}}
\newcommandx{\normLigne}[2][1=]{\ifthenelse{\equal{#1}{}}{\Vert #2 \Vert}{\Vert #2 \Vert^{#1}}}

\def\txts{\textstyle}

\def\target{\pi}

\def\zero{\mathbf{0}}
\def\zero{0}

\def\N{\mathbb{N}}
\def\rmd{\mathrm{d}}

\def\param{\theta}

\def\pE{\mathbb{E}}

\def\gauss{\mathcal{N}}

\def\1{\mathbbm{1}}

\def\pow{\varepsilon}

\newcommand{\esssup}[1]{\mathrm{ess}\sup_{\lambda}}

\newcommand{\sqmapping}[2]{
        \ifthenelse{\equal{#1}{}}
        {
            \mathcal{R}^{\pow,2} _{ #2}
            }
        {
            \mathcal{R}^{\pow_{#1}, 2} _ #2
            }}
\newcommand{\mapping}[2]{
        \ifthenelse{\equal{#1}{}}
        {
            \Psi^{\pow} _{ #2}
            }
        {
            \Psi^{\pow_{#1}} _{#2}
            }}
\newcommandx{\denoiser}[2][2=\param]{D^{#2} _{#1}}
\newcommandx{\score}[2][2=]{\hat{s}^{#2} _{#1}}
\newcommandx{\scores}[2][2=\param]{s^{#2^{\star}} _{#1}}

\newcommand{\piterate}[2]{\ifthenelse{\equal{#2}{}}{\hat\mu_{\param, #1}}{\hat\mu_{\param#2, #1}}}

\newcommand{\A}[2]{\ifthenelse{\equal{#2}{}}{(\textbf{A}{\textbf{\ref{#1})}}}{(\textbf{A}\textbf{\ref{#1}-\ref{#2}})}}

\def\obs{y}

\def\rmd{\mathrm{d}}

\def\pE{\mathbb{E}}

\def\eqsp{\,}
\def\1{\mathbbm{1}}
\def\dimx{{d_{x}}}

\def\rset{\mathbb{R}}
\def\nset{\mathbb{N}}

\newcommand{\m}[1]{\operatorname{#1}}

\def\gauss{\mathcal{N}}

\def\param{\theta}

\def\normpdf{\mathrm{N}}

\newcommandx{\acp}[2][2=]{\ifthenelse{\equal{#2}{}}{\alpha_{#1}}{\alpha_{#1:#2}}}

\newcommandx{\xpred}[2][2=]{\boldsymbol{\chi}^{#2} _{#1}}
\newcommandx{\bottomxpred}[2][2=]{\underline{\boldsymbol{\chi}}^{#2} _{#1}}
\newcommandx{\topxpred}[2][2=]{\overline{\boldsymbol{\chi}}^{#2} _{#1}}

\newcommand{\epsprop}[2]{\ifthenelse{\equal{#1}{}}{\mathbf{e}^{#2}}{\mathbf{e}^{#2} (#1)}}

\newcommand{\realrevbridge}[3]{\ifthenelse{\equal{#2}{}}{\mathsf{q} _{#1}}{\mathsf{q} _{#1}(#3|#2)}}
\newcommand{\revbridge}[3]{\ifthenelse{\equal{#2}{}}{\smash{{q}}^{\sigma} _{#1}}{\smash{{q}}^{\sigma} _{#1}(#3|#2)}}
\newcommand{\idxrevbridge}[4]{\ifthenelse{\equal{#2}{}}{\smash{{q}}^{\sigma, #4} _{#1}}{\smash{{q}}^{\sigma, #4} _{#1}(#3|#2)}}
\newcommand{\trevbridge}[3]{\ifthenelse{\equal{#2}{}}{\smash{\overline{q}}^{\sigma} _{#1}}{\smash{\overline{q}}^{\sigma} _{#1}(#3|#2)}}
\newcommand{\brevbridge}[3]{\ifthenelse{\equal{#2}{}}{\underline{q}^{\sigma} _{#1}}{\underline{q}^{\sigma} _{#1}(#3|#2)}}
\newcommand{\trealrevbridge}[3]{\ifthenelse{\equal{#2}{}}{\overline{\mathsf{q}}_{#1}}{\overline{\mathsf{q}} _{#1}(#3|#2)}}

\newcommand{\bridge}[4]{{\ifthenelse{\equal{#3}{}}{p^{#2} _{#1}}{p^{#2} _{#1}(#3, #4)}}}
\newcommand{\outrans}[3]{\ifthenelse{\equal{#2}{}}{\overrightarrow{m}_{#1}}{\overrightarrow{m}_{#1}(#2, #3)}}
\newcommandx{\bwtrans}[4][4=]{\ifthenelse{\equal{#3}{}}{q^{#4} _{#1}}{q^{#4} _{#1}(#3|#2)}}
\newcommandx{\rbwker}[4][4=]{\ifthenelse{\equal{#3}{}}{\smash{q}^{#4} _{#1}}{\smash{q}^{#4} _{#1}(#3|#2)}}
\newcommandx{\pbwker}[4][4=]{\ifthenelse{\equal{#3}{}}{\hat{p}^{#4} _{#1}}{\hat{p}^{#4} _{#1}(#3|#2)}}
\newcommandx{\abwker}[4][4=]{\ifthenelse{\equal{#3}{}}{\smash{p}^{#4} _{#1}}{\smash{p}^{#4} _{#1}(#3|#2)}}
\newcommandx{\pibwker}[4][4=]{\ifthenelse{\equal{#3}{}}{\smash{\pi}^{#4} _{#1}}{\smash{\pi}^{#4} _{#1}(#3|#2)}}
\newcommandx{\hbwker}[4][4=]{\ifthenelse{\equal{#3}{}}{\smash{\hat{p}}^{#4} _{#1}}{\smash{\hat{p}}^{#4} _{#1}(#3|#2)}}
\newcommandx{\cabwker}[4][4=]{\ifthenelse{\equal{#3}{}}{\smash{p}^{#4} _{#1}}{\smash{p}^{#4} _{#1}(#3|#2)}}
\newcommandx{\aabwker}[4][4=]{\ifthenelse{\equal{#3}{}}{\smash{\tilde{p}}^{#4} _{#1}}{\smash{\tilde{p}}^{#4} _{#1}(#3|#2)}}
\newcommandx{\vbwker}[4][4=\vparam]{\ifthenelse{\equal{#3}{}}{\lambda^{#4} _{#1}}{\lambda^{#4} _{#1}(#3|#2)}}
\newcommandx{\cbwtrans}[4][4=]{\ifthenelse{\equal{#3}{}}{\smash{\overset{{}_{\leftarrow}}{p}}^{#4} _{#1}}{\smash{\overset{{}_{\leftarrow}}{p}}^{#4} _{#1}(#3|#2)}}
\newcommandx{\Pbwtrans}[4][4=]{\ifthenelse{\equal{#3}{}}{\smash{\overset{{}_{\leftarrow}}{P}}^{#4} _{#1}}{\smash{\overset{{}_{\leftarrow}}{P}}^{#4} _{#1}(#3|#2)}}
\newcommandx{\condbw}[4][4=]{\ifthenelse{\equal{#3}{}}{\smash{\overset{{}_{\leftarrow}}{q}}^{#4} _{#1}}{\smash{\overset{{}_{\leftarrow}}{q}}^{#4} _{#1}(#3|#2)}}

\newcommandx{\qbwtrans}[4][4=]{\ifthenelse{\equal{#3}{}}{\smash{\overset{{}_{\leftarrow}}{q}}^{#4} _{#1}}{\smash{\overset{{}_{\leftarrow}}{q}}^{#4} _{#1}(#3|#2)}}

\newcommandx{\fwtrans}[4][4=]{\ifthenelse{\equal{#2}{}}{p_{#1}}{q _{#1}(#3|#2)}}
\newcommand{\idxfwtrans}[4]{\ifthenelse{\equal{#2}{}}{m^{#4} _{#1}}{m^{#4} _{#1}(#3|#2)}}
\newcommand{\bfwtrans}[3]{\ifthenelse{\equal{#2}{}}{\smash{\underline{m}} _{#1}}{\smash{\underline{m}} _{#1}(#3|#2)}}
\newcommand{\tfwtrans}[3]{\ifthenelse{\equal{#2}{}}{\smash{\overline{m}} _{#1}}{\smash{\overline{m}} _{#1}(#3|#2)}}
\newcommand{\idxtfwtrans}[4]{\ifthenelse{\equal{#2}{}}{\smash{\overline{m}}^{\sigma_\delta, #4} _{#1}}{\smash{\overline{m}}^{\sigma_\delta, #4} _{#1}(#3|#2)}}
\newcommand{\idxbwtrans}[4]{\ifthenelse{\equal{#2}{}}{p^{#4} _{#1}}{p^{#4} _{#1}(#3|#2)}}
\newcommand{\idxtbwtrans}[4]{\ifthenelse{\equal{#2}{}}{\smash{\overline{p}}^{#4} _{#1}}{\smash{\overline{p}}^{#4} _{#1}(#3|#2)}}
\newcommandx{\cbwmarg}[3][3=]{\ifthenelse{\equal{#2}{}}{\pi^{#3} _{#1}}{\pi^{#3} _{#1}(#2)}}
\newcommandx{\bwmarg}[3][3=]{\ifthenelse{\equal{#2}{}}{p^{#3} _{#1}}{p^{#3}_{#1}(#2)}}
\newcommandx{\ibwmarg}[3]{\ifthenelse{\equal{#2}{}}{p^{#3} _{#1}}{p^{#3}_{#1}(#2)}}
\newcommandx{\iabwmarg}[3][3=]{\ifthenelse{\equal{#2}{}}{\tilde{p}^{#3} _{#1}}{\tilde{p}^{#3}_{#1}(#2)}}
\newcommand{\bwoptmarg}[3]{\ifthenelse{\equal{#3}{}}{\smash{\overset{{}_{\leftarrow}}{{\mathfrak{m}}}}^{#2} _{#1}}{\smash{\overset{{}_{\leftarrow}}{{\mathfrak{m}}}}^{#2} _{#1}(#3)}}
\newcommandx{\fwmarg}[3][3=]{\ifthenelse{\equal{#2}{}}{q^{#3} _{#1}}{q^{#3} _{#1}(#2)}}
\newcommandx{\cfwmarg}[3][3=]{\ifthenelse{\equal{#2}{}}{q^{#3} _{#1}}{q^{#3} _{#1}(#2)}}
\newcommand{\posterior}[2]{\ifthenelse{\equal{#2}{}}{p^\obs _{#1}}{p^\\stdobss _{#1}(#2)}}
\newcommand{\initdistr}[3]{\ifthenelse{\equal{#2}{}}{\chi_{#1}}{\chi_{#1}(#2, #3)}}
\newcommand{\bwopt}[4]{\ifthenelse{\equal{#3}{}}{p ^{#2} _{#1}}{p^{#2} _{#1}(#4|#3)}}
\newcommand{\bbwopt}[4]{\ifthenelse{\equal{#3}{}}{p ^{#2} _{#1}}{\underline{p}^{#2} _{#1}(#4|#3)}}
\newcommand{\tbwopt}[4]{\ifthenelse{\equal{#3}{}}{p ^{#2} _{#1}}{\overline{p}^{#2} _{#1}(#4|#3)}}
\newcommand{\fwopt}[4]{\ifthenelse{\equal{#3}{}}{\overset{{}_{\rightarrow}}{m} ^{#2} _{#1}}{\smash{\overset{{}_{\rightarrow}}{m}}^{#2} _{#1}(#4|#3)}}
\newcommand{\ufilter}[3]{{\ifthenelse{\equal{#3}{}}{\eta^{#2} _{#1}}{\eta^{#2} _{#1}(#3)}}}
\newcommand{\filter}[3]{{\ifthenelse{\equal{#3}{}}{\phi^{#2} _{#1}}{\phi^{#2} _{#1}(#3)}}}
\newcommand{\noisyfilter}[3]{{\ifthenelse{\equal{#3}{}}{\phi^{#2} _{#1}}{\phi^{#2} _{#1}(#3)}}}
\newcommand{\noiselessfilter}[3]{{\ifthenelse{\equal{#3}{}}{\phi^{#2} _{#1}}{\phi _{#1}(#3 | #2)}}}
\newcommand{\bfilter}[3]{{\ifthenelse{\equal{#3}{}}{\Phi^{#2} _{#1}}{\Phi^{#2} _{#1}(#3)}}}

\newcommand{\pot}[3]{\ifthenelse{\equal{#3}{}}{\smash{g^{#2} _{#1}}}{\smash{g^{#2} _{#1}}(#3)}}
\newcommand{\ipot}[3]{\ifthenelse{\equal{#3}{}}{\bar{g}^{#2} _{#1}}{\bar{g}^{#2} _{#1}(#3)}}
\newcommand{\iapot}[3]{\ifthenelse{\equal{#3}{}}{\hat{g}^{#2} _{#1}}{\hat{g}^{#2} _{#1}(#3)}}
\newcommand{\dubpot}[3]{\ifthenelse{\equal{#3}{}}{G^{#2} _{#1}}{G^{#2} _{#1}(#3)}}

\newcommand{\noisytrans}[3]{\ifthenelse{\equal{#2}{}}{\mathsf{P}_{#1}}{\mathsf{P}_{#1}(#3|#2)}}
\def\ddrm{DDRM}

\def\Id{I}

\newcommand\projidx[3]{\ifthenelse{\equal{#2}{}}{#1^{\setminus #3}}{#1^{\setminus #3}_{#2}}}

\newcommand{\tstep}[1]{{k_{#1}}}
\def\vparam{\varphi}

\def\bfA{A}

\renewcommand{\citep}{\cite}
\renewcommand{\citet}{\cite}

\title{Divide-and-Conquer Posterior Sampling for Denoising Diffusion Priors}

\author{Yazid Janati$^{*, 1}$
\quad
Badr Moufad$^{*, 1}$\\
\textbf{
Alain Durmus$^1$ \quad
Eric Moulines$^{1, 3}$
\quad
Jimmy Olsson$^2$} \\
$^1$ CMAP, Ecole polytechnique \quad $^2$ KTH Royal Institute of Technology \quad $^3$ MBZUAI}

\begin{document}
\maketitle

\begin{abstract}
Recent advancements in solving Bayesian inverse problems have spotlighted denoising diffusion models (DDMs) as effective priors.
Although these have great potential, DDM priors yield complex posterior distributions that are challenging to sample.
Existing approaches to posterior sampling in this context address this problem either by retraining model-specific components, leading to stiff and cumbersome methods, or by introducing approximations with uncontrolled errors that affect the accuracy of the produced samples.
We present an innovative framework, divide-and-conquer posterior sampling, which leverages the inherent structure of DDMs to construct a sequence of intermediate posteriors that guide the produced samples to the target posterior.
Our method significantly reduces the approximation error associated with current techniques without the need for retraining.
We demonstrate the versatility and effectiveness of our approach for a wide range of Bayesian inverse problems.
The code is available at \url{https://github.com/Badr-MOUFAD/dcps}
\end{abstract}

\blfootnote{* Equal contribution}
\blfootnote{Corresponding authors: \texttt{$\{$yazid.janati,badr.moufad$\}$@polytechnique.edu}}
 \section{Introduction}
 Many problems in machine learning can be formulated as inverse problems, such as superresolution, deblurring, and inpainting, to name but a few.
They all have the same goal, namely to recover a signal of interest from an indirect observation.
One line of research addresses these problems through the lens of the Bayesian framework by specifying two components: a prior distribution, which embodies the specification of the signal, and a likelihood that describes the law of the observation conditionally on the signal.
Once these elements are specified, the inverse problem is solved by sampling from the posterior distribution, which, after including the observation, contains all available information about the signal and thus about its uncertainty as well
\citep{dashti2017bayesian}.
The importance of the specification of the prior in solving Bayesian ill-posed inverse problems is paramount.
In the last decade, the success of priors based on deep generative models has fundamentally changed the field of linear inverse problems \citep{romano2017little,ulyanov2018deep,guo2019agem,pan2021exploiting,kawar2022denoising}.
Recently, denoising diffusion probabilistic models (DDMs) have received special attention.
Thanks to their ability to learn complex and multimodal data distributions, DDM represent the state-of-the-art in many generative modeling tasks, \emph{e.g.} image generation \citep{sohl2015deep,ho2020denoising,song2019generative,song2021score,dhariwal2021diffusion,song2021denoising,song2021maximum}, super-resolution \citep{saharia2022image,batzolis2021conditional}, and inpainting \citep{sohl2015deep,chung2022come,jing2022subspace}.

Popular methods to sample from posterior distribution include Markov chain Monte Carlo (MCMC) and variational inference; see \citep{stuart2010inverse,calvetti2018inverse} and the references therein.
These methods are iterative schemes that require an explicit procedure to evaluate pointwise the prior distribution and often its (Stein) score function \citep{hyvarinen2007some} in order to compute acceptance ratios and construct efficient proposals.
While sampling from the DDM priors is straightforward, posterior sampling is usually challenging since the intractability of the posterior density and its score make them computationally prohibitive and thus invalidate all conventional simulation methods.
Although approximations exist, their associated iterative sampling schemes can be computationally intensive and exhibit high sensitivity to the choice of hyperparameters; see \emph{e.g.} \citep{kawar2022denoising}.
This paper proposes the \algoname\ (\algo), a novel approach to posterior sampling in Bayesian inverse problems with DDM priors. Thanks to the Markov property of the data-generating backward diffusion, the posterior can be expressed as the marginal distribution of a Feynman--Kac (FK) path measure \citep{del2004feynman}, whose length corresponds to the number of diffusion steps and whose user-defined potentials serve to bias the dynamics of the data-generating backward diffusion to align with the likelihood of the observation.
Besides, for a given choice of potentials, the FK path law becomes Markovian, making it possible to express the posterior as the marginal of a time-reversed inhomogeneous Markov chain.

This approach is tempting, yet, the backward Markov decomposition remains difficult to apply in practice as these specific potential functions are difficult to approximate, especially when the number of diffusion steps is large.
We tackle this problem with a divide-and-conquer approach.
More precisely, instead of targeting the given posterior by a single simulation run through the full backward decomposition, our proposed scheme targets backward a sequence $(\cbwmarg{\tstep{\ell}}{})_{\ell=0}^L$ of distributions along the path measure leading to the target posterior distribution (\cref{sec:DCPS}).
These distributions are induced by a sequence of increasingly complex potentials and converge to the target distribution.
Starting with a sample from $\cbwmarg{\tstep{\ell+1}}{}$, a draw from $\cbwmarg{\tstep{\ell}}{}$ is formed by a combination of Langevin iterations and the simulation of an inhomogeneous Markov chain.
In other words, $\cbwmarg{\tstep{\ell}}{}$ is expressed as the final marginal distribution of a time-reversed inhomogeneous Markov chain of moderate length $k_{\ell+1}-k_{\ell} \in \nset^*$ with an initial distribution $\cbwmarg{\tstep{\ell+1}}{}[\ell]$.
This chain, whose transition densities are intractable, is approximately sampled using Gaussian variational inference.
The rationale behind our approach stems from the observation that the Gaussian approximation error can be reduced by shortening the length of the intermediate FK path measures (\emph{i.e.}, by increasing $L$); a result that we show in \Cref{prop:w2}.
We finally illustrate that our algorithm can provide high-quality solutions to Bayesian inverse problems involving a variety of datasets and tasks.

To sum up our contribution, we
\begin{itemize}[leftmargin=.25in]
    \item show that the existing approximations of the Markovian backward decomposition can be improved using a bridge-kernel smoothing technique
    \item design a novel divide-and-conquer sampling approach that enables efficient bias-reduced sampling from the posterior, and illustrate its performance on several Bayesian inverse problems including inpainting, outpainting, Poisson imaging, and JPEG dequantization,
    \item propose a new technique to efficiently generate approximate samples from the backward decomposition using Gaussian variational inference.
\end{itemize}

\paragraph{Notation.}
For $(m,n) \in\nset^2$ such that $m < n$, we let $\iint{m}{n} \eqdef \{m,\ldots,n\}$. We use $\normpdf(x; \mu, \Sigma)$ to denote the density at $x$ of a Gaussian distribution with mean $\mu$ and covariance matrix $\Sigma$.
$\Id_d$ is the $d$-dimensional identity matrix and $\delta_a$ denotes the Dirac mass at $a$.
$W_2$ denotes the Wasserstein distance of order 2.
We use uppercase for random variables and lowercase for their realizations.

  \section{Posterior sampling with DDM prior}
 \paragraph{DDM priors.} We provide a brief overview of DDMs \citep{sohl2015deep,song2019generative,ho2020denoising}.
Suppose we can access an empirical sample from some data distribution $\datadistr$ defined on $\rset^{\dimx}$. 
For $n \in \nset$ large enough and $k \in \iint{0}{n}$, define the distribution $\fwmarg{k}{x_k} \eqdef \int \datadistr(x_0) \, \fwtrans{k|0}{x_0}{x_k} \rmd x_0$ with $\fwtrans{k|0}{x_0}{x_k} \eqdef \normpdf(x_k; \sqrt{\acp{k}} x_0, (1 - \acp{k})\Id_\dimx)$, where
$(\acp{k})_{k=0}^n$ is a decreasing sequence with $\alpha_0=1$ and $\alpha_n$ approximately equals zero.
The probability density $\fwmarg{k}{}$ corresponds to the marginal distribution at time $k$ of an auto-regressive process on $\rset^{d_x}$ given by $X_{k + 1} = \sqrt{\acp{k+1} / \acp{k}} X_k + \sqrt{1 - \acp{k + 1} / \acp{k}} \epsilon_{k + 1}$,
with $X_0 \sim \datadistr$ and $(\epsilon_k)_{k=0}^n$ being a sequence of {\iid} $\dimx$-dimensional standard Gaussians.

DDMs leverage parametric approximations $\smash{\predx{k}[\param]}$ of the mappings $\smash{x_k \mapsto \int x_0 \, \bwtrans{0|k}{x_k}{x_0} \rmd x_0}$, where $\bwtrans{0|k}{x_k}{x_0} \propto \datadistr(x_0) \fwtrans{k|0}{x_0}{x_k}$ is the conditional distribution of $X_0$ given $X_k = x_k$.
Each $\predx{k}[\param]$ is defined as $\predx{k}[\param](x_k) \eqdef (x_k - \sqrt{1 - \acp{k}} \prednoise{k}(x_k)) / \sqrt{\acp{k}}$, where $\prednoise{k}$ is a noise predictor network trained by minimizing a denoising objective; see \cite[Eq.~(5)]{song2021denoising} and \Cref{sec:DDIM} for details.
Following \citep[Section 4.2]{dhariwal2021diffusion}, $\prednoise{k}$ also provides an estimate of the score $\nabla \log \fwmarg{k}{x_k}$ given by $\score{k}[\param](x_k) \eqdef -\big(x_k - \sqrt{\acp{k}}\predx{k}[\param](x_k)\big) / (1 - \acp{k})$.
We denote by $\param^\star$ the minimizer of the denoising objective.
Having access to $\param^\star$, we can define a generative model for $\datadistr$ by adopting the denoising diffusion probabilistic model (DDPM) framework of \citep{ho2020denoising}.
As long as $n$ is large enough, $\fwmarg{n}{}$ can be confused with a multivariate standard Gaussian.
Define the \emph{bridge kernel} $\fwtrans{k|0, k+1}{x_0, x_{k+1}}{x_k} \propto \fwtrans{k|0}{x_0}{x_k} \fwtrans{k+1|k}{x_k}{x_{k+1}}$ which is a Gaussian distribution with mean $\muDDIM{k|0, k+1}(x_0, x_{k+1})$ and diagonal covariance $\sigma^2 _{k|k+1} \Id_\dimx$ defined in \Cref{apdx:ddim}.
Define the generative model for $\datadistr$ as
\begin{equation}
    \label{eq:ddim_joint}
    \txts \bwmarg{0:n}{x_{0:n}}[\theta^{\star}] = \bwmarg{n}{x_n} \prod_{k = 0}^{n-1} \abwker{k|k+1}{x_{k+1}}{x_k}[\theta^{\star}] \eqsp,
\end{equation}
where for every $k \in \intset{1}{n-1}$, the backward transitions are
\begin{equation}
    \label{eq:ddpm_kernel}
\abwker{k|k+1}{x_{k+1}}{x_k}[\theta^{\star}] \eqdef \fwtrans{k|0, k+1}{\predxs{k+1}(x_{k+1}), x_{k+1}}{x_k}\eqsp,
\end{equation}
with $\abwker{0|1}{x_1}{\cdot}[\theta^{\star}] \eqdef \delta_{\predxs{1}(x_1)}$ and $\bwmarg{n}{x_n} = \normpdf(x_n ; 0,\Id_{\dimx})$.
In the following, we assume that we have access to a pre-trained DDM and omit the superscript $\theta^{\star}$ from the notation, writing simply $\bwmarg{}{}$ and $\smash{\predx{k}[]}$ when referring to the generative model and the denoiser, respectively.
In addition, we denote by $\smash{\bwmarg{k}{}[]}$ the $k$-th marginal of $\bwmarg{0:n}{}[]$ and write, for all $(\ell, m) \in \intset{0}{n}^2$ such that $\ell < m$, $\abwker{\ell|m}{x_m}{x_\ell} \eqdef \prod_{k = \ell} ^{m-1} \abwker{k|k+1}{x_{k+1}}{x_k}$.

 \label{sec:posterior:sampling}
 \paragraph{Posterior sampling.} Let $\pot{0}{}{}$ be a nonnegative function on $\rset^{\dimx}$. When solving Bayesian inverse problems, $\pot{0}{}{}$ is taken as the likelihood of the signal given the  observation specified using the forward model (see the next section). Our objective is to sample from the posterior distribution
\begin{equation}
    \label{eq:posterior_def}
    \cbwmarg{0}{x_0} \eqdef \pot{0}{}{x_0} \, \bwmarg{0}{x_0} / \normconst \eqsp,
  \end{equation}
where $\normconst \eqdef \int \pot{0}{}{x_0} \, \bwmarg{0}{x_0} \rmd x_0$ is the normalizing constant and 
the prior 
$\bwmarg{0}{}$ is the marginal of \eqref{eq:ddim_joint} w.r.t. $x_0$, in which case the posterior \eqref{eq:posterior_def} can be expressed as
$$
    \cbwmarg{0}{x_0} = \frac{1}{\normconst} \int \pot{0}{}{x_0} \prod_{k = 0} ^{\last-1} \abwker{k|k+1}{x_{k+1}}{x_k} \, \bwmarg{\last}{x_\last} \, \rmd x_{1:n}\eqsp.
$$
Thus, \Cref{eq:posterior_def} can be interpreted as the marginal of a time-reversed FK (Feynman--Kac) model with a non-trivial potential only for $k = 0$; see \cite{del2004feynman} for a comprehensive introduction to FK models.
In this work, we twist, without modifying the law of the FK model, the backward transitions $\smash{\abwker{k|k+1}{}{}}$ by artificial positive potentials  $\smash{(\pot{k}{}{})_{k = 0}^\last}$, each being a function on $\rset^\dimx$, and write 
\begin{equation}
    \label{eq:FK_posterior}
     \cbwmarg{0}{x_0} = \frac{1}{\normconst} \int \pot{\last}{}{x_\last} \, \bwmarg{\last}{x_\last}
     \prod_{k = 0}^{\last-1} \frac{\pot{k}{}{x_k}}{\pot{k+1}{}{x_{k+1}}}\, \abwker{k|k+1}{x_{k+1}}{x_k} \, \rmd x_{1:n} \eqsp.
\end{equation}
This allows the posterior of interest to be expressed as the time-zero marginal of an FK model with initial distribution $\bwmarg{\last}{}$, Markov transition kernels $(\abwker{k|k+1}{x_{x+1}}{})_{k = 0}^{n-1}$, and
$(\pot{k}{}{})_{k=0}^{\last}$.

Recent works that aim to sample from the posterior  \eqref{eq:posterior_def} generally employ the FK representation \eqref{eq:FK_posterior}. 
These studies, however, adopt varying auxiliary potentials \citep{chung2023diffusion, song2022pseudoinverse, zhang2023towards, boys2023tweedie, trippe2023diffusion, wu2023practical}.
FK models can be effectively sampled using sequential Monte Carlo (SMC) methods; see, \emph{\eg}, \citep{del2004feynman,chopin2020introduction}. 
SMC methods sequentially propagate weighted samples, whose associated weighted empirical distributions target the flow of the FK marginal distributions.
The effectiveness of this technique depends heavily on the choice of intermediate potentials $( \pot{k}{}{} )_{k = 1} ^n$, as discussed in \citep{trippe2023diffusion, wu2023practical, cardoso2023monte, dou2024diffusion}.
However, SMC methods require a number of samples proportional and often exponential in the dimensionality of the problems hence limiting their application in these setups due to the resulting probabitive memory cost \citep{bickel:li:bengtsson:2008}.
On the other hand, reducing the number of samples makes them vulnerable to mode collapse.


In the following, we will focus on a particular choice of potential functions $(\pot{k}{}{})_{k = 1} ^n$ for which the posterior $\pi_0$ can be expressed as the time-zero marginal distribution of a time-reversed Markov chain.
The transition densities of this chain are obtained by twisting the transition densities of the generative model with the considered potential functions.
More precisely, define, for all $k$, the potentials $\pot{k}{\star}{x_k} \eqdef \int \pot{0}{}{x_0} \, \abwker{0|k}{x_k}{x_0} \, \rmd x_0$.
Note that these potentials satisfy the recursion $\pot{k+1}{\star}{x_{k+1}} = \int \pot{k}{\star}{x_k} \, \abwker{k|k+1}{x_{k+1}}{x_k} \, \rmd x_k$.
Builing upon that, define the Markov transitions
\begin{equation}
\label{eq:optimal_bw_kernel}
\pibwker{k|k+1}{x_{k+1}}{x_k} \eqdef \frac{\pot{k}{\star}{x_k}}{\pot{k+1}{\star}{x_{k+1}}} \, \abwker{k|k+1}{x_{k+1}}{x _k},
\end{equation}
allowing the posterior \eqref{eq:FK_posterior} to be rewritten as
\begin{equation}
    \label{eq:posterior_markov}
    \cbwmarg{0}{x_0} = \int \cbwmarg{\last}{x_\last} \prod_{k = 0}^{\last-1} \pibwker{k|k+1}{x_{k+1}}{x_k} \, \rmd x_{1:n}\eqsp, \quad \cbwmarg{\last}{x_\last} = \pot{n}{\star}{x_n} \bwmarg{\last}{x_\last}[] / \normconst.
\end{equation}
In other words, the distribution $\cbwmarg{0}{}$ is the  time-zero marginal of a Markov model with transition densities $(\pibwker{k|k+1}{}{})_{k = \last-1}^0$ and initial distribution $\cbwmarg{\last}{}$. 
According to this decomposition, a sample $X^\star_0$ from the posterior \eqref{eq:posterior_def} can be obtained by sampling $X^\star_\last \sim \cbwmarg{\last}{}$ and then, recursively sampling $X^\star_k \sim \pibwker{k|k+1}{X^\star_{k+1}}{\cdot}$ from $k = n - 1$ till $k = 0$.
In practice, however, neither the Markov transition densities $\pibwker{k|k+1}{}{}$ nor the probability density function $\cbwmarg{\last}{}$ are tractable. 
The main challenge in estimating $\pibwker{k|k+1}{}{}$ stems essentially from the intractability of the potential $\pot{k}{\star}{x_k}$ as it involves computing an expectation under the high-cost sampling distribution $\abwker{0|k}{x_k}{\cdot}$.

Recent works have focused on developing tractable approximations of $\smash{\abwker{0|k}{x_k}{\cdot}}$. For the \emph{Diffusion Posterior Sampling} (DPS) algorithm \citep{chung2023diffusion}, the point mass approximation $\smash{\delta_{\predx{k}(x_k)}}$ of $\smash{\abwker{0|k}{x_k}{\cdot}}$ results in the estimate $\smash{\nabla_{x_k} \log \pot{0}{}{\predx{k}(x_{k})}}$ of $\smash{\nabla_{x_k} \log \pot{k}{\star}{x_k}}$. 
Then, given a sample $X _{k+1}$, an approximate sample $X _k$ from $\pibwker{k|k+1}{X _{k+1}}{\cdot}$ is obtained by first sampling $\tilde{X} _{k} \sim \abwker{k|k+1}{X_{k+1}}{\cdot}$ and then setting
\begin{equation}
    \label{eq:dps_update}
     X _k = \tilde{X} _{k} + \zeta \nabla_{x_{k + 1}} \log \pot{0}{}{\predx{k+1}(x_{k + 1})} |_{x_{k + 1} = X _{k+1}} \eqsp,
\end{equation}
where $\zeta > 0$ is a tuning parameter.
As noted in \citep{song2023loss, cardoso2023monte,boys2023tweedie}, the DPS updates \eqref{eq:dps_update} do not lead to an accurate approximation of the posterior $\cbwmarg{0}{}$ even in the simplest examples; see also \Cref{sec:experiments}.
Alternatively, \citet{song2022pseudoinverse} proposed the \emph{Pseudoinverse-Guided Diffusion Model} (\PGDM), which uses a Gaussian approximation of $\smash{\abwker{0|k}{x_k}{\cdot}}$ with mean $\smash{\predx{k}(x_k)}$ and diagonal covariance matrix set to $(1 - \acp{k}) \Id_\dimx$, which corresponds to the covariance of $\smash{\rbwker{0|k}{x_k}{\cdot}}$ if $\datadistr$ had been a standard Gaussian; see \cite[Appendix 1.3]{song2022pseudoinverse}.
More recently, \cite{finzi2023user, boys2023tweedie} proposed to approximate the exact KL projection of $\smash{\abwker{0|k}{x_k}{x_0}}$ onto the space of Gaussian distributions by noting that 
both its mean and covariance matrix can be estimated using $\smash{\predx{k}(x_k)}$ and its Jacobian matrix. We discuss in more depth the related works in \Cref{apdx:related}.

 \section{The DCPS algorithm}
\paragraph{Smoothing the DPS approximation.}

The bias of the \DPS\ updates \eqref{eq:dps_update} stems from the point mass approximation of the conditional distribution $\smash{\abwker{0|k}{x_k}{\cdot}}$. This approximation becomes more accurate as $k$ tends to zero and is crude otherwise. We aim here to mitigate the resulting approximation errors.
A core result that we leverage in this paper is that for any $\smash{(k, \ell) \in \intset{0}{n}^2}$ such that $\ell < k$, we can construct an estimate $\smash{\hbwker{\ell|k}{x_k}{\cdot}}$ of $\smash{\abwker{\ell|k}{x_k}{\cdot}}$ that bears a smaller approximation error than the estimate $\delta_{\predx{k}(x_k)}$ relatively to $\abwker{0|k}{x_k}{\cdot}$.
Formally, let $\hbwker{0|k}{x_k}{\cdot}$ denote any approximation of $\abwker{0|k}{x_k}{\cdot}$, such as that of the \DPS\ 
or \PGDM,  
and define the approximation of $\abwker{\ell|k}{x_k}{\cdot}$
\begin{equation}
    \hbwker{\ell|k}{x_k}{x_\ell} \eqdef \int \fwtrans{\ell|0, k}{x_0, x_k}{x_\ell} \hbwker{0|k}{x_k}{x_0} \, \rmd x_0
    \eqsp,
\end{equation}
where $\fwtrans{\ell|0, k}{x_0, x_k}{x_\ell}$ is defined in \eqref{eq:bridge_kernel}. We then have the following result.
\begin{proposition}[informal]
    \label{prop:w2-informal}
   Let $k \in \intset{1}{n}$. For all $\ell \in \intset{0}{k-1}$ and $x_k \in \rset^\dimx$,
    \begin{equation}
    \label{eq:w2}
    W_2(\hbwker{\ell|k}{x_k}{\cdot}, \abwker{\ell|k}{x_k}{\cdot})
    \leq \frac{\sqrt{\acp{\ell}}(1 - \acp{k} / \acp{\ell})}{(1 - \acp{k})}  W_2(\hbwker{0|k}{x_k}{\cdot}, \abwker{0|k}{x_k}{\cdot})\eqsp.
    \end{equation}
\end{proposition}
The proof is postponed to \Cref{apdx:proofw2}.
Note that the ratio in the right-hand-side of \eqref{eq:w2} is less than $1$ and decreases as $\ell$ increases.
As an illustration, using the \DPS\ approximation of $\abwker{0|k}{x_k}{\cdot}$, we find that $\hbwker{\ell|k}{x_k}{x_\ell} = \fwtrans{\ell|0, k}{\predx{k}(x_k), x_k}{x_\ell}$ improves upon \DPS\ in terms of approximation error. This observation prompts to consider \DPS-like approximations on shorter time intervals; instead of approximating expectations under $\abwker{0|k}{x_k}{\cdot}$, such as the potential $\pot{k}{\star}{x_k}$, we should transform our initial sampling problem so that we only have to estimate expectations under $\abwker{\ell|k}{x_k}{\cdot}$ for any $\ell$ such that the difference $k - \ell$ is small. This motivates the \emph{blocking approach} introduced next.
\paragraph{Intermediate posteriors.} We approach the original problem of sampling from $\target_0$ via a series of simpler, \emph{intermediate} posterior sampling problems of increasing difficulty.
More precisely, let us consider the intermediate posteriors defined as
\begin{equation}
\label{eq:interm_post}
\cbwmarg{\tstep{\ell}}{x_\tstep{\ell}} \eqdef \pot{\tstep{\ell}}{}{x_\tstep{\ell}} \bwmarg{\tstep{\ell}}{x_\tstep{\ell}} \big/ \normconst_{\tstep{\ell}}, \quad \mbox{with} \quad \normconst[]_\tstep{\ell} \eqdef \int \pot{\tstep{\ell}}{}{x_\tstep{\ell}} \bwmarg{\tstep{\ell}}{x_\tstep{\ell}} \, \rmd x_\tstep{\ell},
\end{equation}
where $(\pot{\tstep{\ell}}{}{})_{\ell = 1} ^\tauks$ are potential functions designed by the user and $(\tstep{\ell})_{\ell = 0}^L$ is an increasing sequence in $\intset{0}{n}$ such that $\tstep{0} = 0$ and $\tstep{L} = n$.
Here, $L$ is typically much smaller than $n$.
To obtain an approximate sample from $\target_0 = \target_{\tstep{0}}$, the \algo\ algorithm recursively uses an approximate sample $X_{\tstep{\ell+1}}$ from $\target_\tstep{\ell+1}$ to obtain an approximate sample $X_\tstep{\ell}$ from $\target_\tstep{\ell}$.
Indeed, mirroring \eqref{eq:posterior_markov} it holds
\begin{equation} \label{eq:posterior_markov:block}
    \cbwmarg{\tstep{\ell}}{x_\tstep{\ell}} = \int \cbwmarg{\tstep{\ell+1}}{x_\tstep{\ell+1}}[\ell] \prod_{m=\tstep{\ell}} ^{\tstep{\ell+1}-1} \pibwker{m|m+1}{x_{m+1}}{x_m}[\ell] \, \rmd x_{\tstep{\ell}+1: \tstep{\ell+1}}\eqsp,
\end{equation}
where for $m \in \intset{\tstep{\ell}}{\tstep{\ell+1} - 1}$,
\begin{align*}
    \cbwmarg{\tstep{\ell+1}}{x_\tstep{\ell+1}}[\ell] & \eqdef \pot{\tstep{\ell+1}}{\ell, \star}{x_\tstep{\ell+1}} \bwmarg{\tstep{\ell+1}}{x_\tstep{\ell+1}} \big/ \normconst_\tstep{\ell} \eqsp, \\
\pibwker{m|m+1}{x_{m+1}}{x_m}[\ell] & \eqdef \pot{m}{\ell, \star}{x_m} \abwker{m|m+1}{x_{m+1}}{x_m} \big/ \pot{m+1}{\ell, \star}{x_{m+1}} \eqsp
\end{align*}
and for $m \in \intset{\tstep{\ell}+1}{\tstep{\ell+1}}$,
\begin{equation}
    \label{eq:block_potential}
\pot{m}{\ell, \star}{x_m} \eqdef \int \pot{\tstep{\ell}}{}{x_\tstep{\ell}} \abwker{\tstep{\ell}|m}{x_m}{x_\tstep{\ell}} \, \rmd x_\tstep{\ell} \eqsp.
\end{equation}
We emphasize that the initial distribution $\cbwmarg{\tstep{\ell+1}}{}[\ell]$ in \eqref{eq:posterior_markov:block} is \emph{different} from the posterior $\cbwmarg{\tstep{\ell+1}}{}$ as the former involves the user-defined potential whereas the latter the intractable one.
The main advantage of our approach lies in the fact that, unlike the potentials in the transition densities \eqref{eq:optimal_bw_kernel}, which involve expectations under $\abwker{0|k}{x_k}{\cdot}$, the potentials \eqref{eq:block_potential} are given by expectations under the distributions $\abwker{\tstep{\ell}|m}{x_m}{\cdot}$, which are easier to approximate in the light of \Cref{prop:w2-informal}. In the sequel, we use this approximation for the estimation of the potentials \eqref{eq:block_potential}; this yields approximate potentials
\begin{equation}
    \label{eq:opt_pot_approx}
    \iapot{m}{\ell, \star}{x_m} \eqdef \int \pot{\tstep{\ell}}{}{x_\tstep{\ell}} \hbwker{\tstep{\ell}|m}{x_m}{x_\tstep{\ell}} \, \rmd x_\tstep{\ell} \eqsp, \quad m \in \intset{\tstep{\ell} + 1}{\tstep{\ell+1}} \eqsp,
\end{equation}
which serve as a  substitute for the intractable $\pot{m}{\ell, \star}{}$.
Let us now summarize how our algorithm works. Starting from a sample $X_\tstep{\ell+1}$, which is approximately distributed according to $\cbwmarg{\tstep{\ell+1}}{}$, the next sample $X_\tstep{\ell}$ is generated in the next two steps:
\begin{enumerate}[leftmargin=.3in]
    \item Perform Langevin Monte Carlo steps initialized at $X _\tstep{\ell+1}$ and targeting $\cbwmarg{\tstep{\ell+1}}{}[\ell]$, yielding $X^\ell _\tstep{\ell+1}$.
    \item 
    Simulate a Markov chain $(X_j)_{j = \tstep{\ell+1}} ^\tstep{\ell}$ initialized with $X_\tstep{\ell+1} = X^\ell _\tstep{\ell+1}$ and whose transition from $X_{j+1}$ to $X_j$ is the minimizer of 
    \begin{equation}
    \label{eq:kl_div}
        \kldivergence{\vbwker{j|j+1}{X_{j+1}}{\cdot}}{\pibwker{j|j+1}{X_{j+1}}{\cdot}[\ell]},
    \end{equation}
    where $\vbwker{j|j+1}{X_{j+1}}{}$ is a mean-field Gaussian approximation with parameters $\vparam \eqdef (\vmu, \vstd) \in \rset^\dimx \times \rset^\dimx _{> 0}$.
    $X_{j}$ is drawn from $\vbwker{j|j+1}{X_{j+1}}{\cdot}[\vparam_j(X_{j+1})]$, where $\vparam_j(X_{j+1})$ is a minimizer of the proxy of \eqref{eq:kl_div}.
\end{enumerate}
In the following, we elaborate more on Step~1 and Step~2 and discuss the choice of the intermediate potentials.
The pseudo-code of the {\algo} algorithm is in \Cref{algo:algo}.
\label{sec:DCPS}
\paragraph{Sampling the initial distribution.}  In order to perform \textbf{Step~1}, we use the discretized Langevin dynamics \cite{roberts:tweedie:1996} with the estimate $\smash{\nabla \log \iapot{\tstep{\ell+1}}{\ell, \star}{} + \score{\tstep{\ell+1}}{}}$ of the score $\nabla \log \cbwmarg{\tstep{\ell+1}}{}[\ell]$.
This estimate results from the use of $\score{\tstep{\ell+1}}$ as an approximation of $\smash{\nabla \log \bwmarg{\tstep{\ell+1}}{}}$ in combination with the approximate potential \eqref{eq:opt_pot_approx}.
We then obtain the approximate sample $\smash{X^\ell _\tstep{\ell+1}}$ of $\cbwmarg{\tstep{\ell+1}}{}$ by running $M$ steps of the tamed unadjusted Langevin (TULA) scheme \citep{brosse2019tamed}; see \Cref{algo:algo}.
Here, the intractability of the involved densities hinder the usage of the Metropolis-Hastings corrections to reduce the inherent bias of the Langevin algorithm.

\paragraph{Sampling the transitions.} We now turn to \textbf{Step 2}. Given $X_{j+1}$, we optimize the following estimate of \Cref{eq:kl_div}, where we simply replace $\pot{j}{\ell, \star}{}$ by the approximation \eqref{eq:opt_pot_approx}:
\begin{eqnarray*}
    - \int \log \iapot{j}{\ell, \star}{x_j} \vbwker{j|j+1}{x_{j+1}}{x_j} \, \rmd x_j + \kldivergence{\vbwker{j|j+1}{x_{j+1}}{\cdot}}{\abwker{j|j+1}{x_{j+1}}{\cdot}}
    \eqsp.
\end{eqnarray*}
Letting $\vbwker{j|j+1}{x_{j+1}}{x_j} = \normpdf(x_j; \vmu_{j}, \mathrm{diag}(\rme^{\vlogstd_j}))$, where the variational parameters $\vmu_{j}, \vlogstd_{j}$ are in $\rset_\dimx$, the previous estimate yields the objective 
\begin{multline}
 \label{eq:variational-criterion}
 \mathcal{L} _{j} (\vmu_{j}, \vlogstd_{j}; x_{j+1}) \eqdef - \pE \big[ \log \iapot{j}{\ell, \star}{\vmu_j + \rme^{\vlogstd_j / 2}  Z}\big]  \\ +  \frac{\| \vmu_j - \muDDIM{j|j+1}(x_{j+1})\|^2}{2 \sigma^2 _{j|j+1}} - \frac{1}{2}\sum_{i = 1}^\dimx \left( \vlogstd_{j, i} - \frac{\rme^{\vlogstd_{j, i}}}{\sigma^2 _{j|j+1}} \right) \eqsp,
\end{multline}
where $Z$ is $\dimx$-dimensional standard Gaussian and $\muDDIM{j|j+1}(x_{j+1})$ is the mean of \eqref{eq:ddpm_kernel}.
Note here that we have used the reparameterization trick \citep{kingma2013auto} and the closed-form expression of the KL divergence between two multivariate Gaussian distributions.
We optimize the previous objective using a few steps of SGD by estimating the first term on the \rhs\ with a single sample as in \citep{kingma2013auto}.
For each $j \in \intset{\tstep{\ell}}{\tstep{\ell+1} - 1}$, we use $\muDDIM{j|j+1}$ and $\log \sigma^2 _{j|j+1}$ as initialization for $\vmu_j$ and $\vlogstd_j$.

\paragraph{Intermediate potentials.}
Here, we give general guidelines to choose the user-defined potentials $(\pot{\tstep{\ell}}{}{})_{\ell = 1} ^\tauks$.
Our design choice is to rescale the input and then anneal the initial potential $g_0$.
Therefore, we suggest
\begin{equation}
    \txts \pot{\tstep{\ell}}{}{x} = \pot{0}{}{\frac{x}{\beta_{\tstep{\ell}}}}^{\gamma_{\tstep{\ell}}}
    \eqsp,
    \label{eq:potential-design-choice}
\end{equation}
where $\gamma_{\tstep{\ell}}, \beta_{\tstep{\ell}} > 0$ are tunable paramerters.
This design choice is inspired from the tempering sampling scheme \cite{neal2001annealed} which uses the principle of progressively moving an intial distribution to the targeted one.
We provide some examples in the case of Bayesian inverse problems where the unobserved signal and the observation are modelled jointly as a realization of $(X, Y) \sim p(y | x) p_0(x)$, where $p(y | x)$ is the conditional density of $Y$ given $X = x$.
In this case, the posterior $\pi_0$ of $X$ given $Y = y$ is given by \eqref{eq:posterior_def} with $g_0(x) = p(y | x)$.


\emph{Linear inverse problems with Gaussian noise.} In this case, $\pot{0}{}{x} = \normpdf(\obs; \bfA x, \stdobs^2 \Id_\dimobs)$, where $\smash{A \in \rset^{\dimobs \times \dimx}}$. Popular applications in image processing  include super-resolution, inpainting, outpainting, and deblurring.
We use \eqref{eq:potential-design-choice} with $(\beta_{\tstep{\ell}}, \gamma_\tstep{\ell}) = (\sqrt{\acp{\tstep{\ell}}}, \acp{\tstep{\ell}})$,
\begin{equation}
    \pot{\tstep{\ell}}{}{x} =  \normpdf(\sqrt{\acp{\tstep{\ell}}} \obs; \bfA x, \sigma^2 _\obs \Id_\dimobs)
    \eqsp,
    \label{eq:intermediatepot-lininvp}
\end{equation}
which corresponds to the likelihood of $x$ given the \emph{pseudo observation} $\sqrt{\acp{\tstep{\ell}}} \obs$ under the same linear observation model that defines $\pot{0}{}{}$.
This choice of $\pot{\tstep{\ell}}{}{}$ enables exact computation of \eqref{eq:opt_pot_approx} and allows information on the observation $y$ to be taken into account early in the denoising process.

\emph{Low-count (or shot-noise) Poisson denoising.}
In a Poisson model for an image, the grey levels of the
image pixels are modelled as Poisson-distributed random variables. More specifically, let  $A \in \rset^{\dimobs \times \dimx}$ be a matrix with nonnegative entries and $x \in [0, 255]^{C \times H \times W}$, where $C$ is the number of channels and $H$ the height and $W$ the width. For every $i \in \intset{1}{\dimobs}$, $Y_i$ is Poisson-distributed with mean $(A x)_i$, and the likelihood of $x$ given the observation is therefore given by
$\smash{x \mapsto \prod_{j = 1}^{\dimobs} (\lambda A x)_j ^{y_j} \rme^{-(\lambda A x)_j} / y_j} !\,$ where $\lambda > 0$ is the rate. Following \cite{chung2023diffusion} we consider as likelihood its normal approximation, \emph{i.e.} $\smash{\pot{0}{}{} = \prod_{j = 1}^\dimobs \normpdf(\obs_j; \lambda (Ax)_j, y_j)}$. This model is relevant for many tasks such as low-count photon imaging and computed tomography (CT) reconstruction \cite{nowak2000statistical, rodrigues2008denoising, marais2017proximal}.
We use \eqref{eq:potential-design-choice} with $\beta_\tstep{\ell} = \gamma_\tstep{\ell} = \sqrt{\acp{\tstep{\ell}}}$: 
\begin{equation}
    \label{eq:intermediatepot-poisson}
    \smash{\pot{\tstep{\ell}}{}{x} = \prod_{j = 1}^\dimobs \normpdf(\sqrt{\acp{\tstep{\ell}}} y_j; \lambda (A x) _j, \sqrt{\acp{\tstep{\ell}}} y_j)}
    \eqsp.
\end{equation}

\emph{JPEG dequantization.} JPEG \cite{wallace1992jpeg} is a ubiquitous method for lossy compression of images. Use $h_q$ to denote the JPEG encoding function with \emph{quality factor} $q \in \intset{0}{100}$, where a small $q$ is associated with high compression.
Denote by $h^\dagger _q$ the JPEG decoding function that returns an image in RGB space with a certain loss of detail, depending on the degree of compression $q$, compared to the original image.
Since we require the potential to be differentiable almost everywhere, we use the differentiable approximation of JPEG developed in \cite{shin2017jpeg}, which replaces the rounding function used in the quantization matrix with a differentiable approximation that has non-zero derivatives almost everywhere.
In this case, $\pot{0}{}{x} = \normpdf(h^\dagger _q(\obs); h^\dagger _q(h_q(x)), \stdobs^2 \Id_\dimobs)$, where $\obs$ is in YCbCr space.
Combining this with \Cref{eq:potential-design-choice} with $(\beta_{\tstep{\ell}}, \gamma_\tstep{\ell}) = (\acp{\tstep{\ell}}, \acp{\tstep{\ell}})$ and assuming that the composition $h^\dagger _q \circ h_q$ is a homogenious map, the intermediate potentials are
$
    \smash{\pot{\tstep{\ell}}{}{x} = \normpdf(\sqrt{\acp{\tstep{\ell}}} \, h^\dagger _q(\obs); h^\dagger _q(h _q(x)), \stdobs^2 \Id_\dimx)}
    \eqsp.
$

\section{Experiments}
\begin{figure}[htb]
    \centering
    \includegraphics[width=1\textwidth]{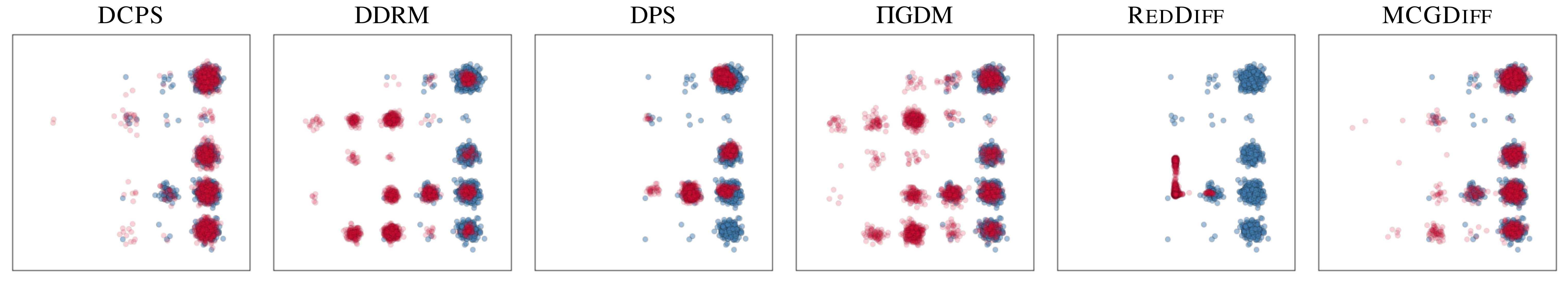}   
    \captionsetup{font=small}
    \caption{First two dimensions of samples (in red) from each algorithm on the 25 component Gaussian mixture posterior sampling problem with $(\dimx, \dimobs) = (100, 1)$. The true posterior samples are given in blue.}
\end{figure}

\label{sec:experiments}
In this section, we demonstrate the performance of \algo\ and compare it with \DPS\ \citep{chung2023diffusion}, \PGDM\ \citep{song2022pseudoinverse}, \ddrm\ \citep{kawar2022denoising}, \reddiff\ \citep{mardani2024a}, and \mcgdiff\ \citep{cardoso2023monte} on several Bayesian inverse problems.
We also benchmark our algorithm against \DiffPIR\ \cite{zhu2023diffpir}, \DDNM\ \cite{wang2023ddnm}, \FPS\ \cite{dou2024diffusion}, and \SDA\ \cite{rozet2023sda} but we defer the results to the \Cref{apdx:add-experiments}.

First, we consider a simple toy experiment in which the posterior distribution is available in closed form. Next, we apply our algorithm to superresolution (SR $4\times$ and $16\times$), inpainting and outpainting tasks with Gaussian and Poisson noise, and JPEG dequantization.
For these imaging experiments, we use the \ffhq  \texttt{256} \cite{karras2019style} and \imagenet \texttt{256} \cite{deng2009imagenet} datasets and the publicly available pre-trained models of \cite{choi2021ilvr} and \cite{dhariwal2021diffusion}.
Finally, we benchmark our method on a trajectory inpainting task using the pedestrian dataset \texttt{UCY} for which we have trained a Diffusion model. All  details can be found in \Cref{apdx:implementation_details}.
\paragraph{Gaussian mixture.}
We first evaluate the accuracy of \algo\ on a linear inverse problem with a Gaussian mixture (GM) prior, for which the posterior can be explicitly computed: it is also a Gaussian mixture whose means, covariance matrices, and weights are in a closed form; see \Cref{subsec:gm}.
\begin{wraptable}{r}{5.5cm}
    \captionsetup{font=small}
    \caption{95\% confidence interval for the SW on the GM experiment.}
    \resizebox{.4\textwidth}{!}{
        \begin{tabular}{lcc}
        \toprule
        &$\dimx=10, \dimobs=1$ & $\dimx=100, \dimobs=1$ \\
        \midrule
        \algo${}_{50}$     & $2.91 \pm 0.74$ & $4.04 \pm 1.00$ \\
        \algo${}_{500}$     & $\mathbf{2.19} \pm 0.68$ & $\underline{3.29} \pm 0.95$ \\
        \DPS      & $5.80 \pm 0.75$ & $5.68 \pm 0.73$ \\
        \ddrm     & $3.77 \pm 0.96$ & $5.70 \pm 0.78$ \\
        \PGDM & $4.23 \pm 0.90$ & $4.61 \pm 0.68$ \\
        \reddiff  & $6.36 \pm 1.27$ & $7.47 \pm 0.87$ \\
        \mcgdiff  & $\underline{2.28} \pm 0.75$ & $\mathbf{2.83} \pm 0.71$ \\
        \bottomrule
        \end{tabular}
    }
        \label{table:gm}
\end{wraptable}
In this case, the predictor $\predx{k}[\param^*]$ is available in a closed form; see \Cref{subsec:gm} for more details.
We consider a Gaussian mixture prior with $25$ components in dimensions $\dimx = 10$ and $\dimx = 100$.
The potential is $\pot{0}{}{x} = \normpdf(\obs; A x, \sigma^2 _\obs \Id_\dimobs)$ with $\dimobs = 1$ and $A$ is a $1 \times \dimx$ vector.
The results are averaged over $30$ randomly generated replicates of the measurement model $(\obs, A, \sigma^2 _\obs)$ and the mixture weights.
Then, for each pair of prior distribution and measurement model, we generate $N_s = 2000$ samples with each algorithm and compare them with $N_s$ samples from the true posterior distribution using the sliced Wasserstein (SW) distance. For \algo, we used $\tauks = 3$ blocks and $K = 2$ gradient steps, respectively, and compared two configurations, denoted by \algo${}_{50}$ and \algo${}_{500}$, of the algorithm with $M = 50$ and $M = 500$ Langevin steps, respectively. See \Cref{algo:algo}. The results are reported in \Cref{table:gm}.
It is worthwhile to note that \algo\ outperforms all baselines except for \mcgdiff. However, by increasing the number of Langevin steps, its performance closely matches that of \mcgdiff.

\begin{figure}[htb]
    \centering
    \hspace*{-1.5mm}
    \subfigure{
        \includegraphics[width=0.49\textwidth]{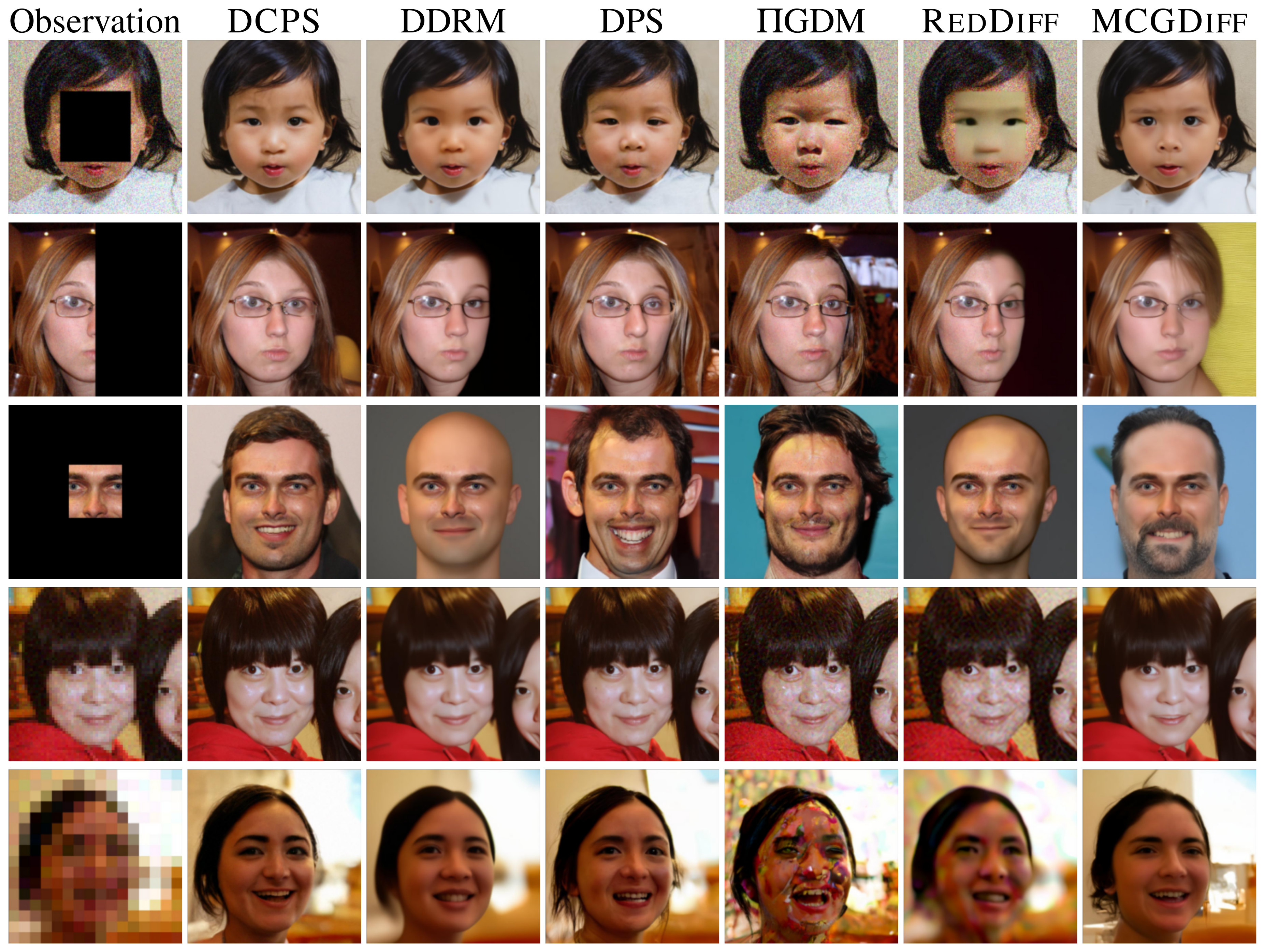}   
    }
    \subfigure{
        \includegraphics[width=0.49\textwidth]{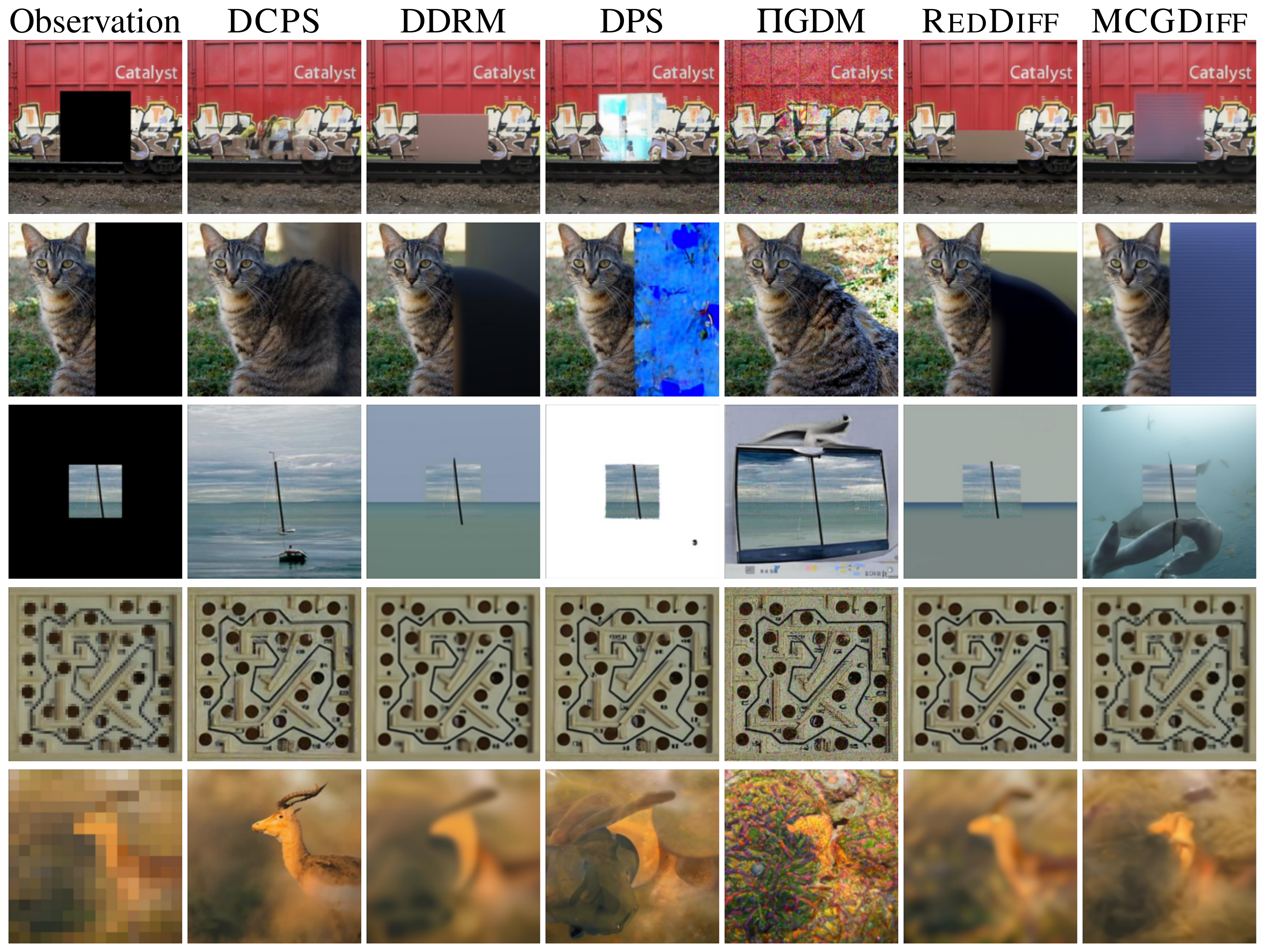}
        }
    \caption{Sample images for inpainting with center, half, expand masks and for Super Resolution with $4 \times$ and $16 \times$ factors. On the left: \ffhq\ dataset and on the right \imagenet\ dataset.}
    \label{fig:samples}
    \vspace*{-1mm}
\end{figure}

\paragraph{Imaging experiment.}
\Cref{table:lin_invp} reports the results for the linear inverse problems with Gaussian noise with two noise variance levels $\stdobs = 0.05$ and $\stdobs = 0.3$, \Cref{table:jpeg-dequantization} for the JPEG dequantization problem with $\stdobs = 10^{-3}$, QF $\in \{2, 8\}$, and \Cref{table:poisson-denoising} for the Poisson denoising task with rate $\lambda = 0.1$.
For all tasks and datasets, we use the same parameters for \algo\ and therefore do not perform any task or dataset-specific tuning. We use $\tauks = 3$, $K = 2$ gradient steps, and $M = 5$ Langevin steps.
To ensure a fair comparison with  \DPS\ and \PGDM\, we use 300 DDPM steps for \algo\ and $1000$ steps for both \DPS\ and \PGDM, which ensures that all the algorithms have the same runtime and memory footprint; see \Cref{fig:runtime}. For \mcgdiff, which has a large memory requirement, we use $N = 32$ particles in the SMC sampling step and then randomly draw one sample from the resulting particle approximation of the posterior. Finally, for \ddrm\ we use 200 diffusion steps and for \reddiff\ we use $1000$ gradient steps and the parameters recommended in the original paper. We provide the implementation details for all algorithms in \Cref{apdx:implementation_details}.
\begin{wraptable}{r}{9cm}
    \captionsetup{font=small}
        \caption{Mean LPIPS value on different tasks. Lower is better.}
    \resizebox{.62\textwidth}{!}{
    \begin{tabular}{@{}cccccccc@{}}
    \toprule
    Dataset / $\stdobs$ & Task & \algo & \ddrm & \DPS & \PGDM & \reddiff & \mcgdiff \\
    \midrule
    \multirow{4}{*}{\centering \ffhq\ / 0.05}
    & Half & \bf{0.20} & 0.25 & \underline{0.24} & 0.26 & 0.28 & 0.36 \\
    & Center & \bf{0.05} & \underline{0.06} & 0.07 & 0.19 & 0.12 & 0.24 \\
    & SR $4 \times$ & \bf{0.09} & 0.18 & \underline{0.09} & 0.33 & 0.36 & 0.15 \\
    & SR $16 \times$ & \bf{0.23} & 0.36 & \underline{0.24} & 0.44 & 0.51 & 0.32 \\
    \midrule
    \multirow{4}{*}{\centering \ffhq\ / 0.3}
    & Half & \bf{0.25} & \underline{0.30} & 0.31 & 0.64 & 0.76 & 0.80 \\
    & Center & \bf{0.10} & \underline{0.13} & 0.11 & 0.62 & 0.75 &  0.55 \\
    & SR $4 \times$ & \underline{0.21} & 0.26 & \bf{0.19} & 0.77 & 0.77 & 0.65 \\
    & SR $16 \times$ & \bf{0.35} & \underline{0.41} & 0.43 & 0.64 & 0.74 & 0.52 \\
    \midrule
    \multirow{4}{*}{\centering \imagenet\ / 0.05}
    & Half & \bf{0.35} & \underline{0.40} & 0.44 & 0.38 & 0.44 & 0.83 \\
    & Center & \underline{0.18} & \bf{0.14} & 0.31 & 0.29 & 0.22 & 0.45 \\
    & SR $4 \times$ & \bf{0.24} & \underline{0.38} & 0.41 & 0.78 & 0.56 & 1.32 \\
    & SR $16 \times$ & \bf{0.44} & 0.72 & 0.50 & \underline{0.60} & 0.83 & 1.33 \\
    \midrule
    \multirow{4}{*}{\centering \imagenet\ / 0.3}
    & Half & \bf{0.40} & \underline{0.46} & 0.48 & 0.82 & 0.76 & 0.86 \\
    & Center & \bf{0.24} & \underline{0.25} & 0.40 & 0.68 & 0.71 & 0.47 \\
    & SR $4 \times$ & \bf{0.43} & 0.50 & \underline{0.47} & 0.87 & 0.83 & 1.31 \\
    & SR $16 \times$ & 0.72 & 0.77 & \bf{0.57} & 0.72 & 0.92 & \underline{0.67} \\
    \midrule
    Average & & \bf{0.28} & 0.35 & \underline{0.32} & 0.57 & 0.60 &  0.67 \\
    \midrule
    \end{tabular}
    }
    \label{table:lin_invp}
    \vspace{-.2cm}
\end{wraptable}
For the JPEG dequantization task, we use $\stdobs = 10^{-3}$ and $\lambda = 0.1$. We only benchmark our method against \PGDM\ and \reddiff, since \mcgdiff\ and \ddrm\ do not handle non-linear inverse problems. We did not include \DPS\ in our benchmark because we have not managed to find a suitable choice of hyperparameters to achieve reasonable results.
Finally, for the Poisson-shot noise case, we compare against \DPS.
We use the step size for super-resolution recommended in the original paper  \citep[see][Appendix D.1]{chung2023diffusion}, and found, via a grid search, that the same value is also effective for the other tasks.
\begin{wraptable}{r}{6cm}
    \centering
        \captionsetup{font=small}
        \caption{Mean LPIPS value on JPEG dequantization.}
    \resizebox{.4\textwidth}{!}{
    \begin{tabular}{@{}ccccc@{}}
    \toprule
    Dataset  & Task & \algo & \PGDM & \reddiff \\
    \midrule
    \multirow{2}{*}{\centering \ffhq}
    & $\textsc{QF} = 2$ & \bf{0.20} & 0.37 & \underline{0.32}  \\
    & $\textsc{QF} = 8$ & \bf{0.08} & \underline{0.15} & 0.18 \\
    \midrule
    \multirow{2}{*}{\centering \imagenet}
    & $\textsc{QF} = 2$ & \bf{0.44} & 0.93 & 0.50 \\
    & $\textsc{QF} = 8$ & \bf{0.24} & 0.95 & 0.31 \\
    \midrule
    \end{tabular}
    }
    \label{table:jpeg-dequantization}
\end{wraptable}

\emph{Evaluation.}
As shown in \Cref{table:lin_invp}, \algo\ outperforms the other baselines on 13 out of 16 tasks and has the best average performance.
In particular, it compares favorably with \PGDM\ and \DPS, its closest competitors, while exhibiting the same runtime and memory requirements; see \Cref{fig:runtime}, where we give the average runtime and memory usage for each algorithm. The memory consumption is measured by how many samples each algorithm can generate in parallel on a single 48GB L40S NVIDIA GPU for the Diffusion model trained on \ffhq\ \cite{dhariwal2021diffusion}. 
\begin{wraptable}{r}{6cm}
    \centering
    \includegraphics[width=.4\textwidth]{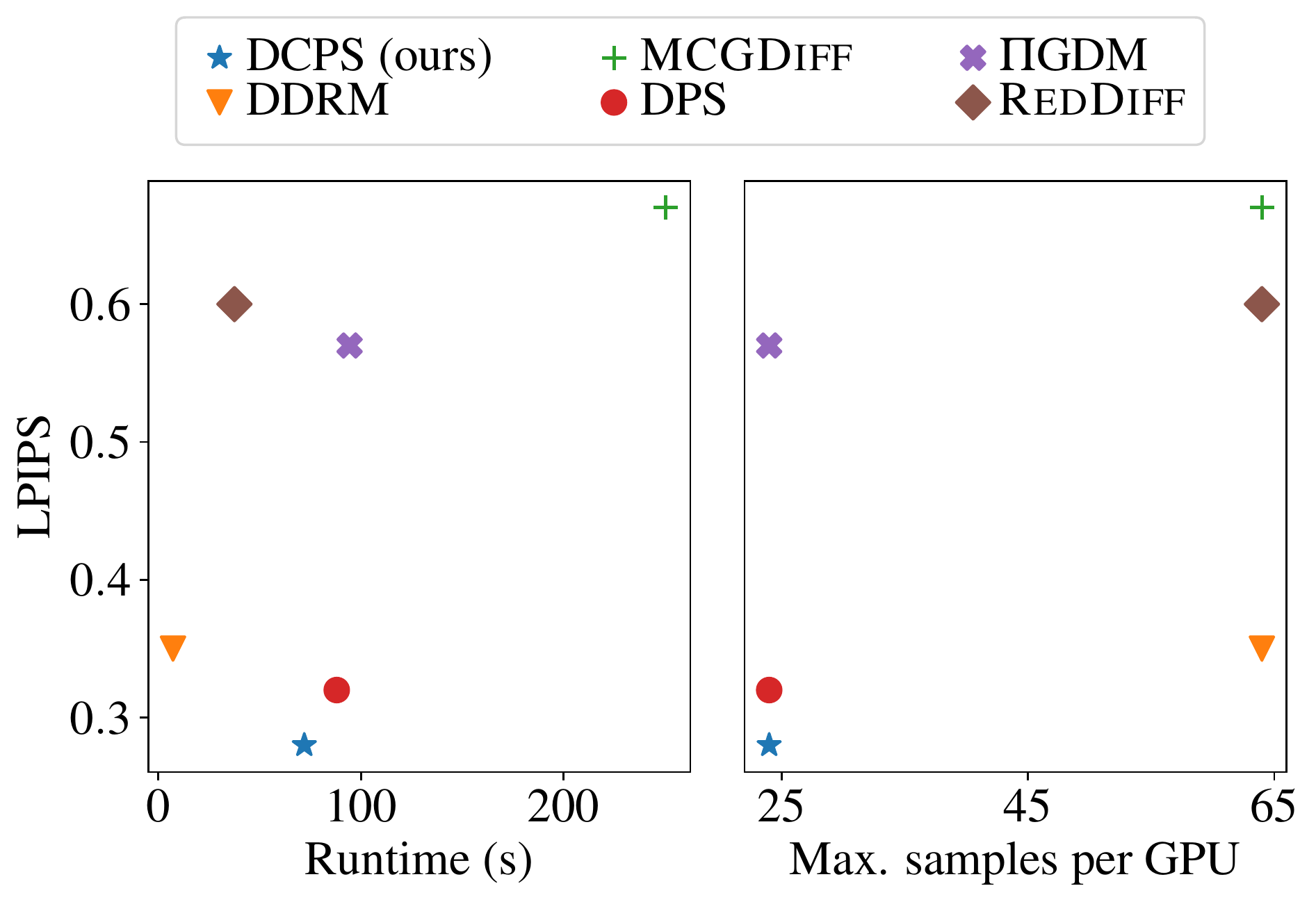}
    \captionsetup{font=small}
    \caption{LPIPS metric against the runtime and memory cost of the algorithms.}
    \vspace{2mm}
    \label{fig:runtime}
    \centering
    \captionsetup{font=small}
    \caption{$\ell_2$ distance quantiles with \mcgdiff\ as reference.}
    \resizebox{.4\textwidth}{!}{
            \begin{tabular}{lccccccc}
            \toprule
            & \multicolumn{3}{c}{$\sigma_y = 0.005$}  & & \multicolumn{3}{c}{$\sigma_y = 0.01$}  \\
            & $q50$ & $q25$ & $q75$ && $q50$ & $q25$ & $q75$ \\
            \midrule
            \algo    & \textbf{1.31} & \textbf{1.33} & \textbf{1.47} && \textbf{1.33} & \textbf{1.42} & \textbf{1.42} \\
            \DPS     & \underline{1.34} & 1.40 & 1.61 && \underline{1.36} & 1.48 & 1.52 \\
            \ddrm    & 1.48 & 1.46 & 1.61 && 1.59 & 1.62 & 1.61 \\
            \PGDM    & 1.36 & \underline{1.35} & \textbf{1.47} && 1.37 & \underline{1.43} & \textbf{1.42} \\
            \reddiff & 1.67 & 1.57 & 1.82 && 1.56 & 1.54 & 1.65 \\
            \bottomrule
            \end{tabular}
        }
        \label{table:trajectory-inpainting-L2}
\end{wraptable}
We emphasize that \algo\ is more robust to larger noise levels than \PGDM\ and \reddiff, as evidenced by the large increase in the LPIPS value for these algorithms in the case $\stdobs = 0.3$.
On the JPEG dequantization task (\Cref{table:jpeg-dequantization}), \algo\ also shows better performance  than these algorithms and even more so for the high compression level ($\textsc{QF} = 2$). On the Poisson-shot noise tasks, \algo\ outperforms \DPS\ by a significant margin; see \Cref{table:poisson-denoising}.
Finally, we display various reconstructions obtained with each algorithm. More specifically, we have generated 4 samples each, with the same seed. \Cref{fig:samples} displays the first sample and the remaining ones are deferred to \Cref{apdx:samples}.
For \mcgdiff\, we show 4 random samples of the same particle filter.
Due to the collapse of the particle filter in very large dimensions \cite{bickel:li:bengtsson:2008}, they are all similar.
Surprisingly, the samples produced by \ddrm\ and \reddiff\ for the outpainting tasks also show striking similarities, although the samples have been drawn independently.
\begin{figure}[htb]
    \centering
    \subfigure{
        \includegraphics[width=0.38\textwidth]{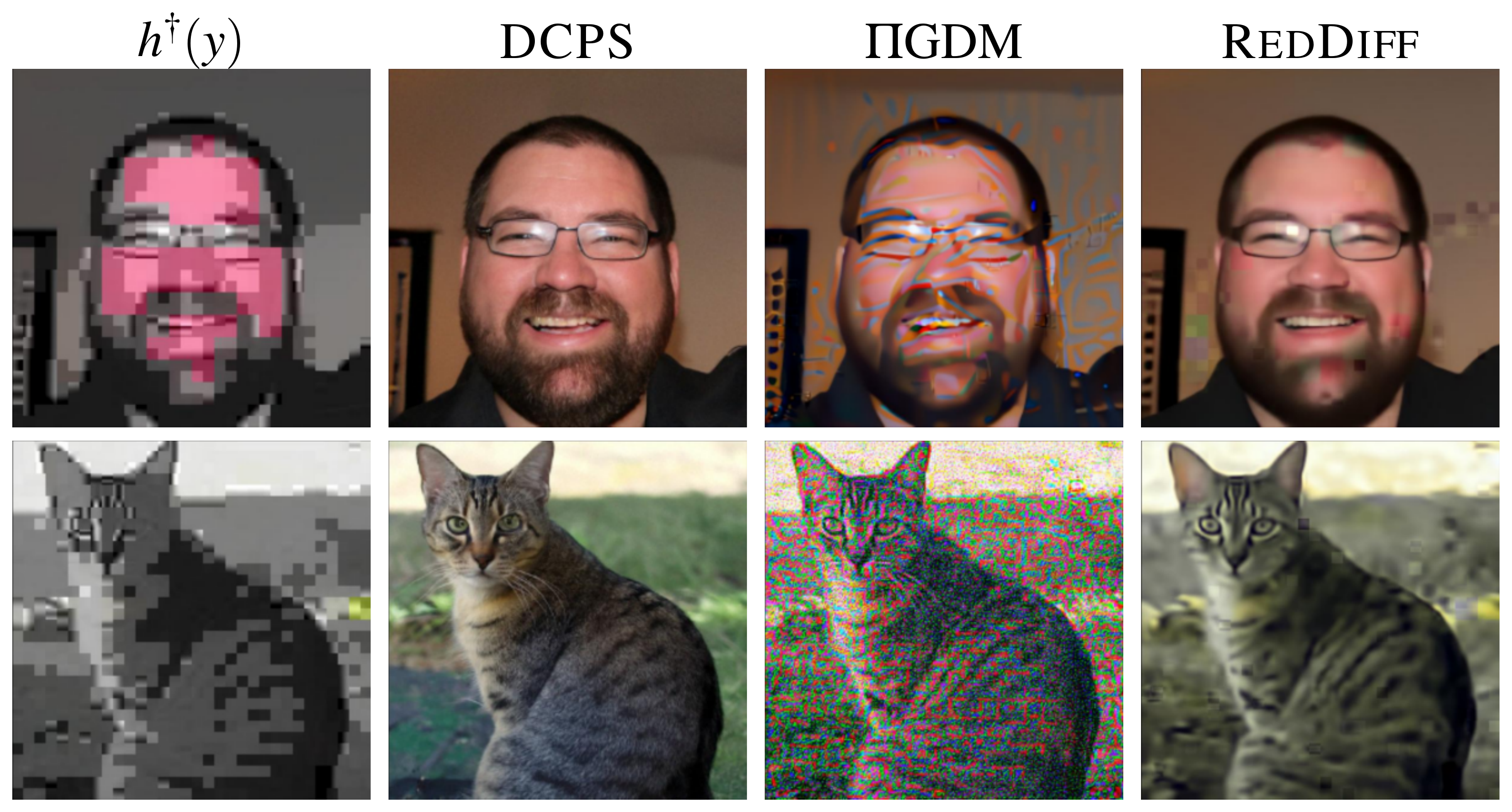}   
    }
    \subfigure{
        \includegraphics[width=0.286\textwidth]{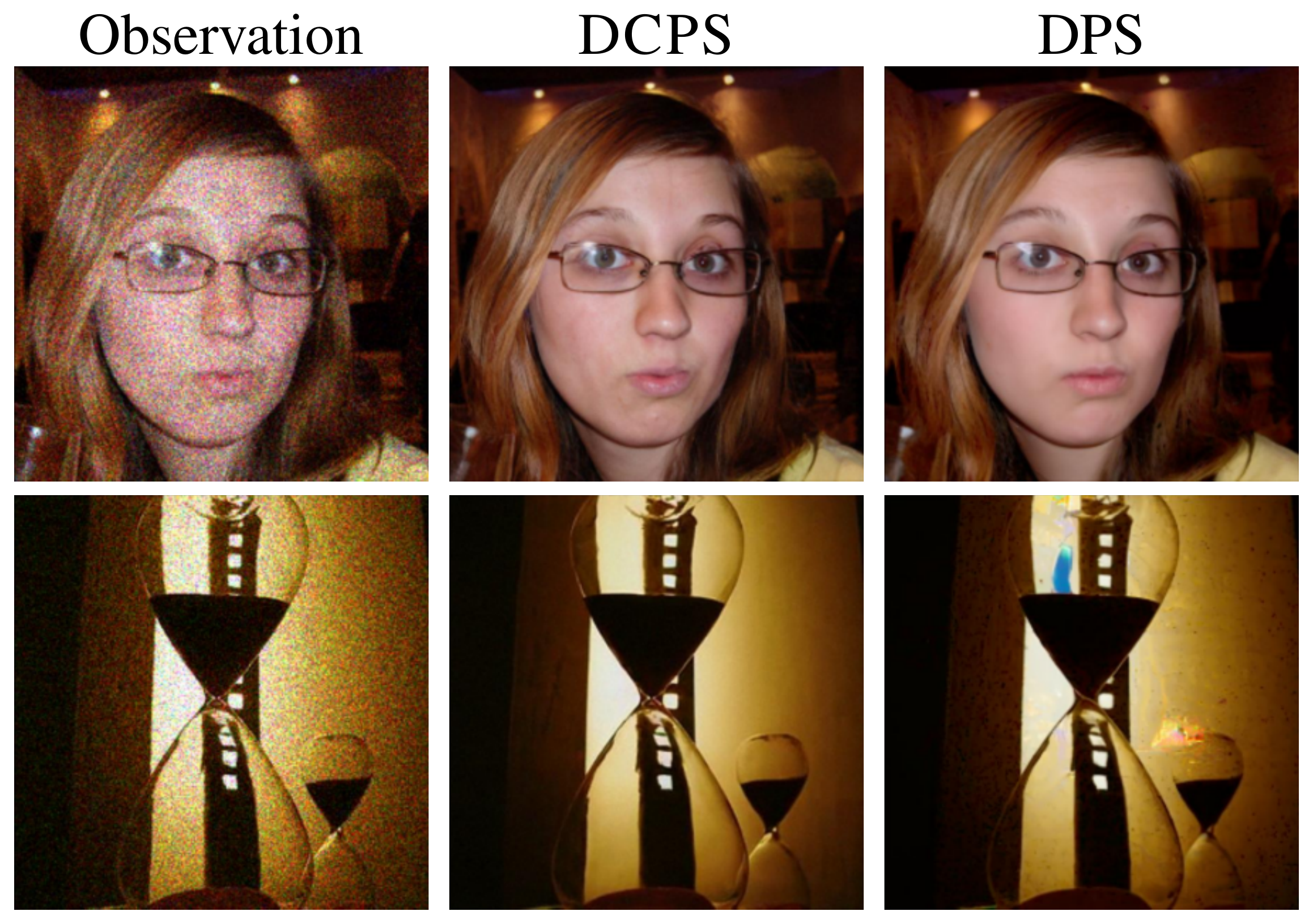}
        }
    \subfigure{
        \includegraphics[width=0.286\textwidth]{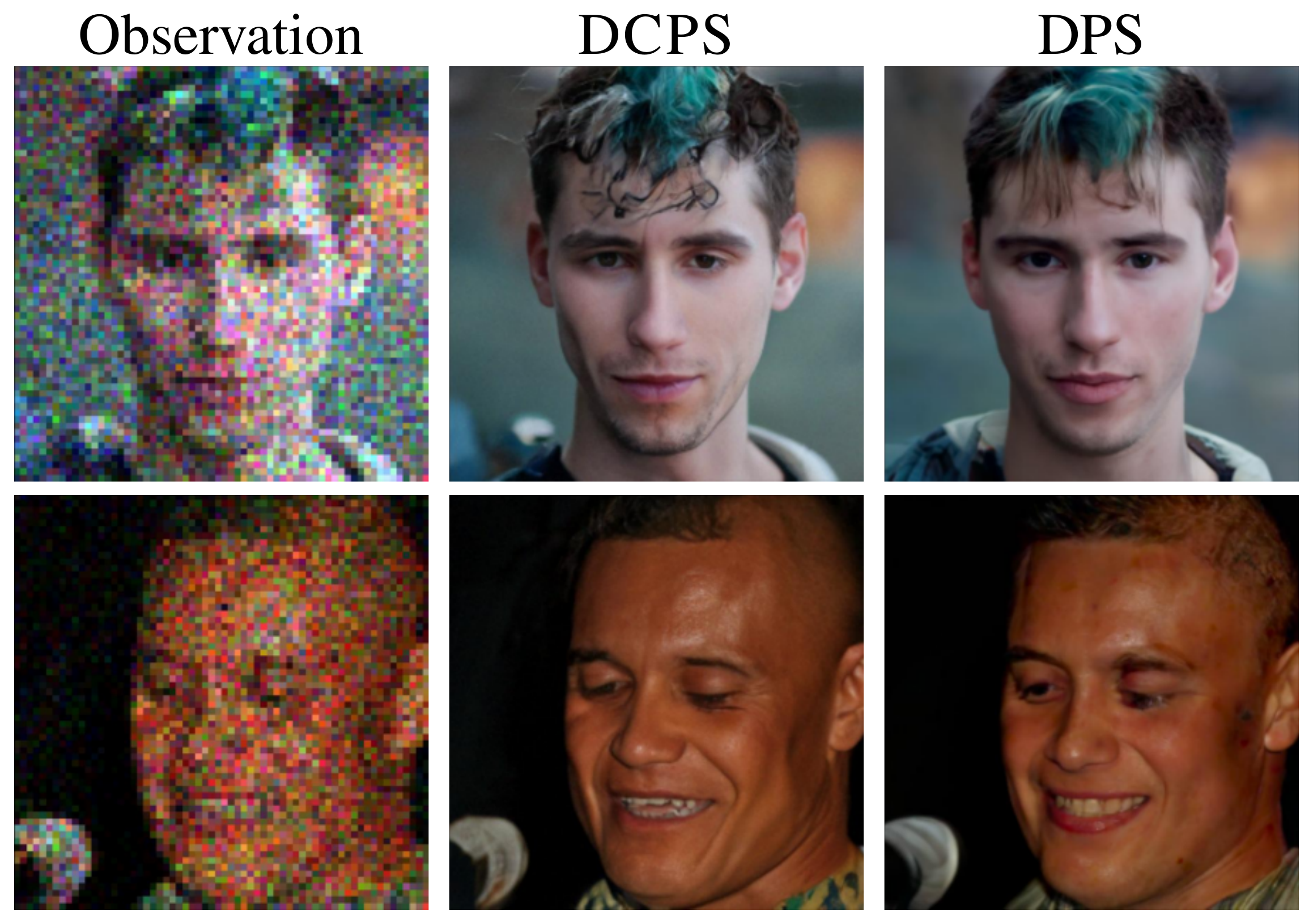}
    }
    \captionsetup{font=small}
    \caption{Left: JPEG dequantization with QF = $2$. Middle: Poisson denoising. Right: SR $4\times$ Poisson denoising.}
\end{figure}


\paragraph{Trajectory prediction.}
We evaluate our algorithm on the UCY dataset consisting of pedestrian trajectories, encoded as 2D time series with 20 time steps \cite{lerner2007ucy, mangalam2021goals-waypoints, gu2022stochastic, mao2023leapfrog}. We pre-train a trajectory model on this dataset and then use it for trajectory reconstruction tasks.
The model architecture and implementation are detailed in \Cref{apdx:trajectory-exp}.
We focus on the completion of trajectories where only a few timesteps are observed. The missing steps are filled in based on the observations and the pre-trained prior model, similar to the inpainting task in the previous section.
We use \mcgdiff\ with $5000$ particles to obtain approximate samples from the posterior.
Indeed, as the dimension of the observation space is low ($\dimx = 40$) and \mcgdiff\ is asymptotically exact as the number of particles tends to infinity, it yields an accurate approximation of the posterior; see \citep[Proposition~2.1]{cardoso2023monte}.
Then, we compute the $\ell_2$ distance between the median, quantile $25$, and quantile $75$ of the \mcgdiff\ samples and the reconstructions of each algorithm.
We report these results in \Cref{table:trajectory-inpainting-L2}.
Finally, in \Cref{fig:trajectory-reconstruction-and-ci} we illustrate the reconstructed trajectories on a specific trajectory completion problem.
\begin{figure}[H]
    \centering
    \includegraphics[width=.75\textwidth]{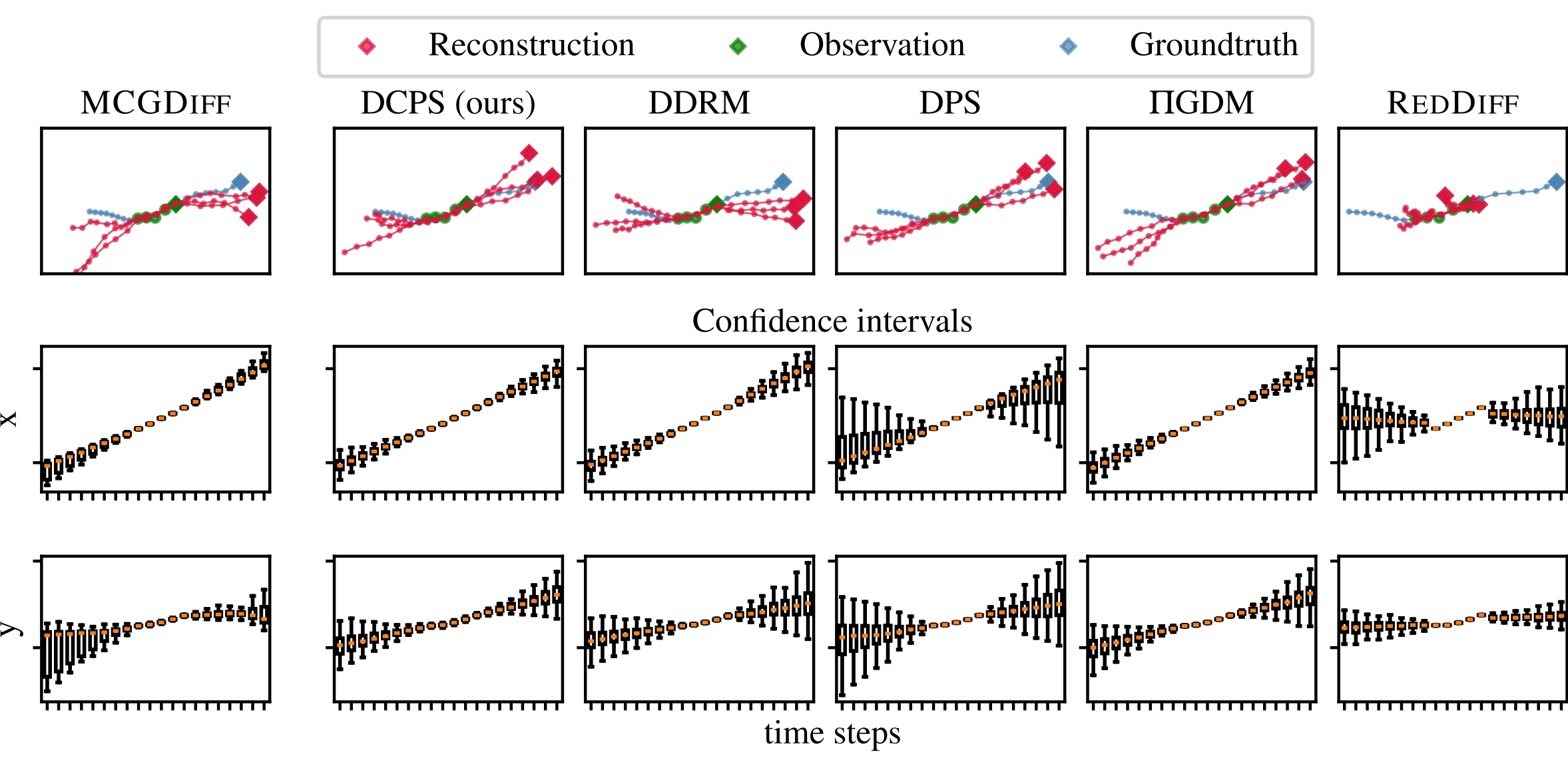}
    \captionsetup{font=small}
    \caption{Trajectory completion where only the middle part of the trajectory is observed.
    The figures in the $1$\textsuperscript{st} row display $3$ reconstructions per algorithm.
    The $2$\textsuperscript{nd} and $3$\textsuperscript{rd} rows show confidence intervals across different time steps. The \emph{Groundtruth} is a trajectory taken from the \texttt{UCY} dataset.}
    \label{fig:trajectory-reconstruction-and-ci}
\end{figure}

 \section{Conclusion.}
 In this paper, we introduce \algo\ to handle Bayesian linear inverse problems with DDM priors without the need for problem-specific additional training.
 Our divide-and-conquer strategy helps to reduce the approximation error of existing approaches, and our variational framework provides a principled method for estimating the backward kernels. \algo\ applies to various relevant inverse problems and is competitive with existing methods. 

\paragraph{Limitations and future directions.}
Our method has some limitations that shed light on opportunities for further development and refinement.
First, the intermediate potentials that we considered were specifically designed for each problem, meaning our method is not universally applicable to all inverse problems. For instance, our approach can not be applied to for linear inverse problems using latent diffusion models \cite{rombach2022high} since there is no clear choice of intermediate potentials.  Therefore, in our opinion, deriving a learning procedure that is capable to automatically design effective intermediate potentials applicable to any $\pot{0}{}{}$ is an important research direction. Moreover, there is an aspect of the choice of the intermediate potentials and the number of blocks $L$ that remains to be understood properly. Indeed, while our backward approximations reduce the local approximation errors \wrt\ \DPS\ and \PGDM; nonetheless \algo\ requires appropriate intermediate potentials in order to perform well. \algo\  can still provide decent performance with \emph{irrelevant} intermediate potentials as long as the number of Langevin steps, in-between the blocks, is large enough.
Finally, although our method provides decent results with the same computational cost as \DPS\ and \PGDM, it remains slower than \reddiff\ and \ddrm\ which which do not compute vector-jacobian product over the denoiser.
Therefore, overcoming this bottleneck when optimizing the KL objective would be a significant improvement for our method.

\paragraph{Acknowledgments.} The work of Y.J. and B.M. has been supported by Technology Innovation Institute (TII), project Fed2Learn. The work of Eric Moulines has been partly funded by the European Union (ERC-2022-SYG-OCEAN-101071601). Views and opinions expressed are however those of the author(s) only and do not necessarily reflect those of the European Union or the European Research Council Executive Agency. Neither the European Union nor the granting authority can be held responsible for them.

\newpage
\bibliographystyle{plain}
\bibliography{bibliography}

\begin{thebibliography}{10}

\bibitem{batzolis2021conditional}
Georgios Batzolis, Jan Stanczuk, Carola-Bibiane Sch{\"o}nlieb, and Christian
  Etmann.
\newblock Conditional image generation with score-based diffusion models.
\newblock {\em arXiv preprint arXiv:2111.13606}, 2021.

\bibitem{bickel:li:bengtsson:2008}
P.~Bickel, B.~Li, and T.~Bengtsson.
\newblock Sharp failure rates for the bootstrap particle filter in high
  dimensions.
\newblock In B.~Clarke and S.~Ghosal, editors, {\em Pushing the Limits of
  Contemporary Statistics: Contributions in Honor of Jayanta K. Ghosh}, pages
  318--329. Institute of Mathematical Statistics, 2008.

\bibitem{bishop2006pattern}
Christopher~M. Bishop.
\newblock {\em Pattern Recognition and Machine Learning (Information Science
  and Statistics)}.
\newblock Springer-Verlag, Berlin, Heidelberg, 2006.

\bibitem{boys2023tweedie}
Benjamin Boys, Mark Girolami, Jakiw Pidstrigach, Sebastian Reich, Alan Mosca,
  and O~Deniz Akyildiz.
\newblock Tweedie moment projected diffusions for inverse problems.
\newblock {\em arXiv preprint arXiv:2310.06721}, 2023.

\bibitem{brosse2019tamed}
Nicolas Brosse, Alain Durmus, {\'E}ric Moulines, and Sotirios Sabanis.
\newblock The tamed unadjusted langevin algorithm.
\newblock {\em Stochastic Processes and their Applications},
  129(10):3638--3663, 2019.

\bibitem{calvetti2018inverse}
Daniela Calvetti and Erkki Somersalo.
\newblock Inverse problems: From regularization to {B}ayesian inference.
\newblock {\em Wiley Interdisciplinary Reviews: Computational Statistics},
  10(3):e1427, 2018.

\bibitem{cardoso2023monte}
Gabriel Cardoso, Yazid Janati, Eric Moulines, and Sylvain~Le Corff.
\newblock Monte carlo guided denoising diffusion models for bayesian linear
  inverse problems.
\newblock In {\em The Twelfth International Conference on Learning
  Representations}, 2024.

\bibitem{choi2021ilvr}
Jooyoung Choi, Sungwon Kim, Yonghyun Jeong, Youngjune Gwon, and Sungroh Yoon.
\newblock Ilvr: Conditioning method for denoising diffusion probabilistic
  models.
\newblock {\em arXiv preprint arXiv:2108.02938}, 2021.

\bibitem{chopin2020introduction}
Nicolas Chopin, Omiros Papaspiliopoulos, et~al.
\newblock {\em An introduction to sequential {M}onte {C}arlo}, volume~4.
\newblock Springer, 2020.

\bibitem{chung2023diffusion}
Hyungjin Chung, Jeongsol Kim, Michael~Thompson Mccann, Marc~Louis Klasky, and
  Jong~Chul Ye.
\newblock Diffusion posterior sampling for general noisy inverse problems.
\newblock In {\em The Eleventh International Conference on Learning
  Representations}, 2023.

\bibitem{chung2022come}
Hyungjin Chung, Byeongsu Sim, and Jong~Chul Ye.
\newblock Come-closer-diffuse-faster: Accelerating conditional diffusion models
  for inverse problems through stochastic contraction.
\newblock In {\em Proceedings of the IEEE/CVF Conference on Computer Vision and
  Pattern Recognition}, pages 12413--12422, 2022.

\bibitem{dashti2017bayesian}
Masoumeh Dashti and Andrew~M. Stuart.
\newblock {\em The Bayesian Approach to Inverse Problems}, pages 311--428.
\newblock Springer International Publishing, Cham, 2017.

\bibitem{del2004feynman}
Pierre Del~Moral.
\newblock Feynman-kac formulae.
\newblock In {\em Feynman-Kac Formulae}, pages 47--93. Springer, 2004.

\bibitem{deng2009imagenet}
Jia Deng, Wei Dong, Richard Socher, Li-Jia Li, Kai Li, and Li~Fei-Fei.
\newblock Imagenet: A large-scale hierarchical image database.
\newblock In {\em 2009 IEEE conference on computer vision and pattern
  recognition}, pages 248--255. Ieee, 2009.

\bibitem{dhariwal2021diffusion}
Prafulla Dhariwal and Alexander Nichol.
\newblock Diffusion models beat gans on image synthesis.
\newblock {\em Advances in neural information processing systems},
  34:8780--8794, 2021.

\bibitem{dou2024diffusion}
Zehao Dou and Yang Song.
\newblock Diffusion posterior sampling for linear inverse problem solving: A
  filtering perspective.
\newblock In {\em The Twelfth International Conference on Learning
  Representations}, 2024.

\bibitem{finzi2023user}
Marc~Anton Finzi, Anudhyan Boral, Andrew~Gordon Wilson, Fei Sha, and Leonardo
  Zepeda-N{\'u}{\~n}ez.
\newblock User-defined event sampling and uncertainty quantification in
  diffusion models for physical dynamical systems.
\newblock In {\em International Conference on Machine Learning}, pages
  10136--10152. PMLR, 2023.

\bibitem{gelfand2000gibbs}
Alan~E Gelfand.
\newblock Gibbs sampling.
\newblock {\em Journal of the American statistical Association},
  95(452):1300--1304, 2000.

\bibitem{gu2022stochastic}
Tianpei Gu, Guangyi Chen, Junlong Li, Chunze Lin, Yongming Rao, Jie Zhou, and
  Jiwen Lu.
\newblock Stochastic trajectory prediction via motion indeterminacy diffusion.
\newblock In {\em Proceedings of the IEEE/CVF Conference on Computer Vision and
  Pattern Recognition}, pages 17113--17122, 2022.

\bibitem{guo2019agem}
Bichuan Guo, Yuxing Han, and Jiangtao Wen.
\newblock Agem: Solving linear inverse problems via deep priors and sampling.
\newblock {\em Advances in Neural Information Processing Systems}, 32, 2019.

\bibitem{ho2020denoising}
Jonathan Ho, Ajay Jain, and Pieter Abbeel.
\newblock Denoising diffusion probabilistic models.
\newblock {\em Advances in Neural Information Processing Systems},
  33:6840--6851, 2020.

\bibitem{hyvarinen2007some}
Aapo Hyv{\"a}rinen.
\newblock Some extensions of score matching.
\newblock {\em Computational statistics \& data analysis}, 51(5):2499--2512,
  2007.

\bibitem{jing2022subspace}
Bowen Jing, Gabriele Corso, Renato Berlinghieri, and Tommi Jaakkola.
\newblock Subspace diffusion generative models.
\newblock In {\em European Conference on Computer Vision}, pages 274--289.
  Springer, 2022.

\bibitem{karras2019style}
Tero Karras, Samuli Laine, and Timo Aila.
\newblock A style-based generator architecture for generative adversarial
  networks.
\newblock In {\em Proceedings of the IEEE/CVF conference on computer vision and
  pattern recognition}, pages 4401--4410, 2019.

\bibitem{kawar2022denoising}
Bahjat Kawar, Michael Elad, Stefano Ermon, and Jiaming Song.
\newblock Denoising diffusion restoration models.
\newblock {\em Advances in Neural Information Processing Systems},
  35:23593--23606, 2022.

\bibitem{kingma2014adam}
Diederik~P Kingma and Jimmy Ba.
\newblock Adam: A method for stochastic optimization.

\bibitem{kingma2013auto}
Diederik~P Kingma and Max Welling.
\newblock Auto-encoding variational bayes.
\newblock {\em arXiv preprint arXiv:1312.6114}, 2013.

\bibitem{lerner2007ucy}
Alon Lerner, Yiorgos Chrysanthou, and Dani Lischinski.
\newblock Crowds by example.
\newblock In {\em Computer graphics forum}, 2007.

\bibitem{loshchilov2016sgdr}
Ilya Loshchilov and Frank Hutter.
\newblock Sgdr: Stochastic gradient descent with warm restarts.
\newblock {\em arXiv preprint arXiv:1608.03983}, 2016.

\bibitem{lugmayr2022repaint}
Andreas Lugmayr, Martin Danelljan, Andres Romero, Fisher Yu, Radu Timofte, and
  Luc Van~Gool.
\newblock Repaint: Inpainting using denoising diffusion probabilistic models.
\newblock In {\em Proceedings of the IEEE/CVF Conference on Computer Vision and
  Pattern Recognition}, pages 11461--11471, 2022.

\bibitem{mangalam2021goals-waypoints}
Karttikeya Mangalam, Yang An, Harshayu Girase, and Jitendra Malik.
\newblock From goals, waypoints \& paths to long term human trajectory
  forecasting.
\newblock In {\em IEEE/CVF}, 2021.

\bibitem{mao2023leapfrog}
Weibo Mao, Chenxin Xu, Qi~Zhu, Siheng Chen, and Yanfeng Wang.
\newblock Leapfrog diffusion model for stochastic trajectory prediction.
\newblock In {\em Proceedings of the IEEE/CVF Conference on Computer Vision and
  Pattern Recognition}, pages 5517--5526, 2023.

\bibitem{marais2017proximal}
Willem Marais and Rebecca Willett.
\newblock Proximal-gradient methods for poisson image reconstruction with
  bm3d-based regularization.
\newblock In {\em 2017 IEEE 7th International Workshop on Computational
  Advances in Multi-Sensor Adaptive Processing (CAMSAP)}, pages 1--5. IEEE,
  2017.

\bibitem{mardani2024a}
Morteza Mardani, Jiaming Song, Jan Kautz, and Arash Vahdat.
\newblock A variational perspective on solving inverse problems with diffusion
  models.
\newblock In {\em The Twelfth International Conference on Learning
  Representations}, 2024.

\bibitem{neal2001annealed}
Radford~M Neal.
\newblock Annealed importance sampling.
\newblock {\em Statistics and computing}, 11:125--139, 2001.

\bibitem{nichol2021improved-diffusion-models}
Alexander~Quinn Nichol and Prafulla Dhariwal.
\newblock Improved denoising diffusion probabilistic models.
\newblock In {\em International conference on machine learning}, pages
  8162--8171. PMLR, 2021.

\bibitem{nowak2000statistical}
Robert~D Nowak and Eric~D Kolaczyk.
\newblock A statistical multiscale framework for poisson inverse problems.
\newblock {\em IEEE Transactions on Information Theory}, 46(5):1811--1825,
  2000.

\bibitem{pan2021exploiting}
Xingang Pan, Xiaohang Zhan, Bo~Dai, Dahua Lin, Chen~Change Loy, and Ping Luo.
\newblock Exploiting deep generative prior for versatile image restoration and
  manipulation.
\newblock {\em IEEE Transactions on Pattern Analysis and Machine Intelligence},
  44(11):7474--7489, 2021.

\bibitem{pitt1999filtering}
Michael~K Pitt and Neil Shephard.
\newblock Filtering via simulation: Auxiliary particle filters.
\newblock {\em J. Amer. Statist. Assoc.}, 94(446):590--599, 1999.

\bibitem{roberts:tweedie:1996}
Gareth~O. Roberts and Richard~L. Tweedie.
\newblock Geometric convergence and central limit theorems for multidimensional
  {H}astings and {M}etropolis algorithms.
\newblock {\em Biometrika}, 83:95--110, 1996.

\bibitem{rodrigues2008denoising}
Isabel Rodrigues, Joao Sanches, and Jose Bioucas-Dias.
\newblock Denoising of medical images corrupted by poisson noise.
\newblock In {\em 2008 15th IEEE international conference on image processing},
  pages 1756--1759. IEEE, 2008.

\bibitem{romano2017little}
Yaniv Romano, Michael Elad, and Peyman Milanfar.
\newblock The little engine that could: Regularization by denoising (red).
\newblock {\em SIAM Journal on Imaging Sciences}, 10(4):1804--1844, 2017.

\bibitem{rombach2022high}
Robin Rombach, Andreas Blattmann, Dominik Lorenz, Patrick Esser, and Bj{\"o}rn
  Ommer.
\newblock High-resolution image synthesis with latent diffusion models.
\newblock In {\em Proceedings of the IEEE/CVF Conference on Computer Vision and
  Pattern Recognition}, pages 10684--10695, 2022.

\bibitem{rozet2023sda}
Fran{\c{c}}ois Rozet and Gilles Louppe.
\newblock Score-based data assimilation.
\newblock {\em Advances in Neural Information Processing Systems},
  36:40521--40541, 2023.

\bibitem{saharia2022image}
Chitwan Saharia, Jonathan Ho, William Chan, Tim Salimans, David~J Fleet, and
  Mohammad Norouzi.
\newblock Image super-resolution via iterative refinement.
\newblock {\em IEEE Transactions on Pattern Analysis and Machine Intelligence},
  45(4):4713--4726, 2022.

\bibitem{shin2017jpeg}
Richard Shin and Dawn Song.
\newblock Jpeg-resistant adversarial images.
\newblock In {\em NIPS 2017 workshop on machine learning and computer
  security}, volume~1, page~8, 2017.

\bibitem{sohl2015deep}
Jascha Sohl-Dickstein, Eric Weiss, Niru Maheswaranathan, and Surya Ganguli.
\newblock Deep unsupervised learning using nonequilibrium thermodynamics.
\newblock In {\em International Conference on Machine Learning}, pages
  2256--2265. PMLR, 2015.

\bibitem{song2021denoising}
Jiaming Song, Chenlin Meng, and Stefano Ermon.
\newblock Denoising diffusion implicit models.
\newblock In {\em International Conference on Learning Representations}, 2021.

\bibitem{song2022pseudoinverse}
Jiaming Song, Arash Vahdat, Morteza Mardani, and Jan Kautz.
\newblock Pseudoinverse-guided diffusion models for inverse problems.
\newblock In {\em International Conference on Learning Representations}, 2023.

\bibitem{song2023loss}
Jiaming Song, Qinsheng Zhang, Hongxu Yin, Morteza Mardani, Ming-Yu Liu, Jan
  Kautz, Yongxin Chen, and Arash Vahdat.
\newblock Loss-guided diffusion models for plug-and-play controllable
  generation.
\newblock In {\em International Conference on Machine Learning}, pages
  32483--32498. PMLR, 2023.

\bibitem{song2021maximum}
Yang Song, Conor Durkan, Iain Murray, and Stefano Ermon.
\newblock Maximum likelihood training of score-based diffusion models.
\newblock {\em Advances in Neural Information Processing Systems},
  34:1415--1428, 2021.

\bibitem{song2019generative}
Yang Song and Stefano Ermon.
\newblock Generative modeling by estimating gradients of the data distribution.
\newblock {\em Advances in neural information processing systems}, 32, 2019.

\bibitem{song2021score}
Yang Song, Jascha Sohl-Dickstein, Diederik~P Kingma, Abhishek Kumar, Stefano
  Ermon, and Ben Poole.
\newblock Score-based generative modeling through stochastic differential
  equations.
\newblock In {\em International Conference on Learning Representations}, 2021.

\bibitem{stuart2010inverse}
Andrew~M Stuart.
\newblock Inverse problems: a {B}ayesian perspective.
\newblock {\em Acta numerica}, 19:451--559, 2010.

\bibitem{trippe2023diffusion}
Brian~L. Trippe, Jason Yim, Doug Tischer, David Baker, Tamara Broderick, Regina
  Barzilay, and Tommi~S. Jaakkola.
\newblock Diffusion probabilistic modeling of protein backbones in 3d for the
  motif-scaffolding problem.
\newblock In {\em The Eleventh International Conference on Learning
  Representations}, 2023.

\bibitem{ulyanov2018deep}
Dmitry Ulyanov, Andrea Vedaldi, and Victor Lempitsky.
\newblock Deep image prior.
\newblock In {\em Proceedings of the IEEE conference on computer vision and
  pattern recognition}, pages 9446--9454, 2018.

\bibitem{vaswani2017attention}
Ashish Vaswani, Noam Shazeer, Niki Parmar, Jakob Uszkoreit, Llion Jones,
  Aidan~N Gomez, Lukasz Kaiser, and Illia Polosukhin.
\newblock Attention is all you need.
\newblock {\em Advances in neural information processing systems}, 30, 2017.

\bibitem{wallace1992jpeg}
Gregory~K Wallace.
\newblock The jpeg still picture compression standard.
\newblock {\em IEEE transactions on consumer electronics}, 38(1):xviii--xxxiv,
  1992.

\bibitem{wang2023ddnm}
Yinhuai Wang, Jiwen Yu, and Jian Zhang.
\newblock Zero-shot image restoration using denoising diffusion null-space
  model.
\newblock In {\em The Eleventh International Conference on Learning
  Representations}, 2023.

\bibitem{wu2023practical}
Luhuan Wu, Brian~L. Trippe, Christian~A Naesseth, John~Patrick Cunningham, and
  David Blei.
\newblock Practical and asymptotically exact conditional sampling in diffusion
  models.
\newblock In {\em Thirty-seventh Conference on Neural Information Processing
  Systems}, 2023.

\bibitem{zhang2023towards}
Guanhua Zhang, Jiabao Ji, Yang Zhang, Mo~Yu, Tommi Jaakkola, and Shiyu Chang.
\newblock Towards coherent image inpainting using denoising diffusion implicit
  models.
\newblock In {\em International Conference on Machine Learning}, pages
  41164--41193. PMLR, 2023.

\bibitem{zhang2018unreasonable}
Richard Zhang, Phillip Isola, Alexei~A Efros, Eli Shechtman, and Oliver Wang.
\newblock The unreasonable effectiveness of deep features as a perceptual
  metric.
\newblock In {\em Proceedings of the IEEE conference on computer vision and
  pattern recognition}, pages 586--595, 2018.

\bibitem{zhu2023diffpir}
Yuanzhi Zhu, Kai Zhang, Jingyun Liang, Jiezhang Cao, Bihan Wen, Radu Timofte,
  and Luc Van~Gool.
\newblock Denoising diffusion models for plug-and-play image restoration.
\newblock In {\em Proceedings of the IEEE/CVF Conference on Computer Vision and
  Pattern Recognition}, pages 1219--1229, 2023.

\end{thebibliography}

\clearpage
\newpage

 \begin{appendix}
    \section{Methodology details}
\label{sec:DDIM}
\subsection{Denoising Diffusion models}
\label{apdx:ddim}
DDMs learn a sequence $(\predx{t}[\param])_{t = 1}^T$ of denoisers by minimizing, using SGD, the objective
\begin{equation}
    \label{eq:denoising_objective}
 \sum_{t = 1}^T  w_t \pE \left[ \| \epsilon_t - \prednoise{t}(\sqrt{\acp{t}} X_0 + \sqrt{1 - \acp{t}} \epsilon_t ) \|^2 \right]
\end{equation}
w.r.t. the neural network parameter $\param$,
where $(\epsilon_t)_{t = 1}^T$ are i.i.d. standard normal vectors and $(w_t)_{t = 1}^T$ are some nonnegative weights. 
We denote by $\param^\star$ an estimator of the minimizer of the previous loss.
Having access to $\param^\star$, we can define a generative model for $\datadistr$. Let $( t_{k} )_{k = 0} ^\last$ be an increasing sequence of time instants in $\intset{0}{T}$ with $t_0 = 0$.
We assume that $t_n$ is large enough so that $\fwmarg{t_n}{}$ is approximately multivariate standard normal.
For convenience, we assign the index $k$ to any quantity depending on $t_k$; \emph{e.g.}, we denote $p_{t_k}$ by $p_k$.
For $(j, k) \in \intset{1}{n-1}^2$ such that $j < k$, define
\begin{align}
    \label{eq:mean_std_bridge}
    \muDDIM{j|0, k}(x_0, x_k) & \eqdef \frac{\sqrt{\acp{j}{}}(1 - \acp{k} / \acp{j})}{1 - \acp{k}} x_0 + \frac{\sqrt{\acp{k} / \acp{j}} (1 - \acp{j})}{1 - \acp{k}} x_k \eqsp, \\
    \sigma^2 _{j|k} & \eqdef \frac{(1 - \acp{j})(1 - \acp{k} / \acp{j})}{1 - \acp{k}} \eqsp.
\end{align}
Then the bridge kernel
\begin{equation} 
    \label{eq:bridge_kernel}
    \rbwker{j|0, k}{x_0, x_k}{x_j} = \fwtrans{j|0}{x_0}{x_j} \fwtrans{k|j}{x_j}{x_k} \big/ \fwtrans{k|0}{x_0}{x_k}
\end{equation}
is a Gaussian distribution with mean $\muDDIM{j|0, k}(x_0, x_k)$ and covariance $\sigma^2 _{j|k} \Id_\dimx$. 
DDPM \cite{ho2020denoising} posits the following variational approximation 
$$ 
    \bwmarg{0:n}{x_{0:n}}[\param] = \bwmarg{n}{x_n} \prod_{k = 0}^{n-1} \abwker{k|k+1}{x_{k+1}}{x_k}[\param] \eqsp,
$$
where $\abwker{k|k+1}{x_{k+1}}{x_k}[\param] = \rbwker{k|0, k+1}{\predx{k+1}[\param](x_{k+1}), x_{k+1}}{x_k}$ and $\abwker{0|1}{x_1}{\cdot}[\param] = \delta_{\predx{1}[\param](x_1)}$. An efficient generative model is then obtained by plugging in the parameter $\param^\star$.

\subsection{Further details on \algo}
\label{apdx:algo}
In this section we provide further details on \textbf{Steps 1} and \textbf{2} detailed in the main paper. The complete algorithm is given in \Cref{algo:algo}.

\paragraph{Tamed unadjusted Langevin.} For the tamed unadjusted Langevin steps we simulate the Markov chain $(\tilde{X}_{j})_{j = 0} ^M$ where 
\begin{equation}
   \label{eq:tula_scheme}
  \tilde{X}_{j+1} = \tilde{X}_j + \gamma G^{\ell} _\gamma(\tilde{X}_j) + \sqrt{2 \gamma} Z_j \eqsp, \quad \tilde{X}_0 = X _\ell+1 \eqsp,
\end{equation}
and $(Z_j)_{j = 0}^{M - 1}$ are i.i.d. $\dimx$-dimensional standard normal, $X _\ell+1$ is an approximate sample from $\cbwmarg{\ell+1}{}$ obtained from the previous iteration of the algorithm, and for all $x \in \rset^\dimx$ and $\gamma > 0$,
\begin{equation}
   \label{eq:tula_gradient}
   G^{\ell}_\gamma(x) \eqdef \frac{\nabla \log \iapot{\ell+1}{\ell, \star}{x} + \score{\ell+1}(x)}{1 + \gamma \| \nabla \log \iapot{\ell+1}{\ell, \star}{x} + \score{\ell+1}(x) \|} \eqsp.
\end{equation}
We then set $X^{\ell} _{\ell+1} \eqdef \tilde{X}_M$, which serves as an initialization of the Markov chain in \textbf{Step 2}. 

\paragraph{Potential computation.} In order to perform the tamed Langevin steps and to optimize the variational approximation using the criterion \eqref{eq:variational-criterion}, it is crucial to be able to compute exactly the potential \eqref{eq:opt_pot_approx}. The optimal potentials we have proposed for both linear inverse problems with Gaussian noise \eqref{eq:intermediatepot-lininvp} and low-count Poisson denoising \eqref{eq:intermediatepot-poisson} (for $\ell > 0$) are available in a closed form:
\begin{equation} 
   \label{eq:int_pot}
   \iapot{j}{\ell, \star}{x_j} = \normpdf(\sqrt{\acp{\ell}}\, \obs, A \muDDIM{\ell|j}(x_j), \Sigma^\ell_j) \eqsp,
\end{equation}
where 
\begin{align} 
   \Sigma^\ell _j & = \sigma^2 _{\ell|j} A A^\intercal + \stdobs^2 \Id_\dimobs \eqsp,  \tag{Linear inverse problem} \\
   \Sigma^\ell _j & = \sigma^2 _{\ell|j} A A^\intercal + \sqrt{\acp{\ell}}\mathrm{diag}(\obs) \eqsp, \quad \ell > 0 \eqsp, \tag{Poisson-shot noise}
\end{align}
$\muDDIM{\ell|j}(x_j) \eqdef \muDDIM{\ell|0, j}(\predx{j}(x_j), x_j)$, and $\sigma^2 _{\ell | j}$ is defined in \eqref{eq:mean_std_bridge}. As a result, the first term of the variational criterion $\mathcal{L}(\vmu_j, \vlogstd_j; x_{j+1})$ in \eqref{eq:variational-criterion}, given by 
$$ 
\pE \big[ \log \iapot{j}{\ell, \star}{\vmu_j + \rme^{\vlogstd_j / 2}  Z}\big] = \int \log \iapot{j}{\ell, \star}{x_j} \vbwker{j|j+1}{x_{j+1}
}{x_j} \, \rmd x_j,
$$
can be computed exactly. Indeed, as $\muDDIM{\ell|j}$ is a linear function of $x_j$, this expectation is simply that of a quadratic function under a Gaussian density, given by 
\begin{equation*} 
   \pE \big[ \log \iapot{j}{\ell, \star}{\vmu_j + \rme^{\vlogstd_j / 2}  Z}\big] = - \frac{1}{2} \bigg[ \big\| \sqrt{\acp{\ell}}\, \obs - A \muDDIM{\ell|j}(\vmu_{j})\big\|^2 _{(\Sigma^{\ell} _j)^{-1}} + \mathrm{tr}\big((\Sigma^\ell _j)^{-1} \mathrm{diag}(\rme^{\vlogstd_j})\big) \bigg] + C \eqsp.
\end{equation*}
Hence, for these cases, \eqref{eq:variational-criterion} has a closed-form expression. However, it involves the computation of an inverse matrix which, for many problems, can be prohibitively expensive. To avoid this inversion, we instead optimize a \emph{biased} estimate of $\mathcal{L}_j(\vmu_j, \vlogstd_j; x_{j+1})$ obtained by drawing two noise vectors $(Z, Z^\prime) \sim \normpdf(\zero_\dimx, \Id_\dimx)$ and setting 
\begin{multline} 
   \label{eq:biased-estimate}
   \widetilde{\mathcal{L}}_j(\vmu_j, \vlogstd_j; x_{j+1}) \eqdef -  \log \pot{\ell}{}{\muDDIM{\ell|j}(\vmu_j + \rme^{\vlogstd_j / 2} Z) + \sigma^2 _{\ell|j} Z^\prime} \\
   + \frac{\| \vmu_j - \muDDIM{j|j+1}(x_{j+1})\|^2}{2 \sigma^2 _{j|j+1}} - \frac{1}{2}\sum_{i = 1}^\dimx \left( \vlogstd_{j, i} - \frac{\rme^{\vlogstd_{j, i}}}{\sigma^2 _{j|j+1}} \right) \eqsp.
\end{multline}
This estimator is computable for any choice choice of potential and we have found in practice that it is sufficient to ensure good enough performance for our algorithm. Regarding the tamed unadjusted Langevin steps, we use the same biased estimate when the matrix inversions are expensive to compute; \emph{i.e.} at each Langevin step, we approximate $G^\ell_\gamma(\tilde{X}_j)$ by 
\begin{equation} 
   \label{eq:grad-estimate}
      \widetilde{G}^\ell _\gamma(\tilde{X}_j) \eqdef \frac{\nabla_{x_\ell+1} \log \pot{\ell}{}{\muDDIM{\ell|\ell+1}(x_{\ell+1}) + \sigma _{\ell|\ell+1} \tilde{Z}_\ell} + \score{\ell+1}(x_\ell+1)}{\| \nabla_{x_\ell+1} \log \pot{\ell}{}{\muDDIM{\ell|\ell+1}(x_{\ell+1}) + \sigma _{\ell|\ell+1} \tilde{Z}_\ell} + \score{\ell+1}(x_\ell+1) \|} \eqsp.
\end{equation}
\begin{algorithm}[H]
    \caption{\algoname\ (\algo)}
    \begin{algorithmic}
       \STATE {\bfseries Input:}  timesteps $(\tstep{\ell})_{\ell = 0} ^\tauks$, learning-rate $\zeta$,  
       numbers $K$ and $M$ of gradient and Langevin steps, respectively.
       \\
        Initial sample $X_\tstep{\tauks} \sim \gauss(\zero_\dimx, \Id_\dimx)$;
        \FOR{$\ell=\tauks - 1$ {\bfseries to} $0$}
            \STATE Draw $Z \sim \normpdf(\zero_\dimx, \Id_\dimx)$ and compute $\widetilde{G}^\ell _\gamma(X^\ell _\tstep{\ell+1})$ \eqref{eq:grad-estimate};
            \STATE $X^\ell _\tstep{\ell+1} \leftarrow X_\tstep{\ell+1}$
            \FOR{$i=1$ {\bfseries to} $M$}
                \STATE $Z \sim \normpdf(\zero_\dimx, \Id_\dimx)$;
                \STATE $X^\ell _\tstep{\ell+1} \leftarrow X^\ell _\tstep{\ell+1} + \gamma \widetilde{G}^\ell _\gamma(X^\ell _\tstep{\ell+1}) + \sqrt{2 \gamma} Z$;
            \ENDFOR
            \FOR{$j = \tstep{\ell+1} - 1$ {\bfseries to} $\tstep{\ell}$}
        \STATE $\vmu _j \leftarrow \muDDIM{j|j+1}(X^{\ell} _{j+1})$; $\quad \vlogstd _j \leftarrow \log \sigma^2 _{j|j+1} \cdot \mathbf{1}_\dimx$;
        \FOR{$r = 1$ {\bfseries to} $K$}
            \STATE Draw $(Z, Z^\prime) \sim \normpdf(\zero_\dimx, \Id_\dimx)$ and compute $\widetilde{\mathcal{L}}_j(\vmu_j, \vlogstd_j; X^\ell _{j+1})$ \eqref{eq:biased-estimate};
            \STATE $\begin{bmatrix} \vmu_j \\
               \vlogstd_j \end{bmatrix} \leftarrow \begin{bmatrix} \vmu_j \\
                  \vlogstd_j \end{bmatrix} - \zeta \|\nabla_{\vmu_j, \vlogstd_j} \widetilde{\mathcal{L}}_j(\vmu_j, \vlogstd_j; X^\ell _{j+1})\|^{-1} \nabla_{\vmu_j, \vlogstd_j} \widetilde{\mathcal{L}}_j(\vmu_j, \vlogstd_j; X^\ell _{j+1})$
        \ENDFOR
        \STATE $\varepsilon \sim \gauss(\zero_\dimx, \Id_\dimx)$
        \STATE  $X^{\ell} _j \leftarrow \vmu_j + \mathrm{diag}(\rme^{\vlogstd_j / 2})\varepsilon$;
        \vspace{0.02cm}
            \ENDFOR
        \STATE $X _\tstep{\ell} \leftarrow X^{\ell} _\tstep{\ell}$;
        \ENDFOR
    \end{algorithmic}
    \label{algo:algo}
   \end{algorithm}

   \subsection{Proof of \Cref{prop:w2-informal}}
   \label{apdx:proofw2}
   For all $k \in \intset{0}{n-1}$ we denote by $\rbwker{k|k+1}{x_{k+1}}{x_k}$ the \emph{exact} backward kernel which satisfies 
   \begin{equation}
      \label{eq:time-reversal}
       \fwmarg{k+1}{x_{k+1}} \rbwker{k|k+1}{x_{k+1}}{x_k} = \fwmarg{k}{x_{k}} \rbwker{k+1|k}{x_{k}}{x_{k+1}} \eqsp.
   \end{equation}
   Note that the backward kernels $\abwker{k|k+1}{}{}$ are to be understood as Gaussian approximations of the true backward kernels $\rbwker{k|k+1}{}{}$. Below we give a complete statement of the proposition and provide a proof. 
   \begin{proposition}
      \label{prop:w2}
     Let $k \in \intset{1}{n}$. Assume that $\rbwker{k|k+1}{x_{k+1}}{x_k} = \abwker{k|k+1}{x_{k+1}}{x_k}$ for all $(x_k, x_{k+1}) \in (\rset^\dimx)^2$. For all $\ell \in \intset{0}{k-1}$ and $x_k \in \rset^\dimx$,
      \begin{equation*}
      W_2(\hbwker{\ell|k}{x_k}{\cdot}, \abwker{\ell|k}{x_k}{\cdot})
      \leq \frac{\sqrt{\acp{\ell}}(1 - \acp{k} / \acp{\ell})}{(1 - \acp{k})}  W_2(\hbwker{0|k}{x_k}{\cdot}, \abwker{0|k}{x_k}{\cdot})\eqsp.
      \end{equation*}
  \end{proposition}
  \begin{proof}[Proof of \Cref{prop:w2}]
      Under the assumptions of the proposition, we have, for all $m > \ell$,
      $$
      \abwker{\ell|k}{x_k}{x_\ell} = \rbwker{\ell|k}{x_k}{x_\ell} = \int \rbwker{\ell|0, k}{x_0, x_k}{x_\ell} \, \rbwker{0|k}{x_k}{\rmd x_0}  \eqsp.
      $$ Indeed, by definition of the backward kernel $\rbwker{0|k}{x_k}{x_0}$ and \eqref{eq:time-reversal}, it holds  that
      \begin{align*}
          \int \rbwker{\ell|0, k}{x_0, x_k}{x_\ell} \rbwker{0|k}{x_k}{x_0}\, \rmd x_0 & = \int \frac{\rbwker{\ell|0}{x_0}{x_\ell} \rbwker{k|\ell}{x_\ell}{x_k}}{\rbwker{k|0}{x_0}{x_k}} \frac{\fwmarg{0}{x_0} \rbwker{k|0}{x_0}{x_k}}{\fwmarg{k}{x_k}} \, \rmd x_0 \\
          & = \frac{\rbwker{k|\ell}{x_\ell}{x_k}}{\fwmarg{k}{x_k}} \int \fwmarg{0}{x_0} \rbwker{\ell|0}{x_0}{\rmd x_\ell} \, \rmd x_0 \\
          & = \rbwker{\ell|k}{x_k}{x_\ell} \eqsp.
      \end{align*}
      As a result, we have that
      \begin{align*}
          \abwker{\ell|k}{x_k}{x_\ell} & = \int \rbwker{\ell|0, k}{x_0, x_k}{\rmd x_\ell} \rbwker{0|k}{x_k}{x_0} \, \rmd x_0 \eqsp,\\
          \hbwker{\ell|k}{x_k}{x_\ell} &= \int \rbwker{\ell|0, k}{x_0, x_k}{\rmd x_\ell} \hbwker{0|k}{x_k}{x_0} \, \rmd x_0 \eqsp,
      \end{align*}
      where, by definition, $\hbwker{0|k}{x_k}{\cdot}$ is a Gaussian approximation of $\rbwker{0|k}{x_k}{\cdot}$ as defined in the main paper.
  
      Next, let $\Pi_{0|k}(\cdot | x_k)$ denote a coupling of $\rbwker{0|k}{x_k}{\cdot}$ and $\hbwker{0|k}{x_k}{\cdot}$, \emph{i.e.}, for all $A \in \mathcal{B}(\rset^\dimx)$,
      \begin{align*}
          \int \1_A(x_0) \1_{\rset^\dimx}(\hat{x}_0) \, \Pi_{0|k}(x_0, \hat{x}_0 | x_k)\, \rmd x_0 \rmd \hat{x}_0 & = \int \1_A(x_0) \, \rbwker{0|k}{x_k}{x_0} \, \rmd x_0 \eqsp,\\
          \int \1_{\rset^\dimx}(x_0) \1_{A}(\hat{x}_0) \, \Pi_{0|k}(x_0, \hat{x}_0 | x_k) \, \rmd x_0 \rmd \hat{x}_0& = \int \1_A(\hat{x}_0) \, \hbwker{0|k}{x_k}{\hat{x}_0} \, \rmd \hat{x}_0 \eqsp.
      \end{align*}
      Consider then the random variables
          \begin{align*}
              X_{\ell|k} & = \frac{\sqrt{\acp{\ell}{}}(1 - \acp{k} / \acp{\ell})}{1 - \acp{k}} X_{0|k} + \frac{\sqrt{\acp{k} / \acp{\ell}} (1 - \acp{\ell})}{1 - \acp{k}} x_k + \frac{\sqrt{(1 - \acp{\ell})(1 - \acp{k} / \acp{\ell})}}{\sqrt{1 - \acp{k}}} Z \eqsp, \\
              \hat{X} _{s|k} & = \frac{\sqrt{\acp{\ell}{}}(1 - \acp{k} / \acp{\ell})}{1 - \acp{k}} \hat{X}_{0|k} + \frac{\sqrt{\acp{k} / \acp{\ell}} (1 - \acp{\ell})}{1 - \acp{k}} x_k + \frac{\sqrt{(1 - \acp{\ell})(1 - \acp{k} / \acp{\ell})}}{\sqrt{1 - \acp{k}}} Z \eqsp,
          \end{align*}
          where $(X_{0|k}, \hat{X}_{0|k}) \sim \Pi_{0|k}(\cdot | x_k)$ and $Z \sim \gauss(\zero_\dimx, \Id_\dimx)$. Then $(X_{\ell|k}, \hat{X}_{\ell|k})$ is distributed according to a coupling of  $\hbwker{\ell|k}{x_k}{\cdot}$ and $\abwker{\ell|k}{x_k}{\cdot}$, and consequently
      \begin{align*}
          W_2(\hbwker{\ell|k}{x_k}{\cdot}, \abwker{\ell|k}{x_k}{\cdot}) & \leq \pE \left[ \| X_{\ell|k} - \hat{X}_{\ell|k} \|^2 \right]^{1/2} \\
          & \leq \frac{\sqrt{\acp{\ell}}(1 - \acp{k} / \acp{\ell})}{(1 - \acp{k})} \pE \left[ \| X_{0|k} - \hat{X}_{0|k} \|^2 \right]^{1/2} \eqsp.
      \end{align*}
      The result is obtained by taking the infinimum of the rhs with respect to all couplings of $\rbwker{0|k}{x_k}{\cdot}$ and $\hbwker{0|k}{x_k}{\cdot}$.
  \end{proof} 

\section{Discussion of related methods}
\label{apdx:related}
In this section we discuss in more details existing works that bear some similarities with \algo. 

\paragraph{SMC based approaches.} The \mcgdiff, the Twisted Diffusion sampler (TDS) of \cite{wu2023practical} using the FK representation \eqref{eq:FK_posterior}. \mcgdiff\ is specific to linear inverse problems and the potentials used are $\pot{k}{}{x_k} = \normpdf(\sqrt{\acp{k}} \, \obs; A x_k, (1 - \acp{k}) \Id_\dimobs)$ when $\stdobs = 0$. TDS applies to any potential $\pot{0}{}{}$ and relies on the \DPS\ approximation for its potentials; i.e. $\smash{\pot{k}{}{x_k} = \pot{0}{}{\predx{k}(x_k)}}$. In either cases, a particle approximation of the posterior of interest $\cbwmarg{0}{}$ is obtained using the Auxiliary Particle filter framework \cite{pitt1999filtering}. \cite{dou2024diffusion} also use particle filters for the posterior distribution; the potentials used are $\smash{\pot{k}{}{x_k} = \normpdf(\obs_k; Ax_k, \acp{k} \stdobs^2 \Id_\dimx)}$ where $(\obs_k)_{k = 0} ^n$, with $\obs_0 = \obs$ is a sequence of observations sampled according to an auto-regressive process; see \citep[Equation 7]{dou2024diffusion}. The posterior is thus viewed as approximately the time $0$ marginal of a Hidden Markov model with transition $\smash{\abwker{k|k+1}{}{}}$ and observation likelihood $\pot{k}{}{}$, which is different from the FK representation \eqref{eq:FK_posterior}. Our choice of intermediate potentials for linear inverse problems with Gaussian noise differs from that of \mcgdiff\ by the standard deviation of the observation model, which we set to be $\stdobs$. A major difference of \algo\ with these works lies in the fact that we do not rely on particle filters, thus avoiding the collapse in very large dimensions. As we have shown in the experimental section \algo\ can achieve comparable performance to \mcgdiff\ in low dimensions, see \Cref{table:gm} while also being efficient in very large dimensions, see \Cref{table:lin_invp}. A second and major difference is that we have derived potentials for both the JPEG dequantization and Poisson-shot denoising tasks, which may be used to extend \mcgdiff\ and FPS-SMC \cite{dou2024diffusion} to these problems. 

\paragraph{RedDiff.} In this work we have also proposed to use Gaussian variational inference to approximate the intractable backward transition $\smash{\pibwker{k|k+1}{}{}[\ell]}$. One particularity of our approach is that we do not use amortized variational inference \cite{kingma2013auto} and instead optimize the variational distribution at each step of the diffusion. A similar approach is used in \reddiff\, \cite{mardani2024a} but in a different way. Indeed, the authors use a \emph{non-amortized} Gaussian variational approximation for the posterior $\cbwmarg{0}{}$, meaning that in order to draw one sample from \reddiff, several steps of optimization are performed on a score-matching-like loss. Interestingly, this approach does not require differentiating through the denoising network and is thus faster and more memory efficient. However, we found that this comes at the cost of performance as can be seen in \Cref{table:gm}, \ref{table:lin_invp} and \ref{table:jpeg-dequantization}. 

\paragraph{RePaint.} The \textsc{RePaint} algorithm \cite{lugmayr2022repaint} has been proposed to sample from the posterior of noiseless linear inverse problems. As we now show, it can be viewed as a specific case of our framework, though employing a different sampling method and intermediate potentials. Indeed, setting $\tauks = \last$ (and hence $k_\ell = \ell$) and using a Gibbs sampler to sample from the consecutive distributions allows us to recover a generalization of the \textsc{RePaint} algorithm \cite{lugmayr2022repaint}. First, note that \eqref{eq:interm_post} is the marginal of the joint distribution 
\begin{equation}
    \label{eq:joint_distr}
\cbwmarg{\ell, \ell+1}{x_{\ell},x_{\ell + 1}} \propto \pot{\ell}{}{x_\ell} \abwker{\ell|\ell+1}{x_{\ell+1}}{x_\ell} \bwmarg{\ell+1}{x_{\ell+1}} \eqsp.
\end{equation}
Hence, we can draw approximate samples from $\cbwmarg{\ell}{}$ using a Gibbs sampler targeting $\cbwmarg{\ell,\ell+1}{}$. A Gibbs sampler \cite{gelfand2000gibbs} constructs a Markov chain $(X^\ell _k, X^{\ell+1} _k)_k$ targeting $\cbwmarg{\ell,\ell+1}{}$ by alternating the sampling of the two conditional distributions of \eqref{eq:joint_distr}
\begin{align*} 
    \pibwker{\ell|\ell+1}{x_{\ell+1}}{x_\ell} & \propto \pot{\ell}{}{x_\ell} \abwker{\ell|\ell+1}{x_{\ell+1}}{x_\ell} \eqsp, \\
    \pibwker{\ell+1|\ell}{x_{\ell}}{x_{\ell+1}} & \propto \abwker{\ell|\ell+1}{x_{\ell+1}}{x_\ell} \bwmarg{\ell+1}{x_{\ell+1}} \eqsp. 
\end{align*}
Then, given $\smash{(X^\ell _{k}, X^{\ell+1} _{k})}$, $\smash{(X^\ell _{k+1}, X^{\ell+1} _{k+1})}$ is obtained by drawing $X^\ell _{k+1}$ from $\smash{\pibwker{\ell|\ell+1}{X^{\ell+1} _{k}}{\cdot}}$ and $X^{\ell+1} _{k+1}$ from $\smash{\pibwker{\ell+1|\ell}{X^\ell _{k+1}}{\cdot}}$. Now let $\pot{\ell}{}{x_\ell} = \normpdf(\obs_\ell; Ax_\ell, \sigma^2 _{\obs, \ell} \Id_\dimobs)$ where $\sigma _{\obs, \ell} > 0$ and $\obs_\ell \in \rset^\dimobs$. Then, the first conditional $\smash{\pibwker{\ell|\ell+1}{}{}}$ can be computed exactly and is given by 
\begin{equation}
    \label{eq:cbwker_taul_giv_taulnext}
\pibwker{\ell|\ell+1}{x_{\ell+1}}{x_\ell} = \normpdf\left(x_\ell; \Sigma _{\ell|\ell+1} \left(   \frac{A^\intercal \obs_\ell}{\sigma^2 _{\obs, \ell}} + \frac{\muDDIM{\ell|\ell+1}(x_\ell+1)}{\sigma^2 _{\ell|\ell+1}}  \right) ,\Sigma _{\ell|\ell+1}\right) \eqsp,
\end{equation}
where $\Sigma _{\ell|\ell+1} \eqdef \big( \sigma^{-2} _{\ell|\ell+1} \Id_\dimx + \sigma^{-2} _{\obs, \ell} \bfA^\intercal \bfA \big)^{-1}$.
As to the second conditional $\pibwker{\ell+1|\ell}{x_{\ell}}{x_{\ell+1}}$ it cannot be sampled exactly and \citet{lugmayr2022repaint} approximate it with the forward kernel $\fwtrans{\ell+1|\ell}{}{}$. This approximation is equivalent to assuming that 
$$ 
    \bwmarg{\ell+1}{x_{\ell+1}} \abwker{\ell|\ell+1}{x_\ell}{x_{\ell+1}} = \bwmarg{\ell}{x_\ell} \fwtrans{\ell+1|\ell}{x_\ell}{x_{\ell+1}} \eqsp,
$$ 
which holds true if one assumes that the backward transition is learned perfectly, i.e. $\abwker{\ell|\ell+1}{x_{\ell+1}}{x_\ell} = \fwtrans{\ell|\ell+1}{x_{\ell+1}}{x_\ell}$. We now proceed to show that \textsc{RePaint} is a particular case of the Gibbs sampler we have just described. For the sake of simplicity, we assume that $A$ is rectangular unit diagonal, \emph{i.e.}, that we only observe the first $\dimobs$ coordinates of a sample from the prior. We denote by $\overline{X}$ the first $\dimobs$ coordinates of $X \in \rset^\dimx$ and by $\underline{X}$ the remaining ones. Then sampling $\smash{X^\ell _{k+1} \sim \pibwker{\ell|\ell+1}{X^{\ell+1} _k}{\cdot}}$ \eqref{eq:cbwker_taul_giv_taulnext} is equivalent to setting $X^{\ell} _{k+1} = [ {\overline{X}} ^{\ell} _{k+1},  \underline{X}^{\ell} _{k+1}]$, where
\begin{align*}
    \overline{X} ^{\ell} _{k+1} & = \frac{\sigma^2 _{\ell|\ell+1}}{\sigma^2 _{\obs, \ell} + \sigma^2 _{\ell|\ell+1}} \obs_\ell + \frac{\sigma^2 _{\obs, \ell}}{\sigma^2 _{\obs, \ell} + \sigma^2 _{\ell|\ell+1}} \overline{\muDDIM{\ell|\ell+1}(X^{\ell+1} _{k})} +  \frac{\sigma^2 _{\obs, \ell} \sigma^2 _{\ell|\ell+1}}{\sigma^2 _{\obs, \ell} + \sigma^2 _{\ell|\ell+1}} \overline{Z}_{k+1} \eqsp,\\
     \underline{X}^{\ell} _{k+1} & =  \underline{\muDDIM{\ell|\ell+1}(X^{\ell+1} _{k})} +  \sigma^2 _{\ell|\ell+1} \underline{Z}_{k+1}.
\end{align*}
In the specific case of an inverse problem with $\stdobs = 0$, we can, setting for all $\ell$ $\sigma _{\obs, \ell} = 0$ and $\obs_\ell = \sqrt{\acp{\ell}} \, \obs + \sqrt{1 - \acp{\ell}} Z_\ell$ and $Z_\ell$ are i.i.d. standard Gaussian samples allows us to recover Algorithm 1 in \citet{lugmayr2022repaint}.
\section{Experiments}
\subsection{Implementation details}
In this section we provide the global implementation details for each algorithm. We provide the specific parameters (when needed) used for each experiment (Gaussian mixture, image restoration and trajectory inpainting) in the dedicated sections below. 
\label{apdx:implementation_details}
\paragraph{DCPS.} For all the experiments we implement \Cref{algo:algo}. We use the same parameters $K = 2$, $L = 3$ and $\zeta = 1$ for all the experiments. For the number of Langevin steps, we set it to $M = 50$ and $M = 500$ (respectively) for the Gaussian mixture experiment and $M = 5$ for the imaging and trajectory inpainting experiments. 

\paragraph{DDRM.} We have used the official implementation\footnote{\url{https://github.com/bahjat-kawar/ddrm}} and used the recommended parameters in the original paper.  We use 200 steps for DDRM and found that it works better than when we used 1000 steps.

\paragraph{DPS.} We have implemented both Algorithm 1 (for linear inverse problems) and Algorithm 2 (for Poisson-shot restoration) given in \cite{chung2023diffusion}. In all the experiments we run \DPS\ with $1000$ Diffusion steps. 

\paragraph{RedDiff.} For RedDiff, we have used the publicly available implementation\footnote{\url{https://github.com/NVlabs/RED-diff}}.
We have empirically found that RedDiff works best in the low observation standard deviation regime and produces spatially coherent reconstructions in the larger noise regime but struggles with getting rid of the noise as evidenced by the large increase in LPIPS values in \Cref{table:lin_invp}. Note also that it is not clear how the parameters of the algorithm depend on the inverse problem standard deviation; indeed, looking at Algorithm 1 and then Appendix C.2 where the authors consider a noisy inverse problem\footnote{\url{https://openreview.net/pdf?id=1YO4EE3SPB}} there seems to be no clear dependence of $\lambda$ on $\sigma_v$ ($\sigma_\obs$ with our notations). In fact the authors use $\lambda = 0.25$ similarly to the noiseless experiments in the main paper and we believe that the tuning is performed only on the initial step-size of Adam. As a result, for the experiments with $\stdobs = 0.3$, we have tuned it using a grid-search in $[0.1, 0.25]$ and retained $0.1$.

\paragraph{$\Pi$GDM.} Regarding \PGDM\ \cite{song2022pseudoinverse}, note that there is no publicly available implementation and we have thus implemented the noisy version of \citep[Algorithm 1]{song2022pseudoinverse} in the original paper. However, we did not manage to obtain appropriate results and found it to be quite unstable. We have further investigated the issue and found that \PGDM\ is implemented in the github repository of RedDiff\footnote{\url{https://github.com/NVlabs/RED-diff}}, which is by the same authors. We have noted that it has a slight difference with Algorithm 1 of the \PGDM\ paper; the gradient term, coined $g$ in \citep[Algorithm 1]{song2022pseudoinverse}, is multiplied by $\sqrt{\alpha_{t-1} \alpha_{t}}$ instead of simply $\sqrt{\alpha_t}$. We have found that this stabilizes the algorithm significantly for the linear inverse problem experiment. We use the same rescaling for the Gaussian mixture and trajectory inpainting experiment. However, even with this modification to the algorith we found that \PGDM\ does not perform well when the noise standard deviation is large; see \Cref{table:lin_invp}. For the JPEG experiment we do not use this rescaling as we found that the algorithm remains stable.

\paragraph{MCGDiff.} For MCGDiff we have used the official implementation\footnote{\url{https://github.com/gabrielvc/mcg_diff}} with $N=32$ particles for the imaging experiments. There are no further tuning parameters as far as we can tell.

\paragraph*{DIFFPIR}
We implemented \cite[Algorithm 1]{zhu2023diffpir} and use the hyperparameters recommended in the official, released version\footnote{\url{https://github.com/yuanzhi-zhu/DiffPIR}}.

\paragraph*{DDNM.}
We adapted the implementation in the released code\footnote{\url{https://github.com/wyhuai/DDNM}} to our code base.

\paragraph*{SDA.}
We implement the posterior sampling algorithm by combining \cite[Algo 3 and 4 in Appendix C]{rozet2023sda}.
In the experiments, we use two Langevin corrections steps and found that $\gamma=0.1$ works well across problems for the diagonal approximation the same as $\tau=0.1$ for the Langevin correction steps size.

\paragraph*{FPS}
We implement \cite[Algorithm 2]{dou2024diffusion} provided in the appendix.

\subsection{Gaussian mixtures}
For a given dimension $\dimx$, we consider $\datadistr$ a mixture of $25$ Gaussian random variables.
The means of the Gaussian components of the mixture are $(\mathbf{m}_{i})_{i = 1} ^{25} \eqdef \{ (8i, 8j, \cdots, 8i, 8j) \in \rset^\dimx:\, (i,j) \in \{-2, -1, 0, 1, 2\}^2 \}$. The covariance of each component is identity.
The mixture (unnormalized) weights $w_{i, j}$ are independently drawn from a Dirichlet distribution.
\paragraph{Metrics.} To assess the performance of each algorithm we draw 2000 samples and compare against 2000 samples from the true posterior distribution using the Sliced Wasserstein distance by averaging over $10^4$ slices. In \Cref{table:gm} we report the average SW and the $95\%$ confidence interval over 30 seeds. We found \DPS\ and \PGDM\ to be sometimes unstable, resulting in \texttt{NaN} values. To account for these unstabilities when computing the average SW distance, we replace \texttt{NaN} with 7 which is the typical value obtained when a stable algorithm fails to sample from the posterior. 
\paragraph{Parameters.} For \DPS\ we use $\zeta_m = 0.1 / \|\obs - \bfA \predxs{m}(x_m)\|$ at step $m$ of the Diffusion. As to \algo\ we use $\gamma = 10^{-2}$ for the Langevin step-size. 
\paragraph{Denoisers.} Note that the loss \eqref{eq:denoising_objective} can be written as
\begin{align*}
        \lefteqn{\sum_{t = 1}^T  w_t \pE \left[ \| \epsilon_t - \prednoise{t}(\sqrt{\acp{t}} X_0 + \sqrt{1 - \acp{t}} \epsilon_t ) \|^2 \right]} \hspace{30mm} \\ & = \sum_{t = 1}^T  \frac{w_t}{1 - \acp{t}} \pE \left[ \| \sqrt{1 - \acp{t}} \epsilon_t - \sqrt{1 - \acp{t}} \prednoise{t}(\sqrt{\acp{t}} X_0 + \sqrt{1 - \acp{t}} \epsilon_t ) \|^2 \right] \\
        & = \sum_{t = 1}^T  \frac{w_t}{1 - \acp{t}} \pE \left[ \| X_t - \sqrt{\acp{t}} X_0 - \sqrt{1 - \acp{t}} \prednoise{t}(X_t) \|^2 \right] \\
        & = \sum_{t = 1}^T  \frac{w_t \acp{t}}{1 - \acp{t}} \pE \left[ \left\| X_0 - \frac{X_t - \sqrt{1 - \acp{t}} \prednoise{t}(X_t)}{\sqrt{\acp{t}}} \right\|^2 \right] \eqsp.
\end{align*}
Hence the minimizer is 
$$ 
    \epsilon^{\param^\star} _t(x_t) = \frac{x_t - \sqrt{\acp{t}}\, \pE [X_0 | X_t = x_t]}{\sqrt{1 - \acp{t}}}\eqsp,
$$ 
which yields $\predxs{t} = \pE [X_0 | X_t = \cdot]$. Next, by Tweedie's formula we have that
$$
\predxs{t}(x_t) = \frac{x_t + (1 - \acp{t}) \nabla_x \log \fwmarg{t}{x_t}}{\sqrt{\acp{t}}} \eqsp.
$$
Hence, since $\fwmarg{\mathrm{data}}{}$ is a mixture of Gaussians, $\fwmarg{t}{}$ is also a mixture of Gaussians with means $(\sqrt{\acp{t}} \mathbf{m}_{i})_{i = 1} ^{25}$ and unit covariances.
Therefore, $\nabla_x \log \fwmarg{t}{x_t}$ and hence $\predxs{t}(x_t)$ can be computed using automatic differentiation libraries.
\paragraph{Measurement model.}
For a pair of dimensions $(\dimx, \dimobs)$ the measurement model $(\obs, \bfA, \stdobs)$ is drawn as follows: the elements $\dimx \times \dimobs$ elements of the matrix are drawn i.i.d. from a standard Gaussian distribution, then $\stdobs$ is drawn uniformly in $[0,1]$ and finally we draw $x^\star \sim \datadistr$ and $\varepsilon \sim \gauss(\zero_\dimobs, \Id_\dimobs)$ and set $\obs = \bfA x^\star + \stdobs \varepsilon$.
\paragraph{Posterior.}
Having drawn both $\datadistr$ and $(\obs, \bfA, \stdobs)$, the posterior can be computed exactly using standard Gaussian conjugation formulas \citep[Eq. 2.116]{bishop2006pattern} and hence the posterior is a Gaussian mixture where all the components have the same covariance matrix $\Sigma \eqdef \left(\Id_{\dimx} + \stdobs^{-2} \m{A}^T \m{A}\right)^{-1}$ and means and weights given by
\begin{align*}
    \tilde{\mathbf{m}}_i  & \eqdef \Sigma \left(A^\intercal \obs / \stdobs^2 + \mathbf{m}_{i} \right) \eqsp,\\
    \tilde{w}_i &\propto w_i \normpdf(\obs;\bfA \mathbf{m}_{i}, \stdobs^2 \Id_{\dimx} + \bfA \bfA^\intercal)\eqsp.
\end{align*}
\label{subsec:gm}

\subsection{Imaging experiments}
\begin{table}
    \centering
    \captionsetup{font=small}
    \caption{Mean LPIPS value on low count Poisson restoration.}
    \resizebox{.3\textwidth}{!}{
    \begin{tabular}{@{}cccc@{}}
    \toprule
    Dataset  & Task & \algo & \DPS  \\
    \midrule
    \multirow{2}{*}{\centering \ffhq}
    & Denoising & \bf{0.07} & 0.12 \\
    & SR $4 \times$ & \bf{0.17} & 0.31 \\
    \midrule
    \multirow{2}{*}{\centering \imagenet}
    & Denoising & \bf{0.17} & 0.24 \\
    & SR $4 \times$ & \bf{0.36} & 0.80  \\
    \midrule
    \end{tabular}
    }
    \vspace{-.4cm}
    \label{table:poisson-denoising}
\end{table}
\paragraph{Parameters.} For \algo\ we set $\gamma = 10^{-3}$ for the Langevin step-size. For \DPS\ we use the parameters recommended in the original paper, which we found to work well even on the half and expand masks; see \citep[Appendix D.1]{chung2023diffusion}. 

\paragraph{Evaluation.} In order to evaluate each algorithm we compute the LPIPS metric \cite{zhang2018unreasonable} on each dataset using $100$ samples from the validation sets and report the average in \Cref{table:lin_invp}, \ref{table:jpeg-dequantization} and \ref{table:poisson-denoising}. 

\paragraph{JPEG dequantization.} We use the differentiable JPEG framework \cite{shin2017jpeg} which replaces the rounding function $x \mapsto \lfloor x \rceil$ used in the quantization part with $x \mapsto \lfloor x \rceil + (x - \lfloor x \rceil)^3$ which has non-zero derivatives almost everywhere.

\subsection{Trajectory inpainting experiment}
\label{apdx:trajectory-exp}
\paragraph{Trajectory DDM prior.}
The denoiser of the diffusion model has a Transformer-like architecture.
In the entry of the network, the trajectory is augmented to a higher dimensional space ($512$) via dense layer.
At this stage a positional encoding \citep{vaswani2017attention} is added to account for the diffusion step.
Afterward, the output is flowed through a transformer encoder \citep{vaswani2017attention} whose feedforward layer dimension is $2048$ to learn temporal dependence within the trajectory before being feed to an MLP with $4$ layers ($512 \rightarrow 1024 \rightarrow 1024 \rightarrow 512$) and in between ReLU activation functions, to output the added noise.
A Cosine noise scheduler with $1000$ diffusion steps was used \citep{nichol2021improved-diffusion-models}.
The \texttt{UCY}-student dataset was split int a train and a validation sets with $1450$ and $140$ trajectories respectively.
The batch size was set to 10 times the training set, namely $145$ samples
The denoiser was trained to minimize the loss of DDPM \citep{ho2020denoising} for $1000$ epochs using Adam solver \citep{kingma2014adam} with a Cosine learning rate scheduler \citep{loshchilov2016sgdr}.
The training was performed on 48GB L40S NVIDIA GPU and took roughly one minute to complete.

\paragraph{Metrics.} The trajectory completion experiment was performed on the validation set.
Every trajectory was masked randomly.
Leveraging \mcgdiff\ 's asymptotical approximation of the posterior, it was run with $5000$ particles to sample $100$ samples from the posterior and afterward these were checked against a $100$ reconstructions of each other algorithm by computing the \emph{timestep wise} $\ell_2$ distance between the quantile $50$ (median), $25$, $75$ and also by computing the Sliced Wasserstein distance.
This procedure was repeated for all trajectories in the validation set and later the results of each algorithm were aggregated by the mean $\ell_2$ distances.
Finally, this experiment was performed for two levels of noise $\sigma_y = 0.005$ and $\sigma_y = 0.01$.

\subsection{Additional experiments}
\label{apdx:add-experiments}

Here, we provide the complete tables of results on imaging and trajectories inpainting experiments that includes in addition \DiffPIR, \DDNM, \FPS, and \SDA.
These additional experiments were conducted during the rebuttal phase of our work.

\begin{table}[t]
    \centering
    \caption{Mean LPIPS value on different tasks. Lower is better.}
    \resizebox{0.9\textwidth}{!}{
        \begin{tabular}{@{}cccccccccccc@{}}
        \toprule
        Dataset / $\stdobs$ & Task & \algo & \ddrm & \DPS & \PGDM & \reddiff & \mcgdiff & \sc{DiffPIR} & \sc{DDNM} & \sc{SDA} & \sc{FPS} \ \\
        \midrule
        \multirow{4}{*}{\centering \ffhq\ / 0.05}
        & Half & \bf{0.20} & 0.25 & 0.24 & 0.26 & 0.28 & 0.36 & 0.23 & \underline{0.22} & 0.23 & 0.28 \\
        & Center & \bf{0.05} & \underline{0.06} & 0.07 & 0.19 & 0.12 & 0.24 & \underline{0.06}  & \bf{0.05} & \bf{0.05} & 0.09 \\
        & SR $4 \times$ & \bf{0.09} & 0.18 & \bf{0.09} & 0.33 & 0.36 & 0.15 & 0.13 & 0.14 & \underline{0.10} & \underline{0.10} \\
        & SR $16 \times$ & \bf{0.23} & 0.36 & \underline{0.24} & 0.44 & 0.51 & 0.32 & 0.28 & 0.30 & 0.44 & 0.71 \\
        \midrule
        \multirow{4}{*}{\centering \ffhq\ / 0.3}
        & Half & \bf{0.25} & 0.30 & 0.31 & 0.64 & 0.76 & 0.80 &  0.30 & \underline{0.26} & \underline{0.26} & 0.67 \\
        & Center & \bf{0.10} & 0.13 & \underline{0.11} & 0.62 & 0.75 &  0.55 & 0.16 & \underline{0.11} & \bf{0.10} & 0.69 \\
        & SR $4 \times$ & \underline{0.21} & 0.26 & \bf{0.19} & 0.77 & 0.77 & 0.65 & 0.28 & 0.23 & \bf{0.19} & 0.75 \\
        & SR $16 \times$ & \bf{0.35} & 0.41 & 0.43 & 0.64 & 0.74 & 0.52 & 0.42 & \underline{0.39} & 0.49 & 0.71 \\
        \midrule
        \multirow{4}{*}{\centering \imagenet\ / 0.05}
        & Half & \bf{0.35} & 0.40 & 0.44 & \underline{0.38} & 0.44 & 0.83 & \bf{0.35} & \underline{0.38} & 0.54 & 0.39 \\
        & Center & 0.18 & \underline{0.14} & 0.31 & 0.29 & 0.22 & 0.45 & \underline{0.14} & \bf{0.13} & \underline{0.14} & 0.19 \\
        & SR $4 \times$ & \bf{0.24} & 0.38 & 0.41 & 0.78 & 0.56 & 1.32 & 0.36 & 0.34 & 0.85 & \underline{0.27} \\
        & SR $16 \times$ & \bf{0.44} & 0.72 & \underline{0.50} & 0.60 & 0.83 & 1.33 & 0.63 & 0.70 & 1.13 & 0.69 \\
        \midrule
        \multirow{4}{*}{\centering \imagenet\ / 0.3}
        & Half & \bf{0.40} & 0.46 & 0.48 & 0.82 & 0.76 & 0.86 & 0.50 & \underline{0.44} & 0.61 & 0.71 \\
        & Center & \underline{0.24} & 0.25 & 0.40 & 0.68 & 0.71 & 0.47 & 0.36 & \bf{0.22} & 0.25 & 0.70 \\
        & SR $4 \times$ & \bf{0.43} & 0.50 & 0.47 & 0.87 & 0.83 & 1.31 & 0.61 & \underline{0.46} & 1.14 & 0.84 \\
        & SR $16 \times$ & 0.72 & 0.77 & \bf{0.57} & 0.72 & 0.92 & \underline{0.67} & 0.76 & 0.75 & 1.19 & 0.74 \\
        \midrule
        Average & & \bf{0.28} & 0.35 & \underline{0.32} & 0.57 & 0.60 & 0.67 & 0.35 & \underline{0.32} & 0.48 & 0.53 \\
        Median & & \bf{0.24} & 0.33 & 0.35 & 0.63 & 0.72 & 0.60  & 0.32 & \underline{0.28} & 0.35 & 0.69 \\
        \midrule
        \end{tabular}
    }
    \label{table:lin_invp_continued}
\end{table}

\begin{table}
    \centering
    \captionsetup{font=small}
    \caption{$\ell_2$ distance quantiles with \mcgdiff\ as reference.}
    \resizebox{.45\textwidth}{!}{
            \begin{tabular}{lccccccc}
            \toprule
            & \multicolumn{3}{c}{$\sigma_y = 0.005$}  & & \multicolumn{3}{c}{$\sigma_y = 0.01$}  \\
            & $q50$ & $q25$ & $q75$ && $q50$ & $q25$ & $q75$ \\
            \midrule
            \algo    & \textbf{1.31} & \textbf{1.33} & \textbf{1.47} && \textbf{1.33} & \textbf{1.42} & \textbf{1.42} \\
            \DPS     & \underline{1.34} & 1.40 & 1.61 && \underline{1.36} & 1.48 & 1.52 \\
            \ddrm    & 1.48 & 1.46 & 1.61 && 1.59 & 1.62 & 1.61 \\
            \PGDM    & 1.36 & \underline{1.35} & \textbf{1.47} && 1.37 & \underline{1.43} & \textbf{1.42} \\
            \reddiff & 1.67 & 1.57 & 1.82 && 1.56 & 1.54 & 1.65 \\
            \DiffPIR & 1.57 & 1.84 & 1.98 && 1.52 & 1.94 & 1.89 \\
            \DDNM & 1.45 & 1.45 & 1.65  && 1.52 & 1.59 & 1.59  \\
            \FPS & 2.60 & 2.61 & 2.62 && 2.91 & 2.90 & 2.89 \\
            \SDA & 1.52 & 1.55 & 1.69 && 1.54 & 1.59 & 1.61  \\
            \bottomrule
            \end{tabular}
        }
        \label{table:trajectory-inpainting-L2_continued}
\end{table}
\section{Sample reconstructions}
\label{apdx:samples}
In this section we display the remaining samples from the experiments in the main paper. We remind the reader that all algorithms are run with the same seed and we draw in parallel $4$ samples from each algorithm and display them in their order of appearance.


\begin{figure}[htb]
    \centering
    \subfigure{
        \includegraphics[width=0.42\textwidth]{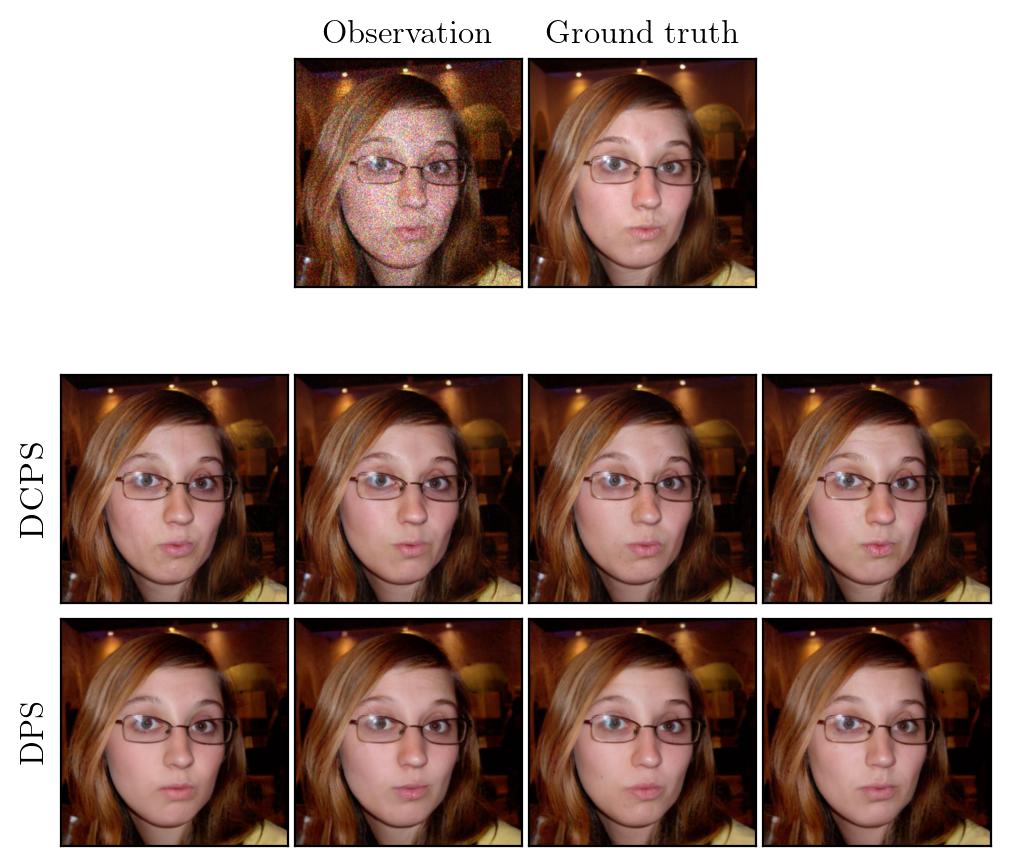}
        \includegraphics[width=0.42\textwidth]{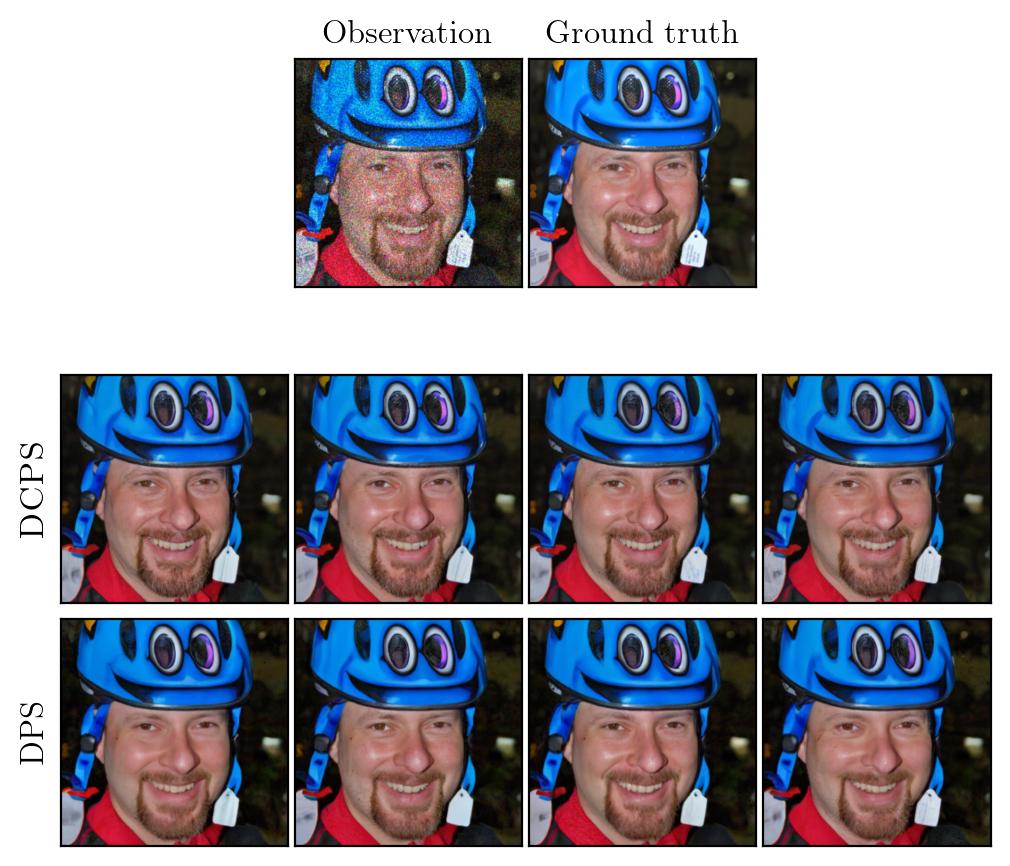}

    }
    \caption{Denoising task with Poisson noise on \ffhq.}
\end{figure}

\begin{figure}[htb]
    \centering
    \subfigure{
        \includegraphics[width=0.42\textwidth]{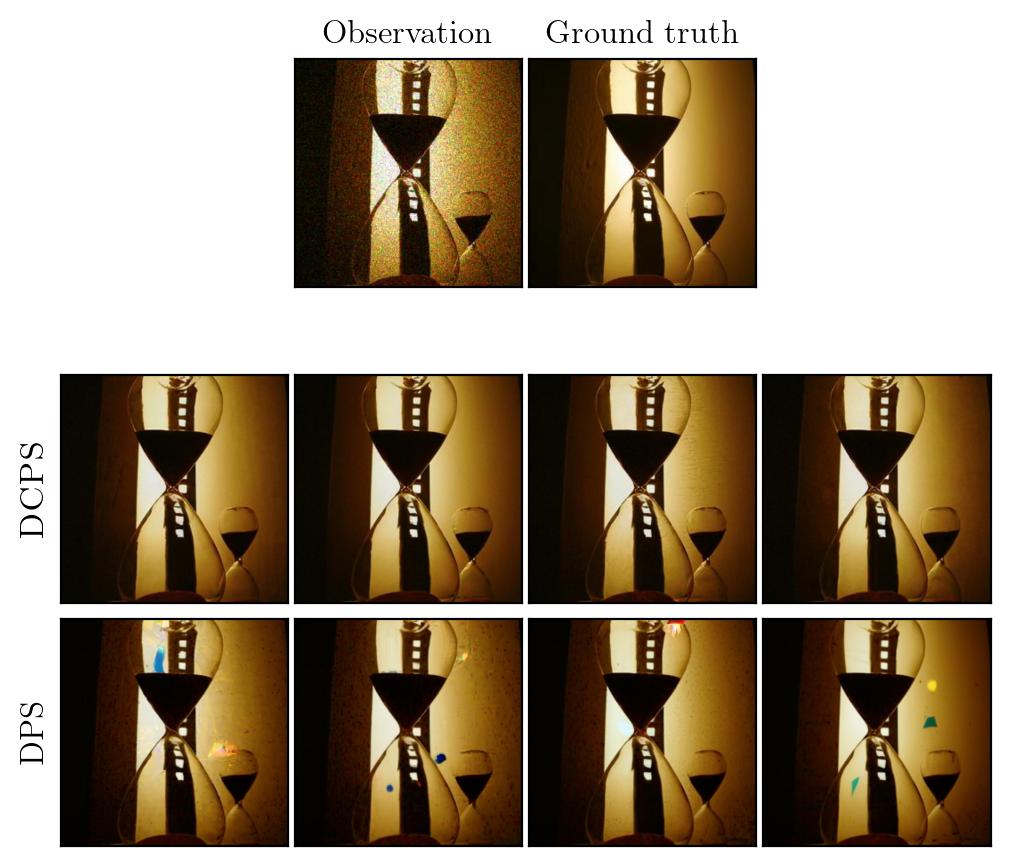}
        \includegraphics[width=0.42\textwidth]{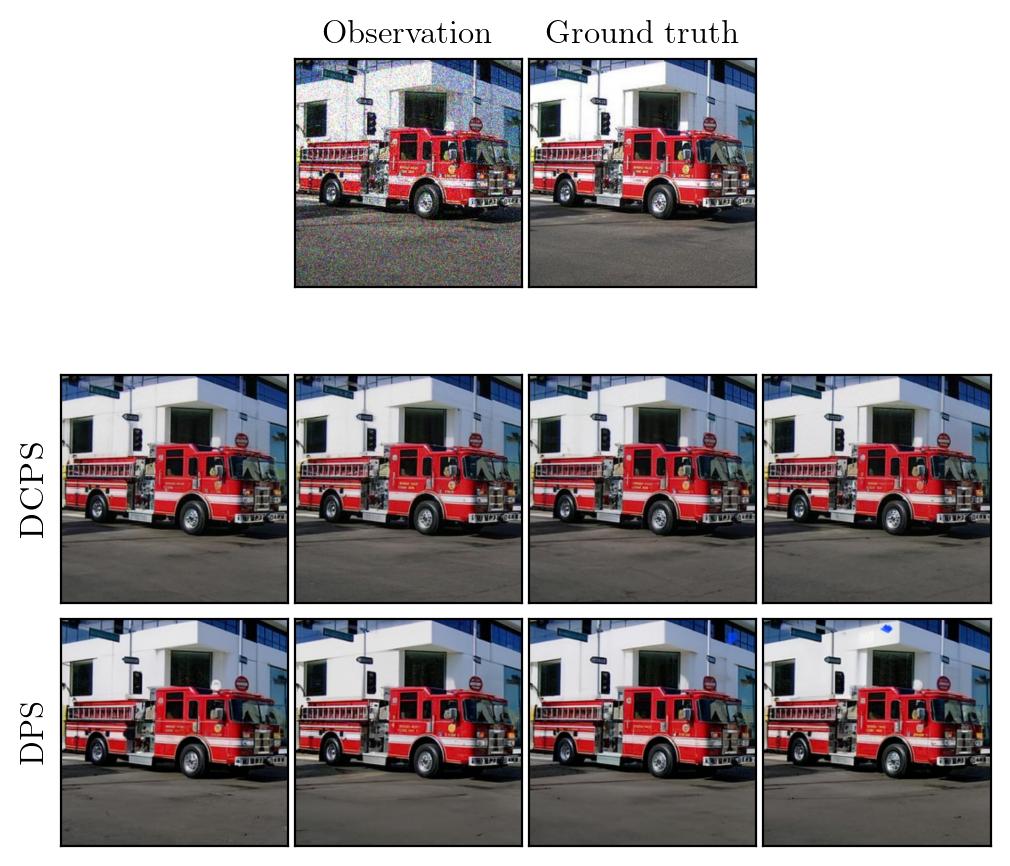}
    }
    \caption{Denoising task with Poisson noise on \imagenet.}
\end{figure}

\begin{figure}[htb]
    \centering
    \subfigure{
        \includegraphics[width=0.42\textwidth]{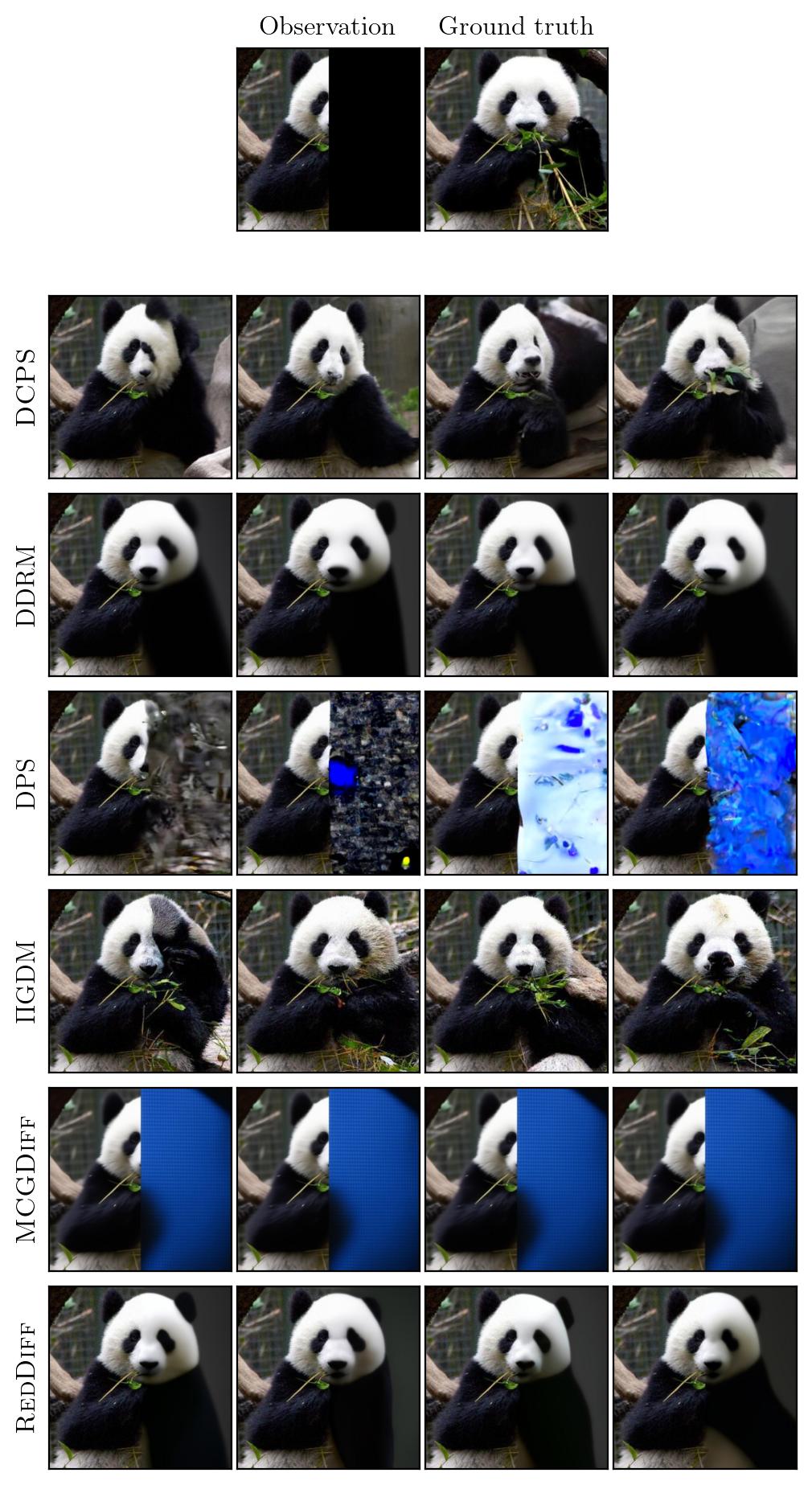}
    }
    \caption{Outpainting task with half mask on \imagenet.}
\end{figure}

\begin{figure}[htb]
    \centering
    \subfigure{
        \includegraphics[width=0.42\textwidth]{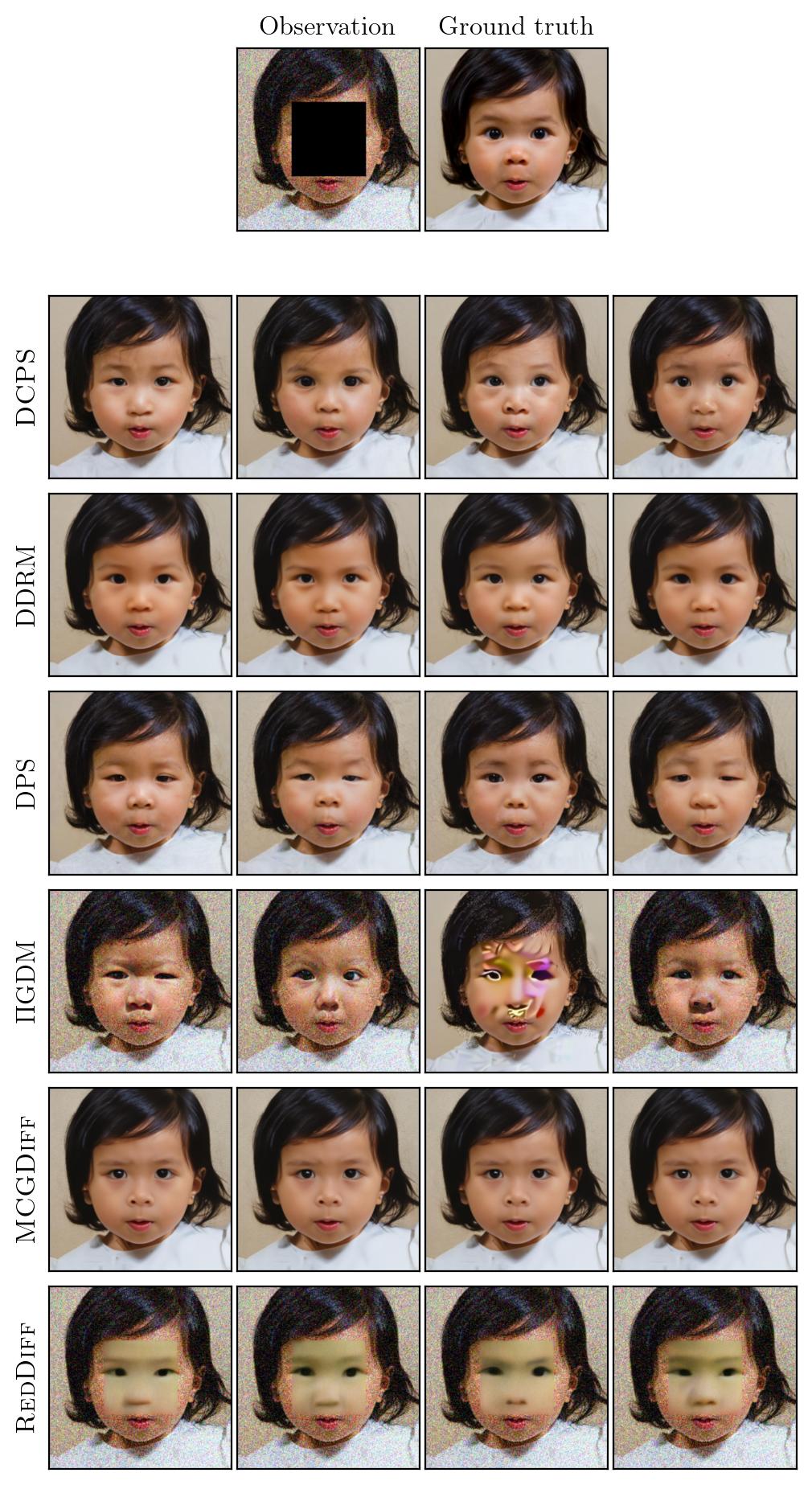}
        \includegraphics[width=0.42\textwidth]{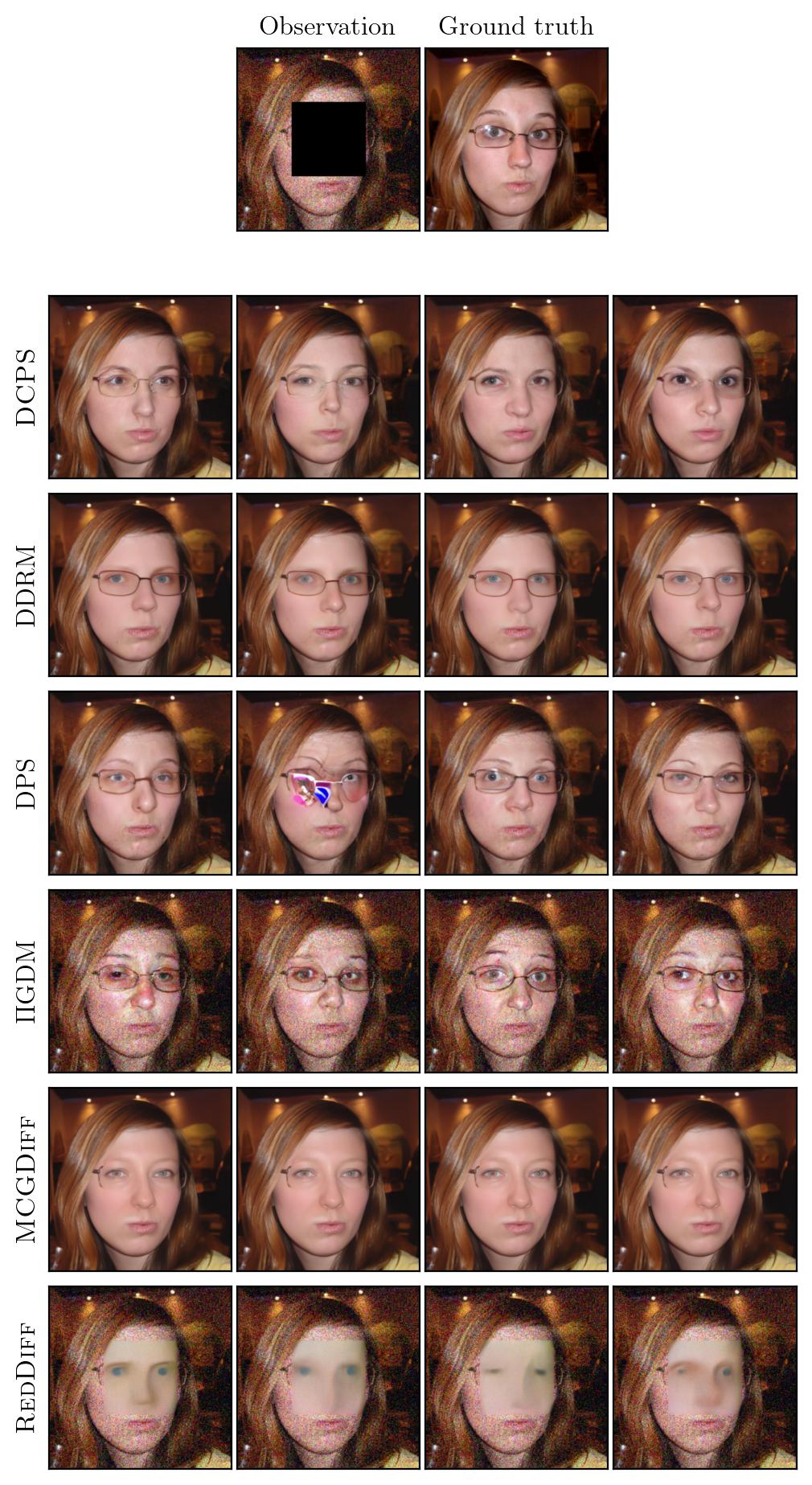}
    }
    \caption{Inpainting with box mask on \ffhq.}
    \subfigure{
        \includegraphics[width=0.42\textwidth]{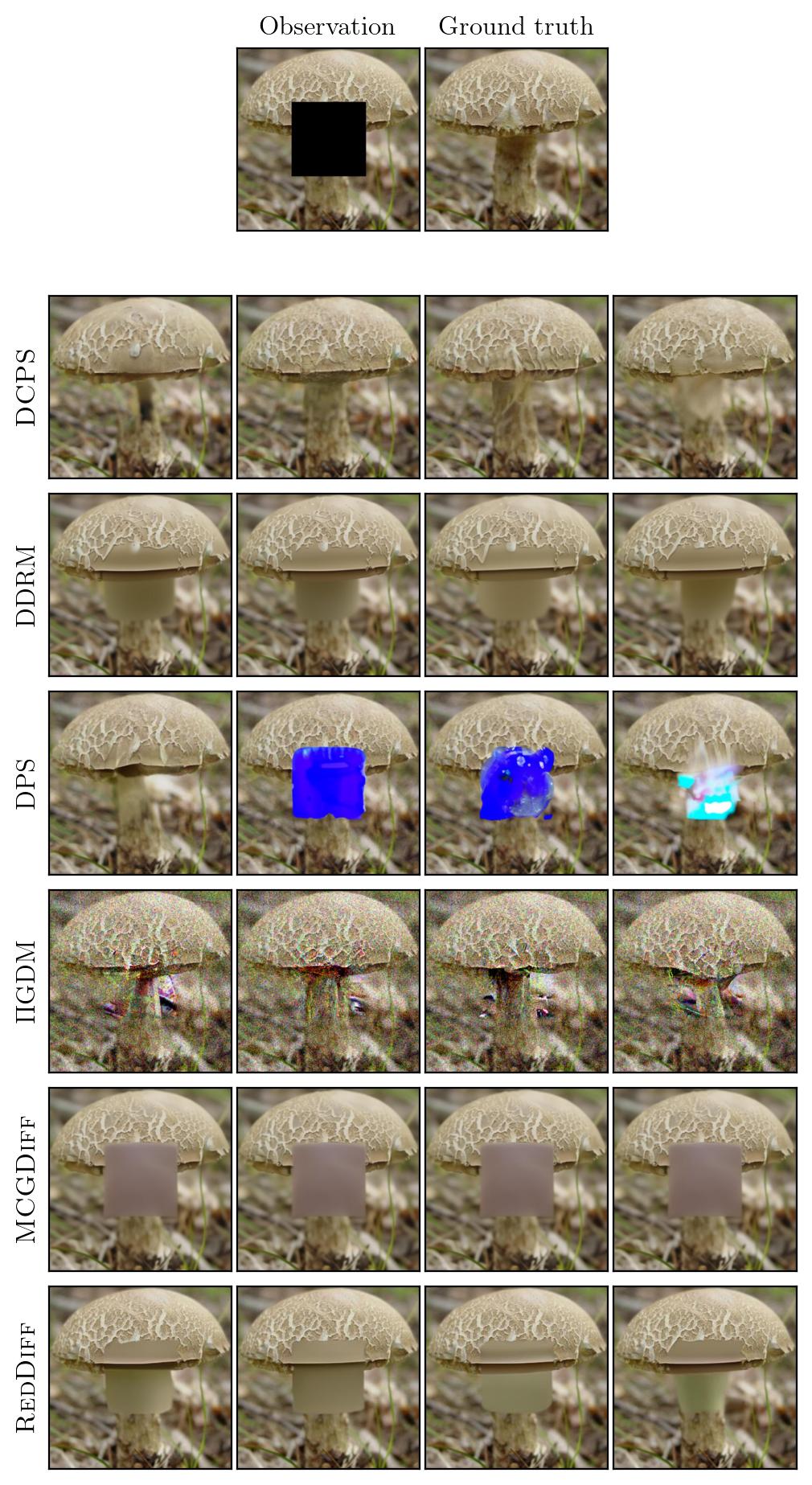}
        \includegraphics[width=0.42\textwidth]{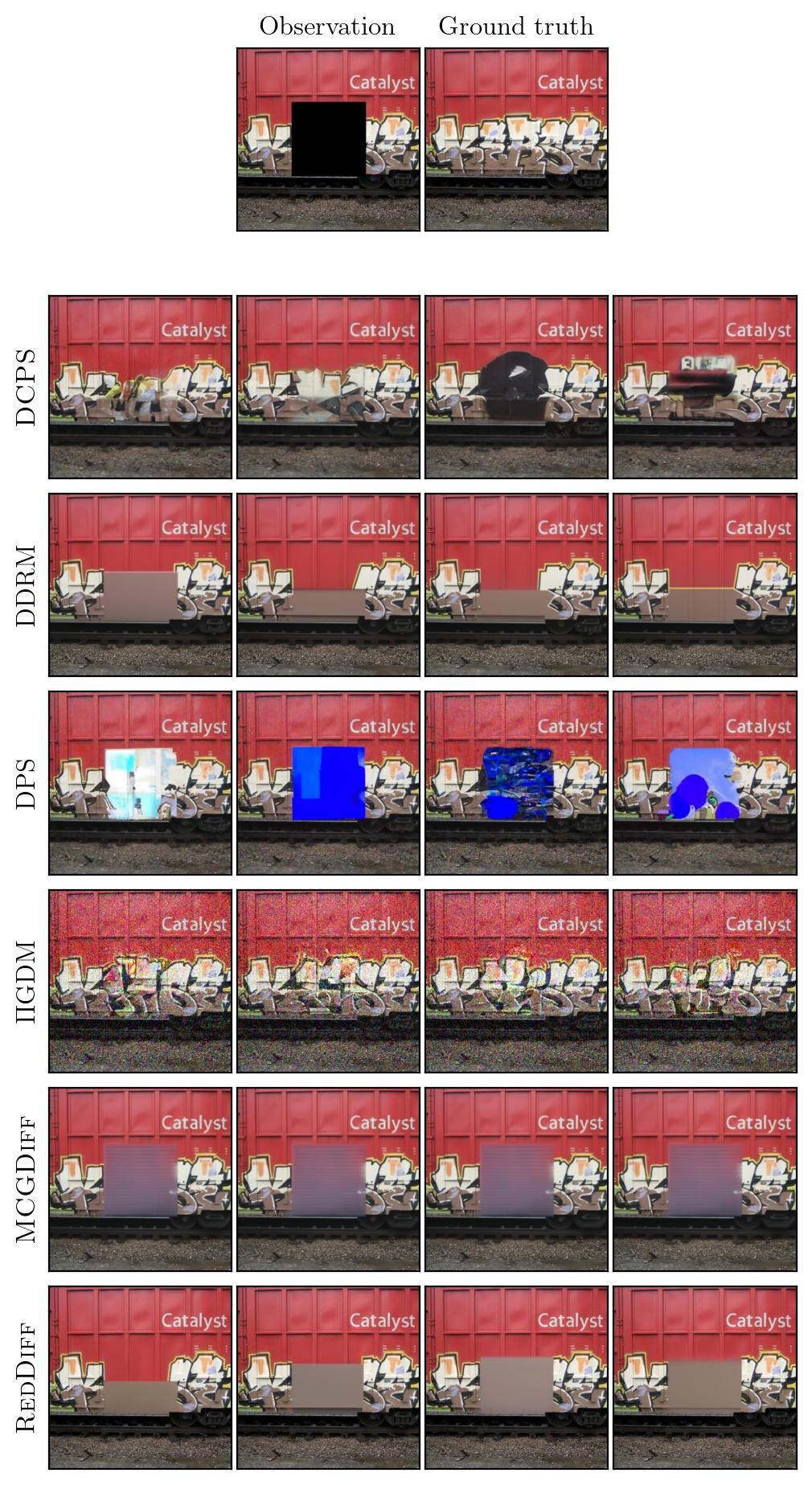}
    }
    \caption{Inpainting task with box mask on \imagenet.}
\end{figure}

\begin{figure}[htb]
    \centering
    \subfigure{
        \includegraphics[width=0.42\textwidth]{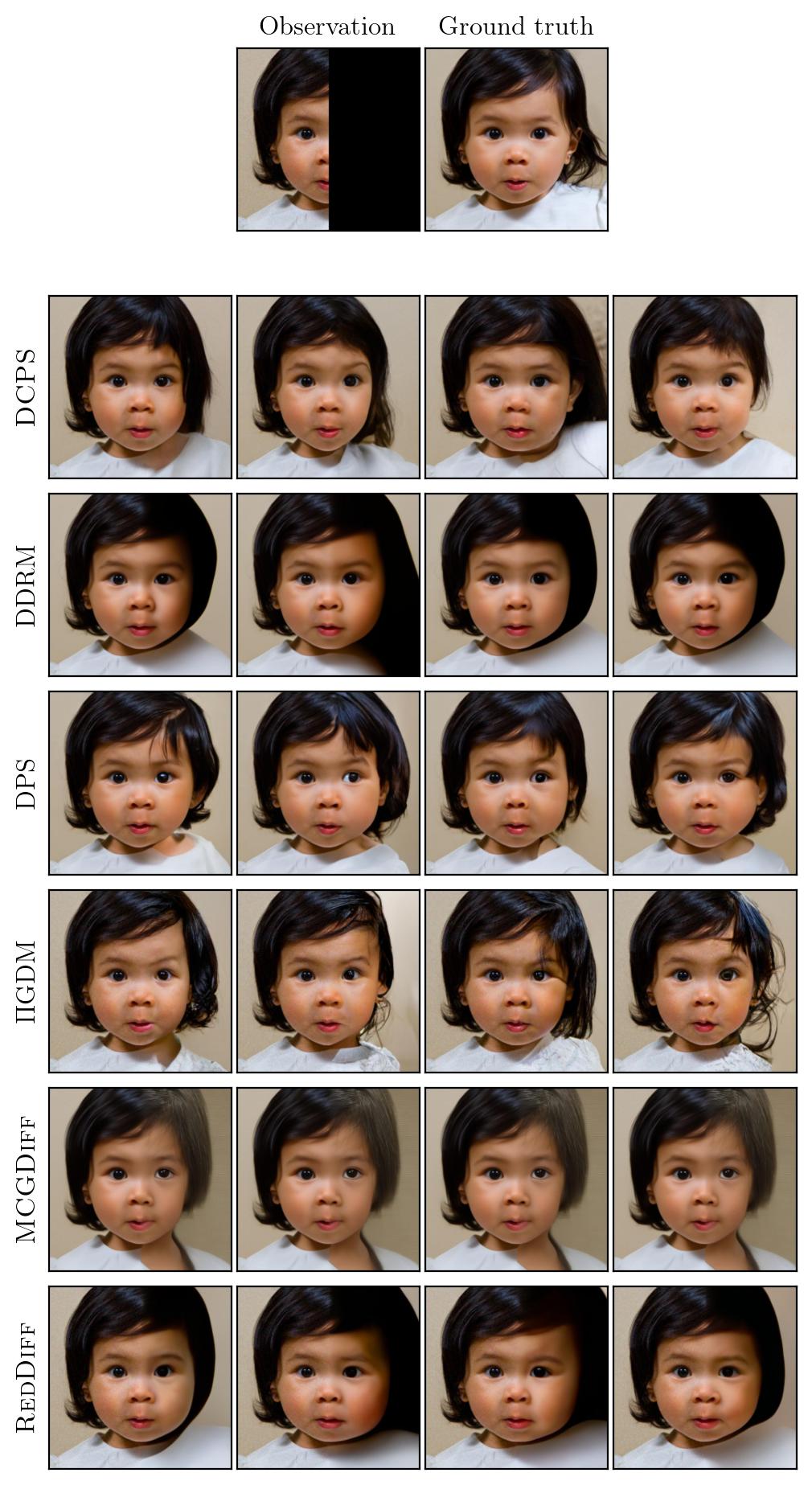}
        \includegraphics[width=0.42\textwidth]{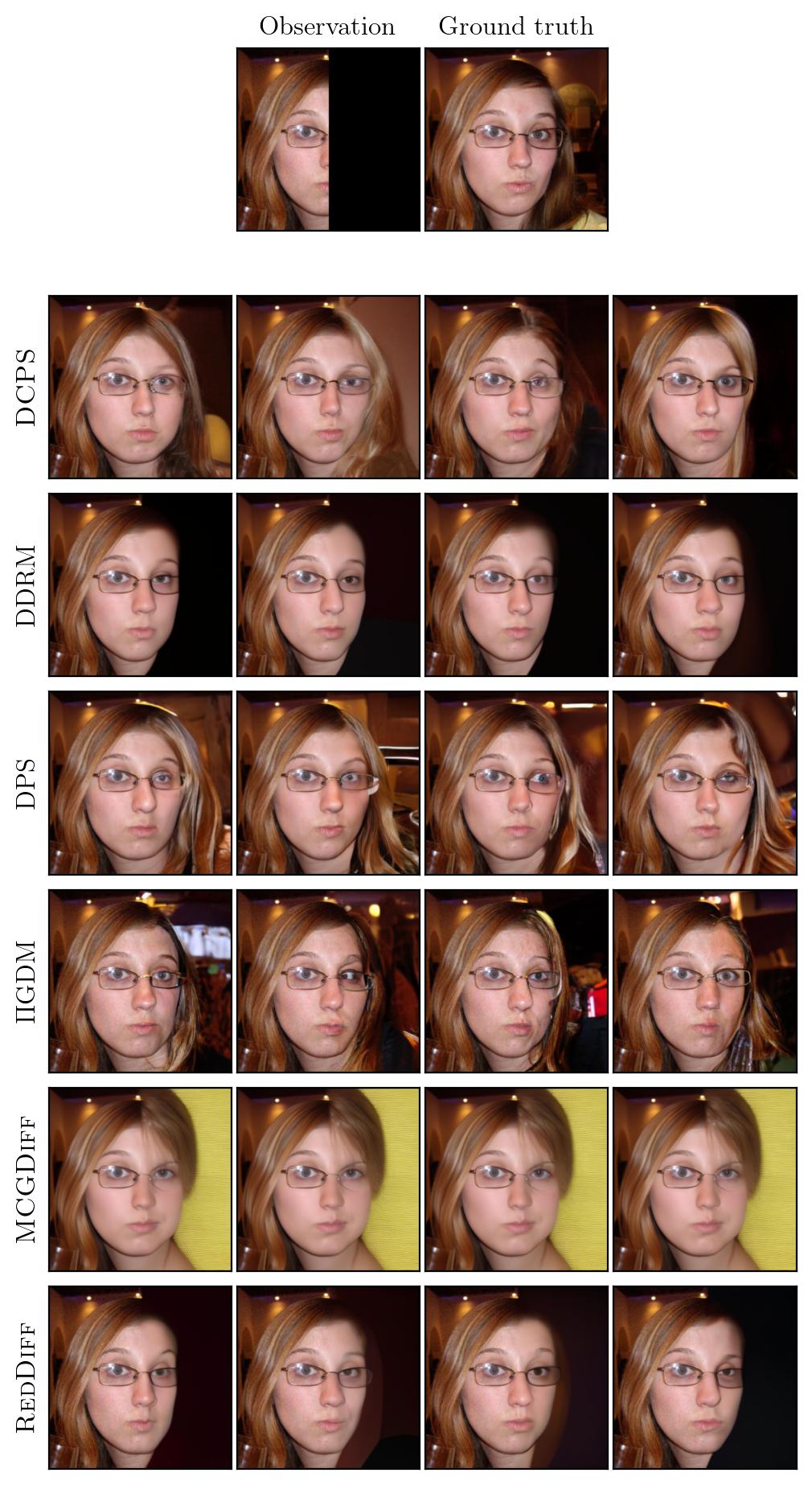}
    }
    \caption{Outpainting task with half mask on \ffhq.}
    \subfigure{
        \includegraphics[width=0.42\textwidth]{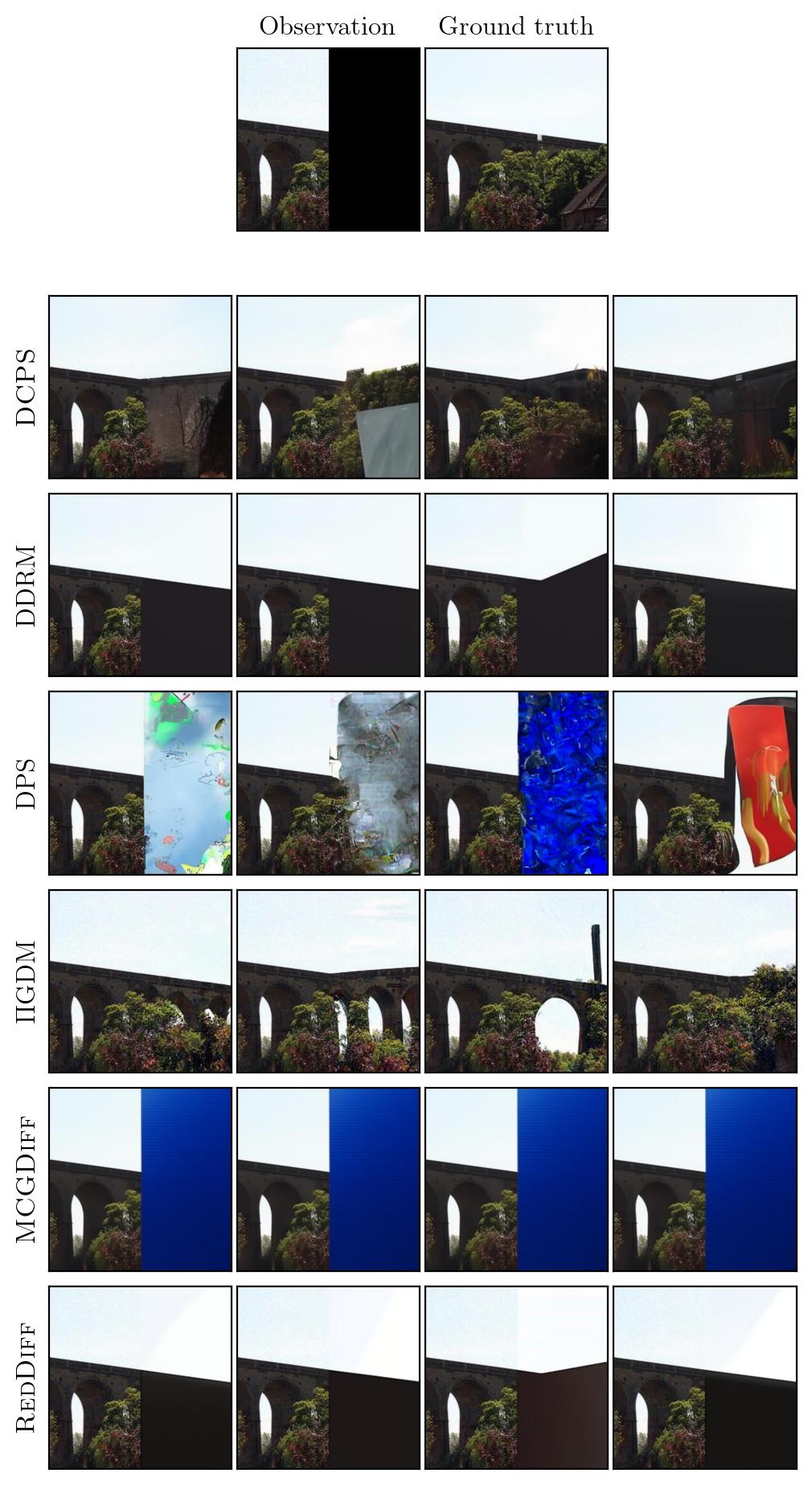}
        \includegraphics[width=0.42\textwidth]{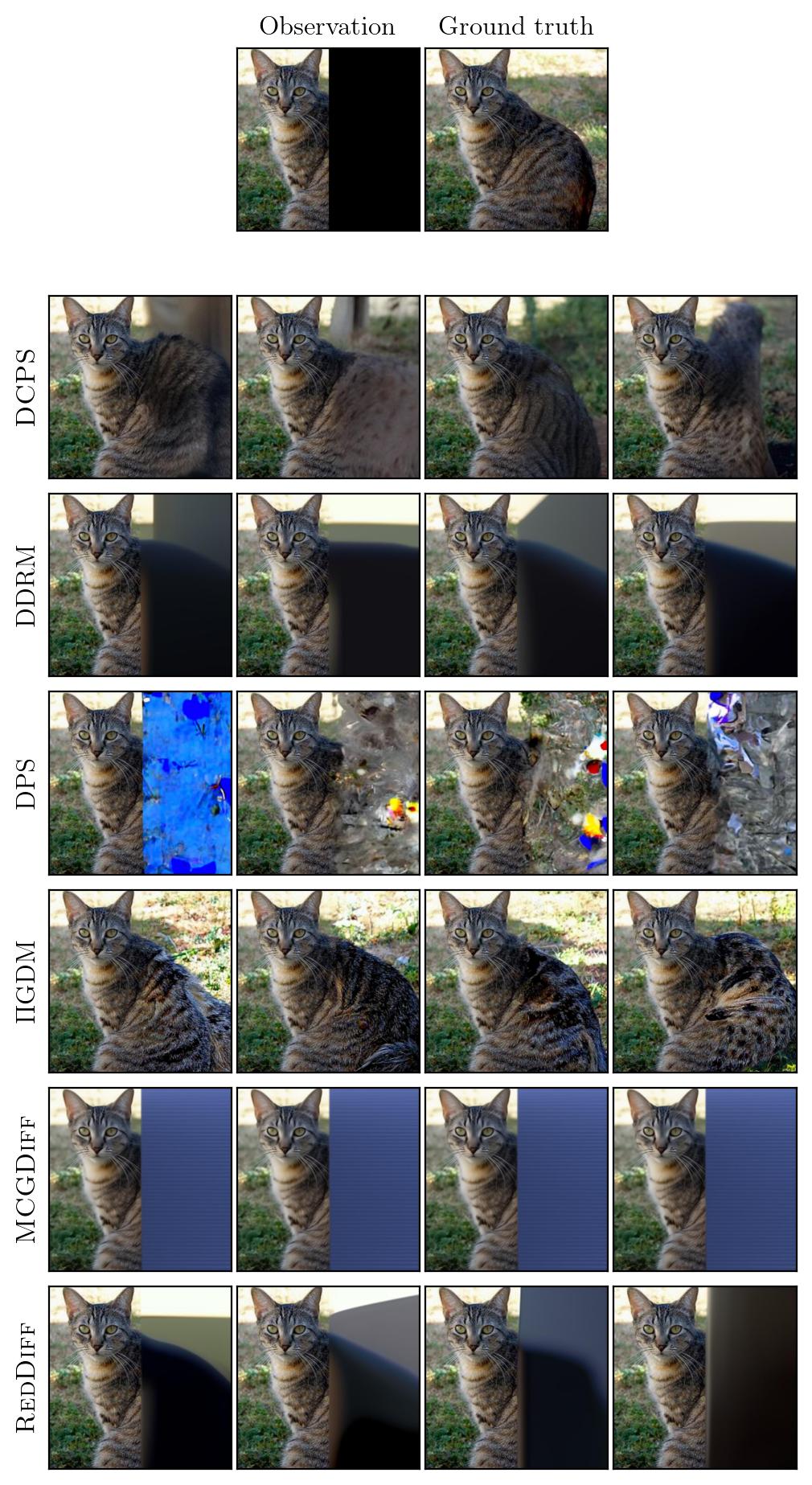}
    }
    \caption{Outpainting task with half mask on \imagenet.}
\end{figure}

\begin{figure}[htb]
    \centering
    \subfigure{
        \includegraphics[width=0.42\textwidth]{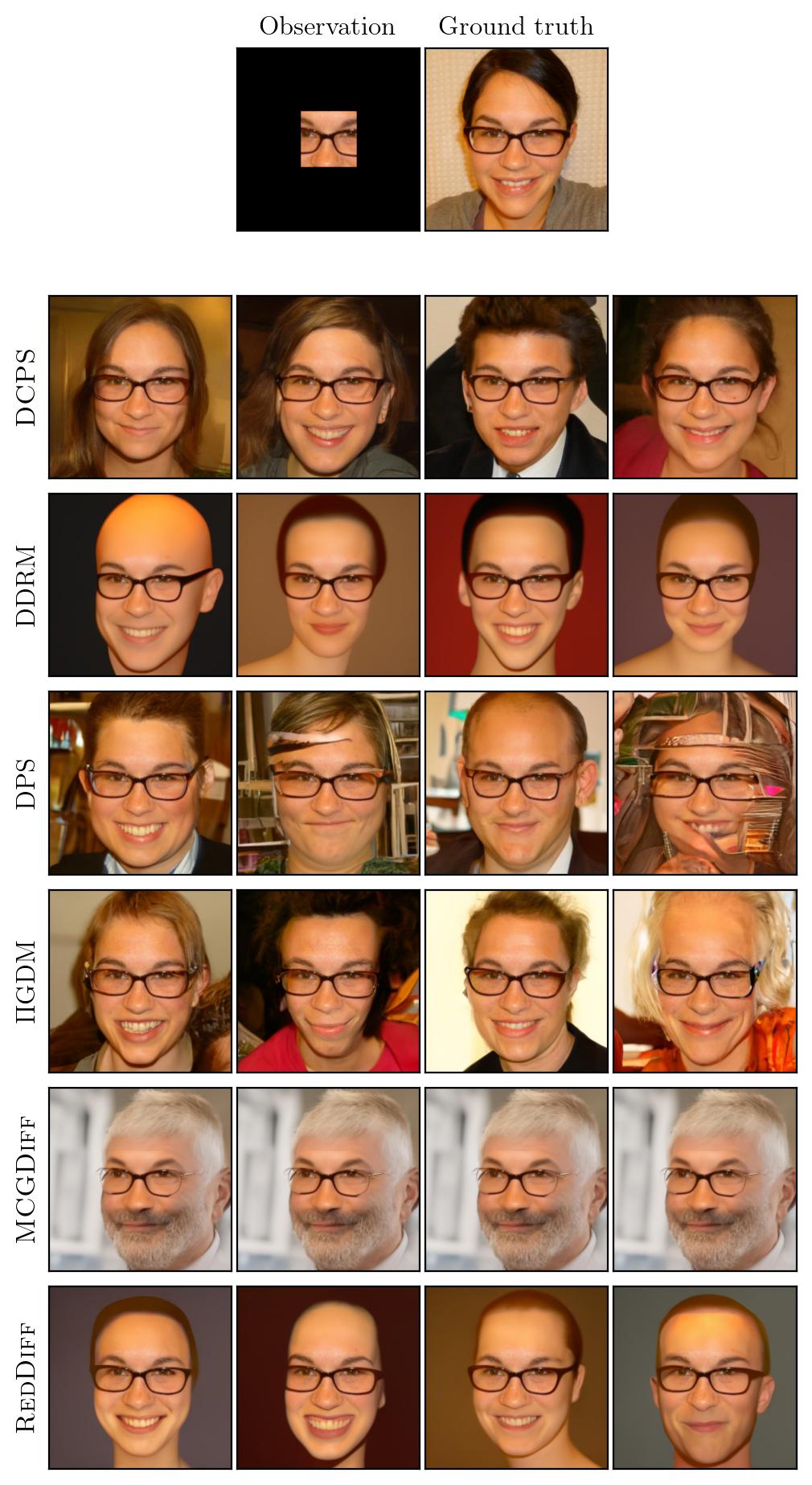}
        \includegraphics[width=0.42\textwidth]{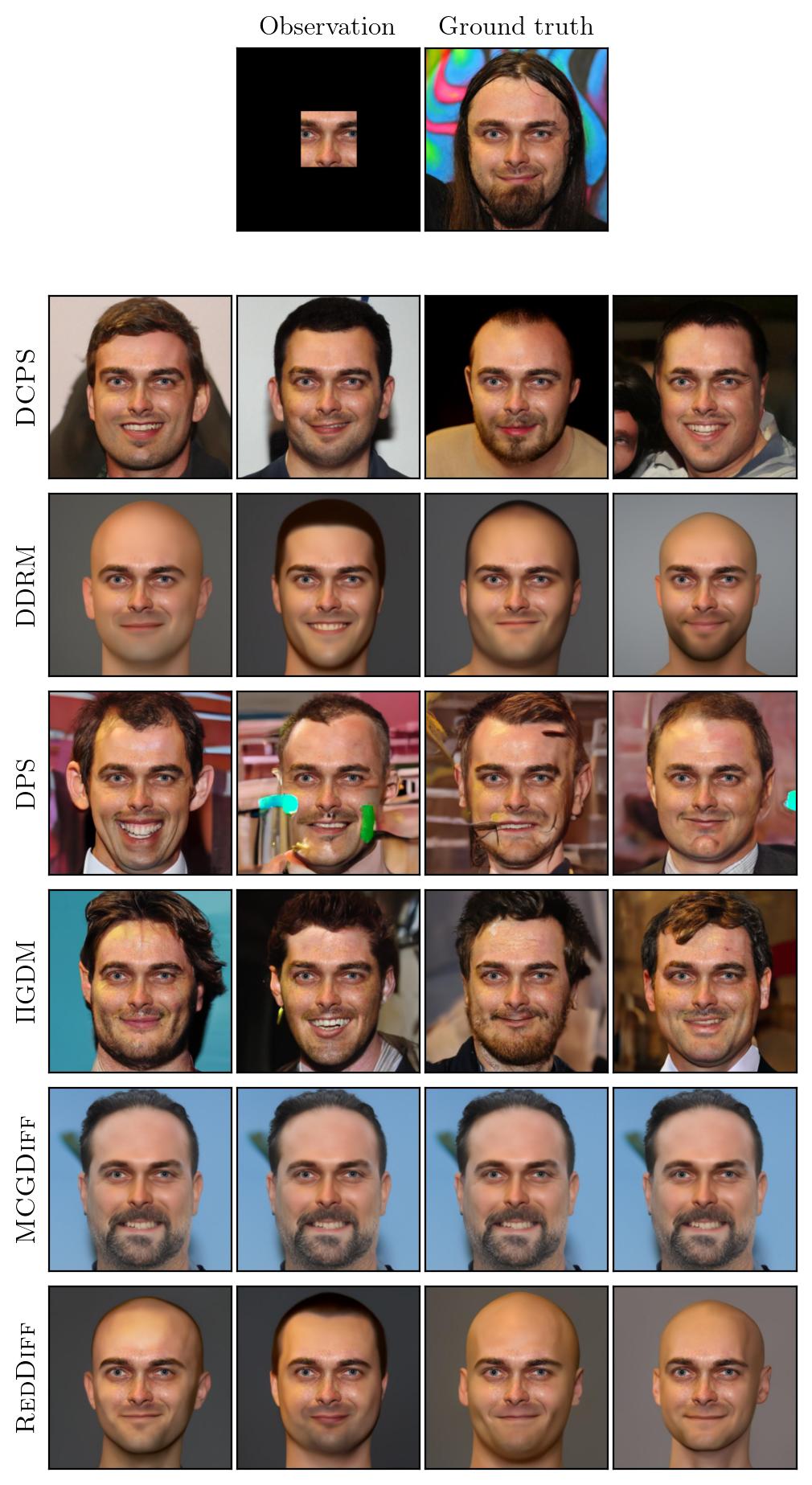}
    }
    \caption{Outpainting expend task on \ffhq.}
    \subfigure{
        \includegraphics[width=0.42\textwidth]{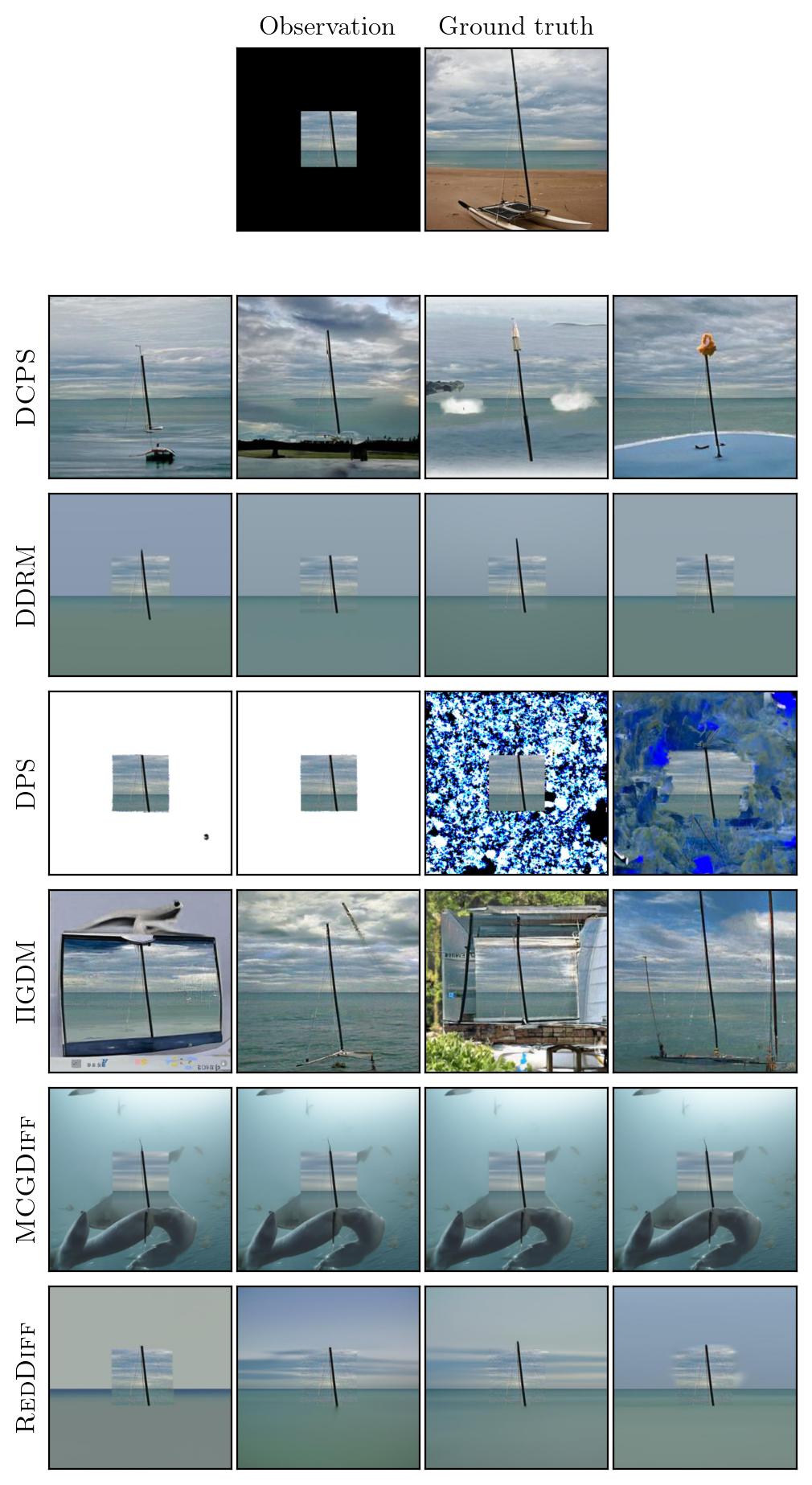}
        \includegraphics[width=0.42\textwidth]{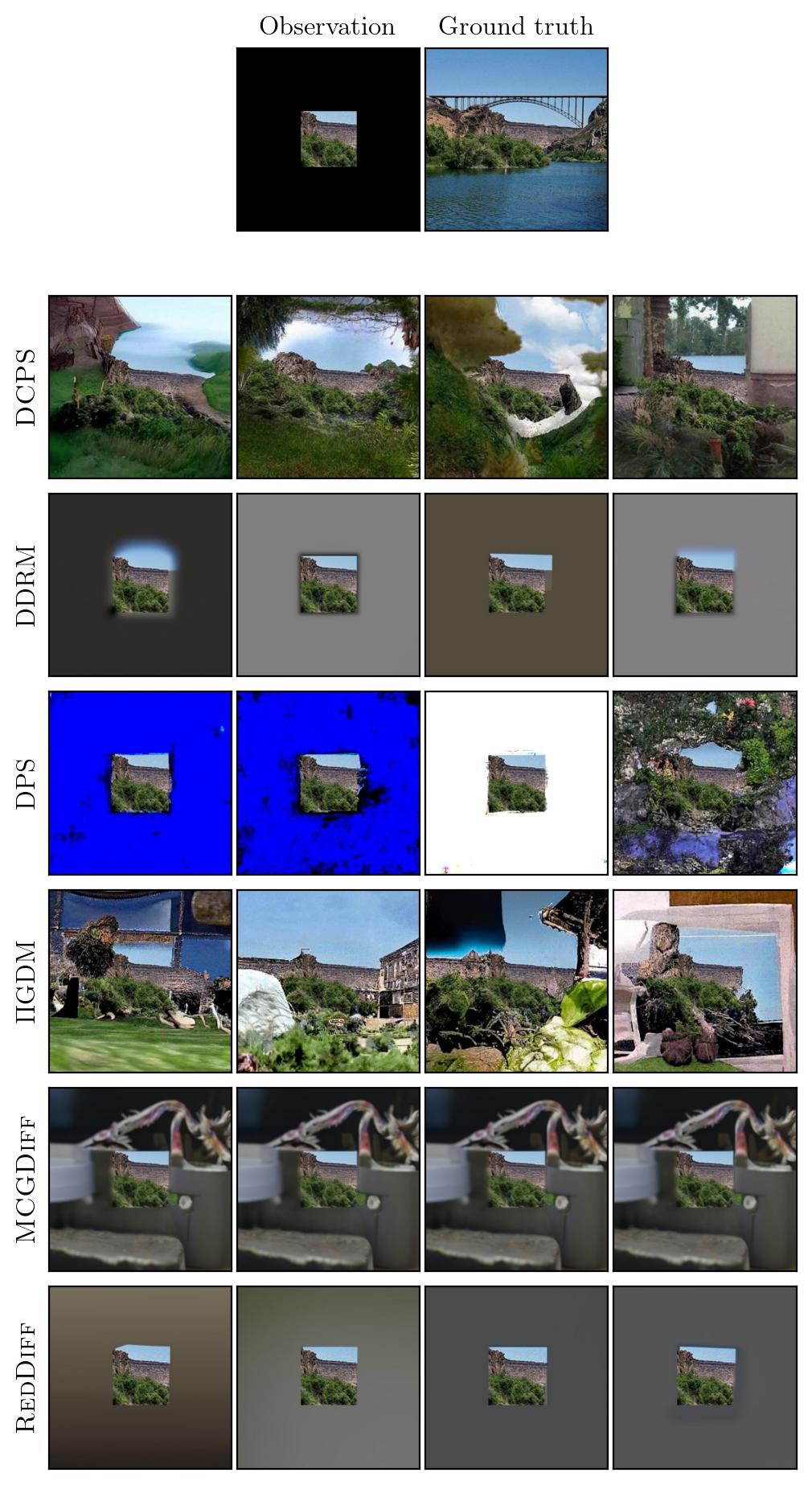}
    }
    \caption{Outpainting expend task on \imagenet.}
\end{figure}

\begin{figure}[htb]
    \centering
    \subfigure{
        \includegraphics[width=0.42\textwidth]{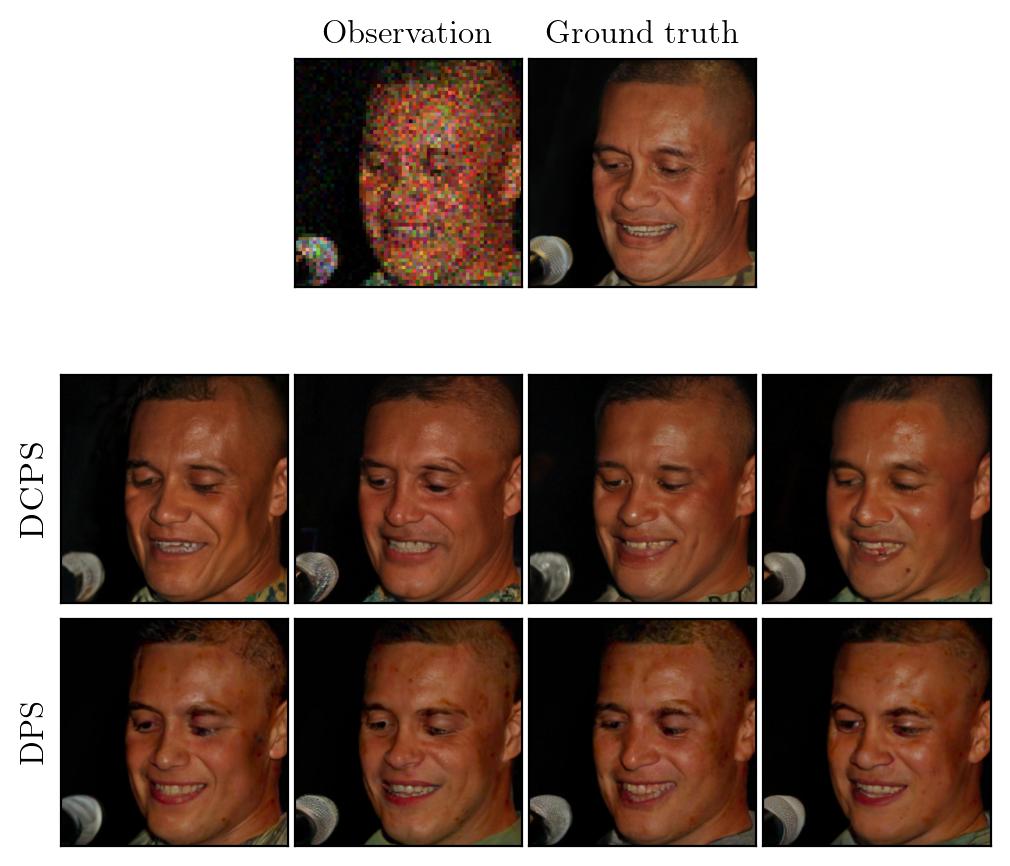}
        \includegraphics[width=0.42\textwidth]{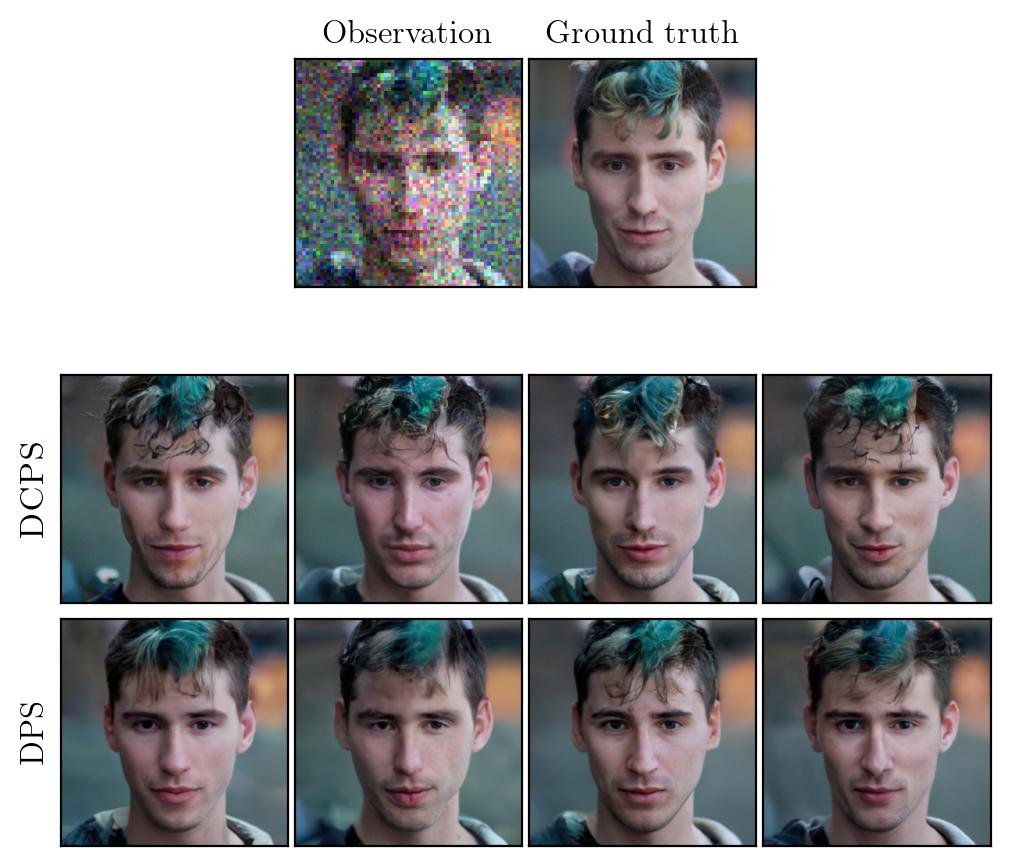}
    }
    \caption{SR $4 \times$ task with Poisson noise on \ffhq.}
    \subfigure{
        \includegraphics[width=0.42\textwidth]{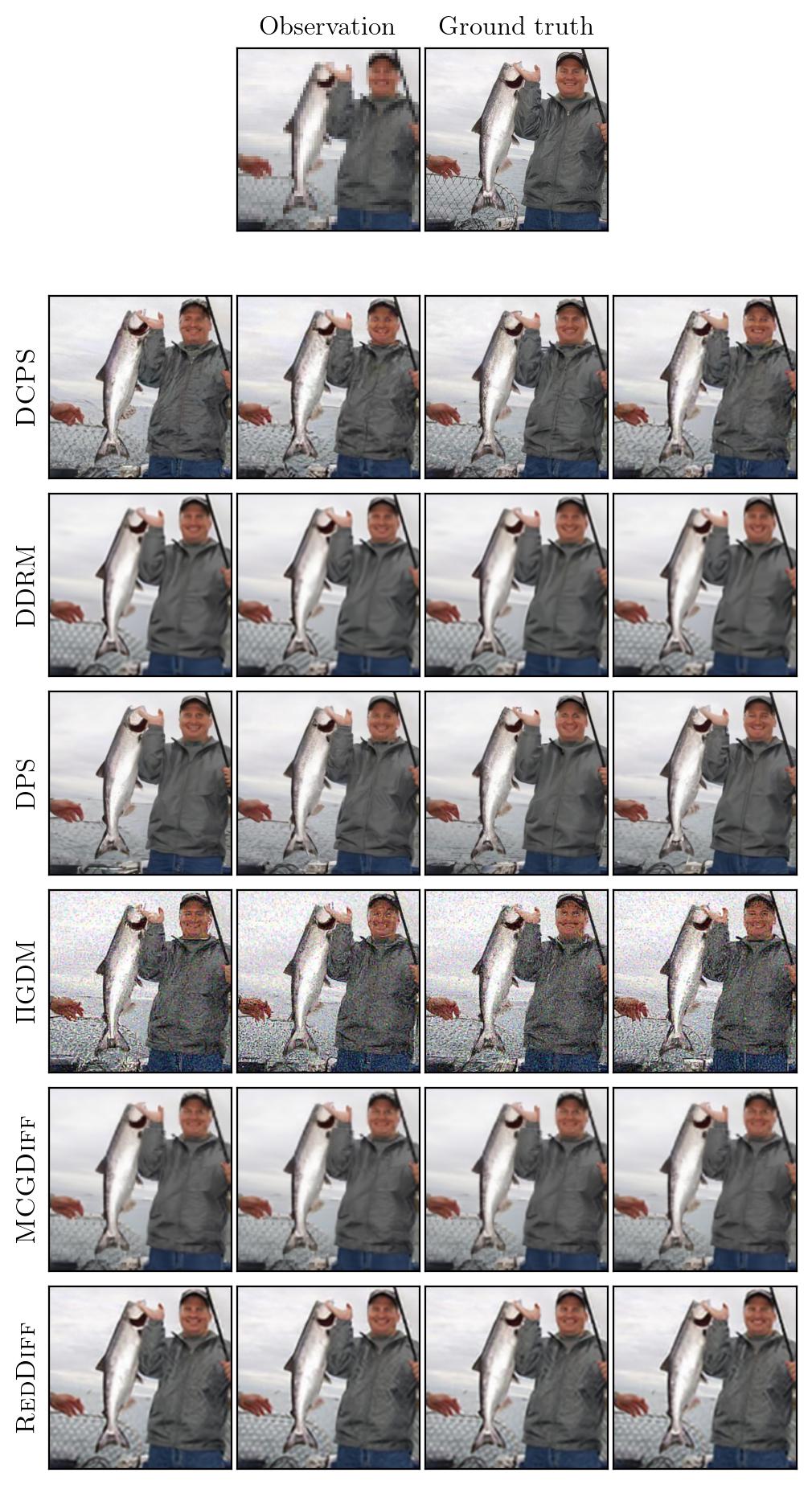}
        \includegraphics[width=0.42\textwidth]{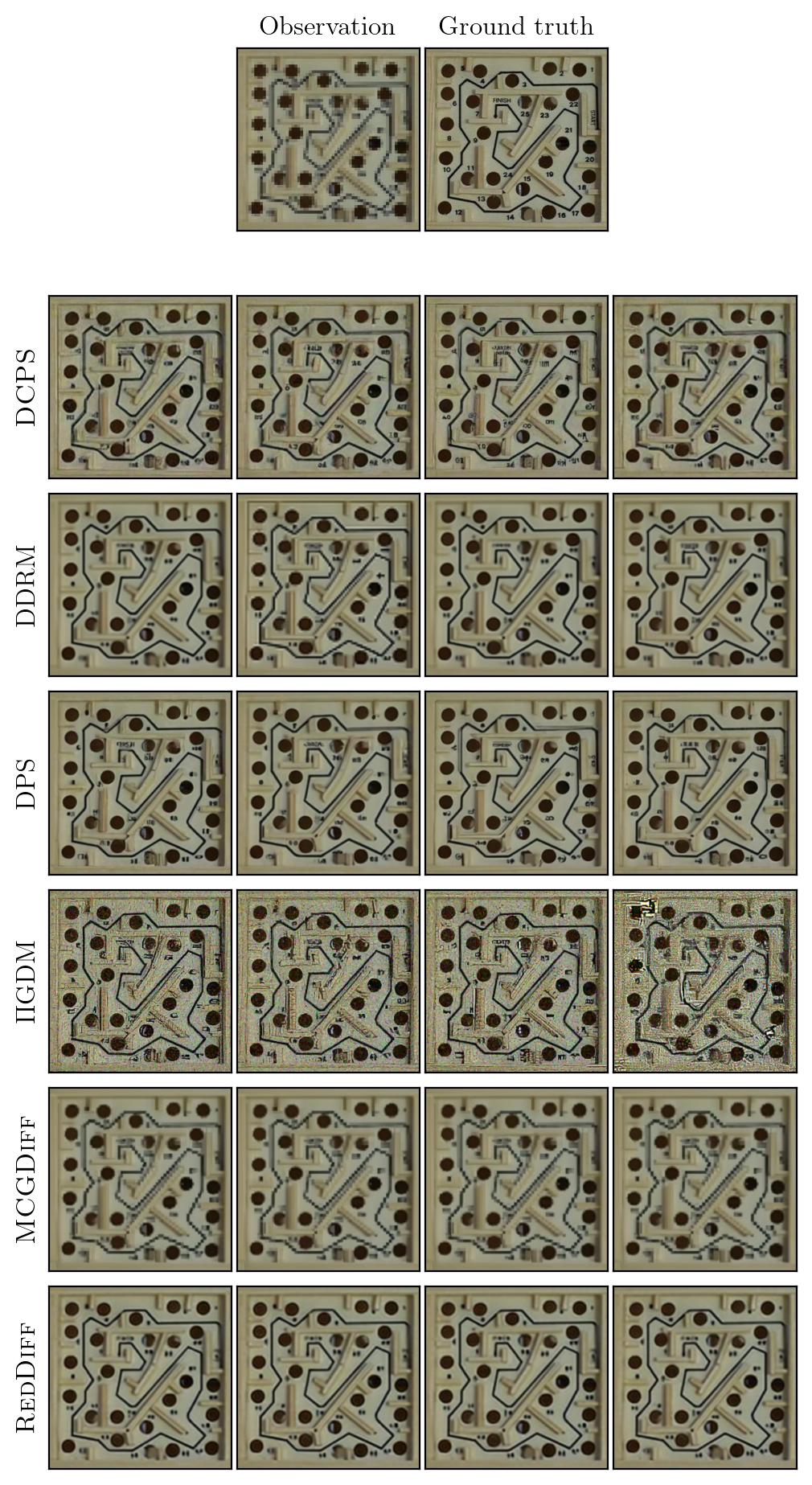}
    }
    \caption{SR $4\times$ task on \imagenet.}
\end{figure}

\begin{figure}[htb]
    \centering
    \subfigure{
        \includegraphics[width=0.42\textwidth]{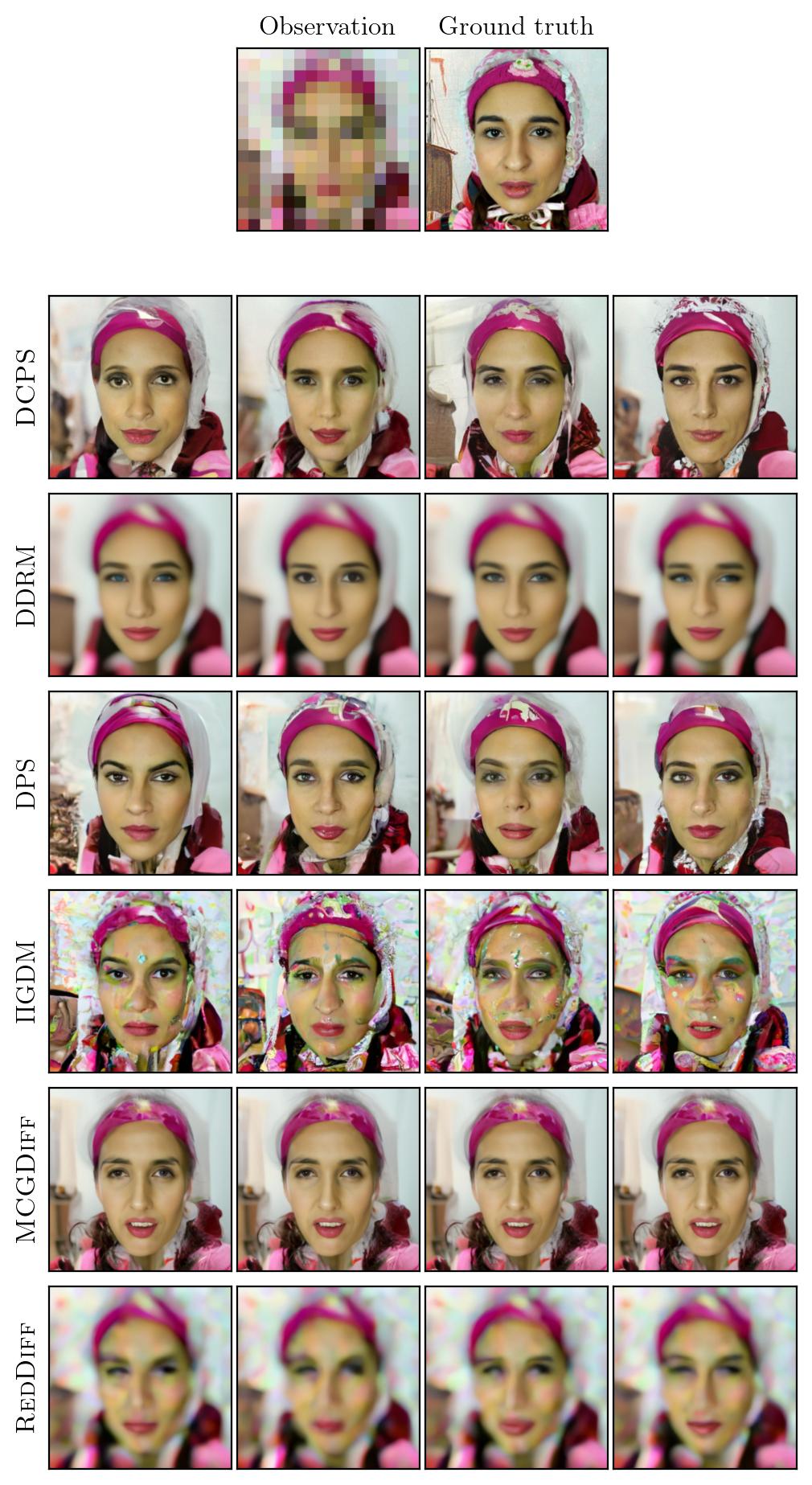}
        \includegraphics[width=0.42\textwidth]{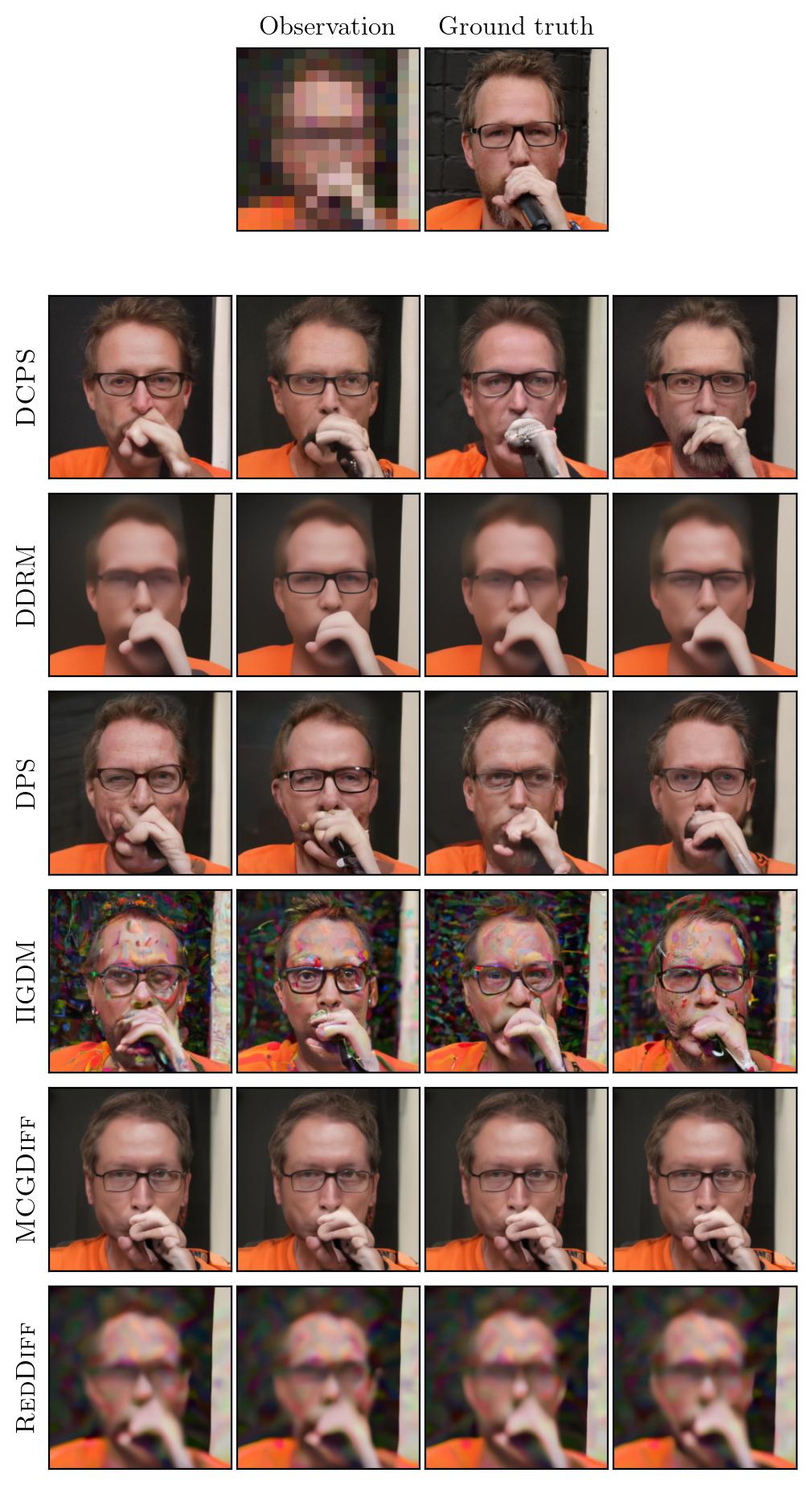}
    }
    \subfigure{
        \includegraphics[width=0.42\textwidth]{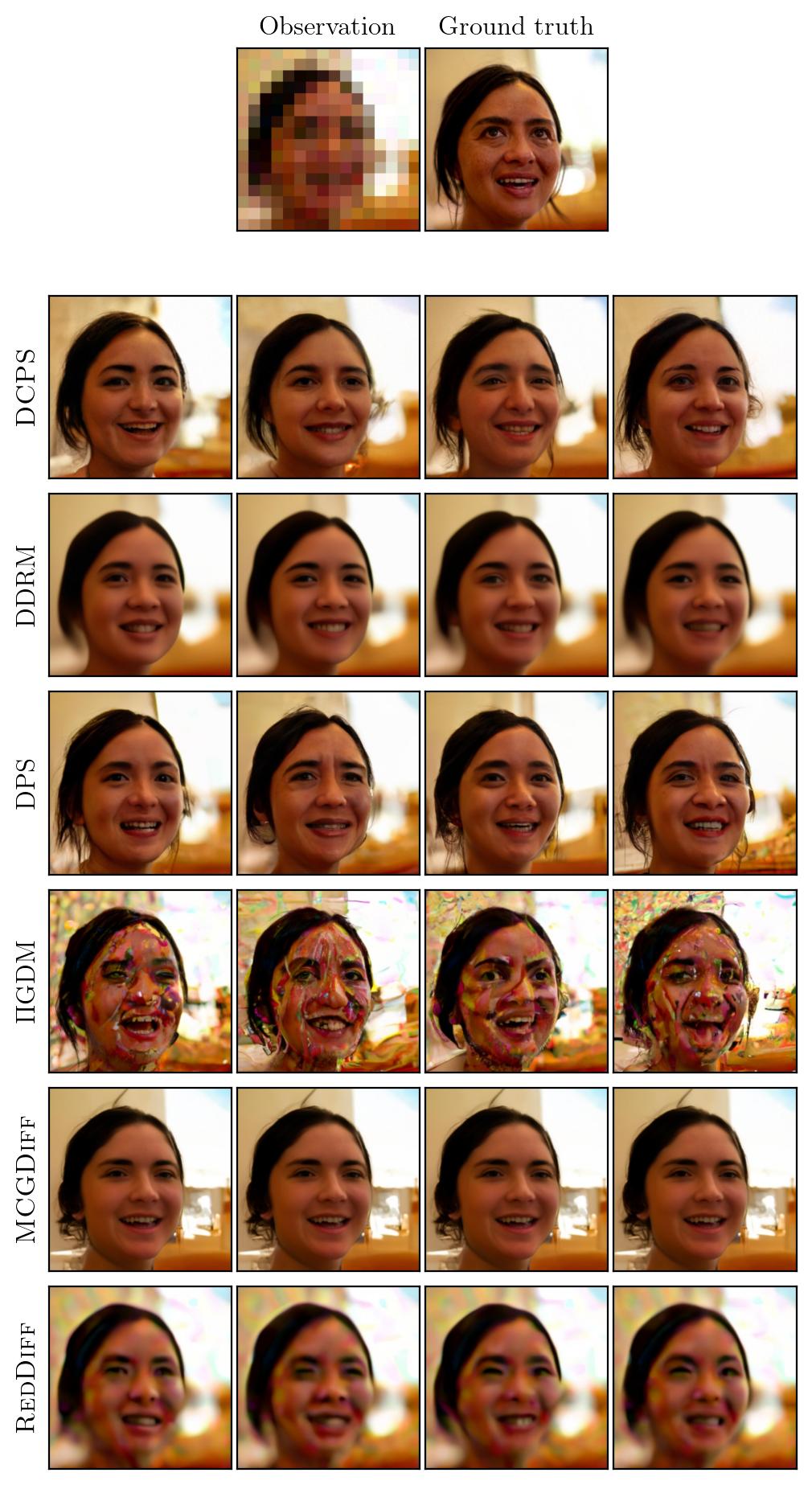}
        \includegraphics[width=0.42\textwidth]{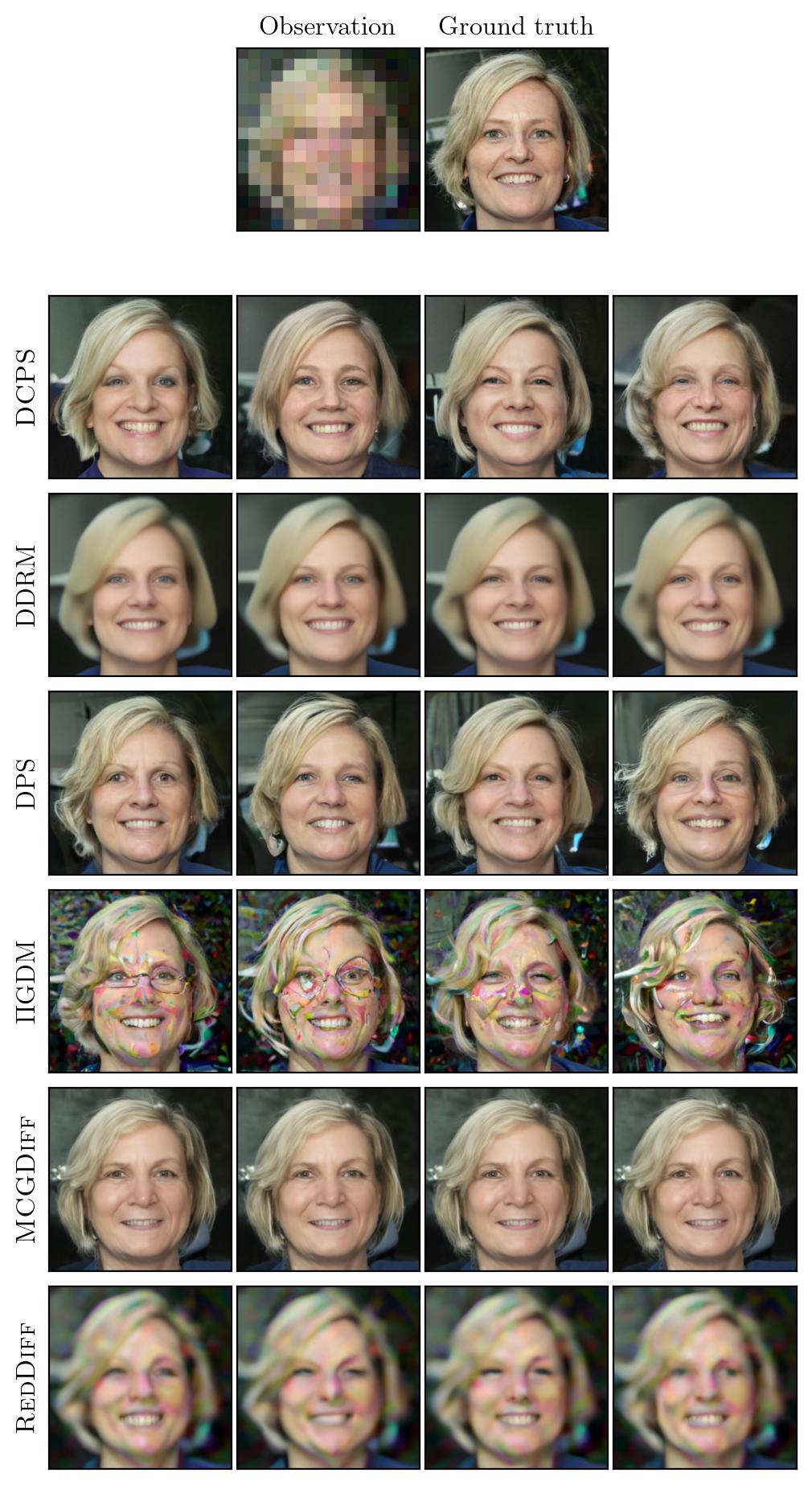}
    }
    \caption{SR $16 \times$ task on \ffhq.}
\end{figure}

\begin{figure}[htb]
    \centering
    \subfigure{
        \includegraphics[width=0.42\textwidth]{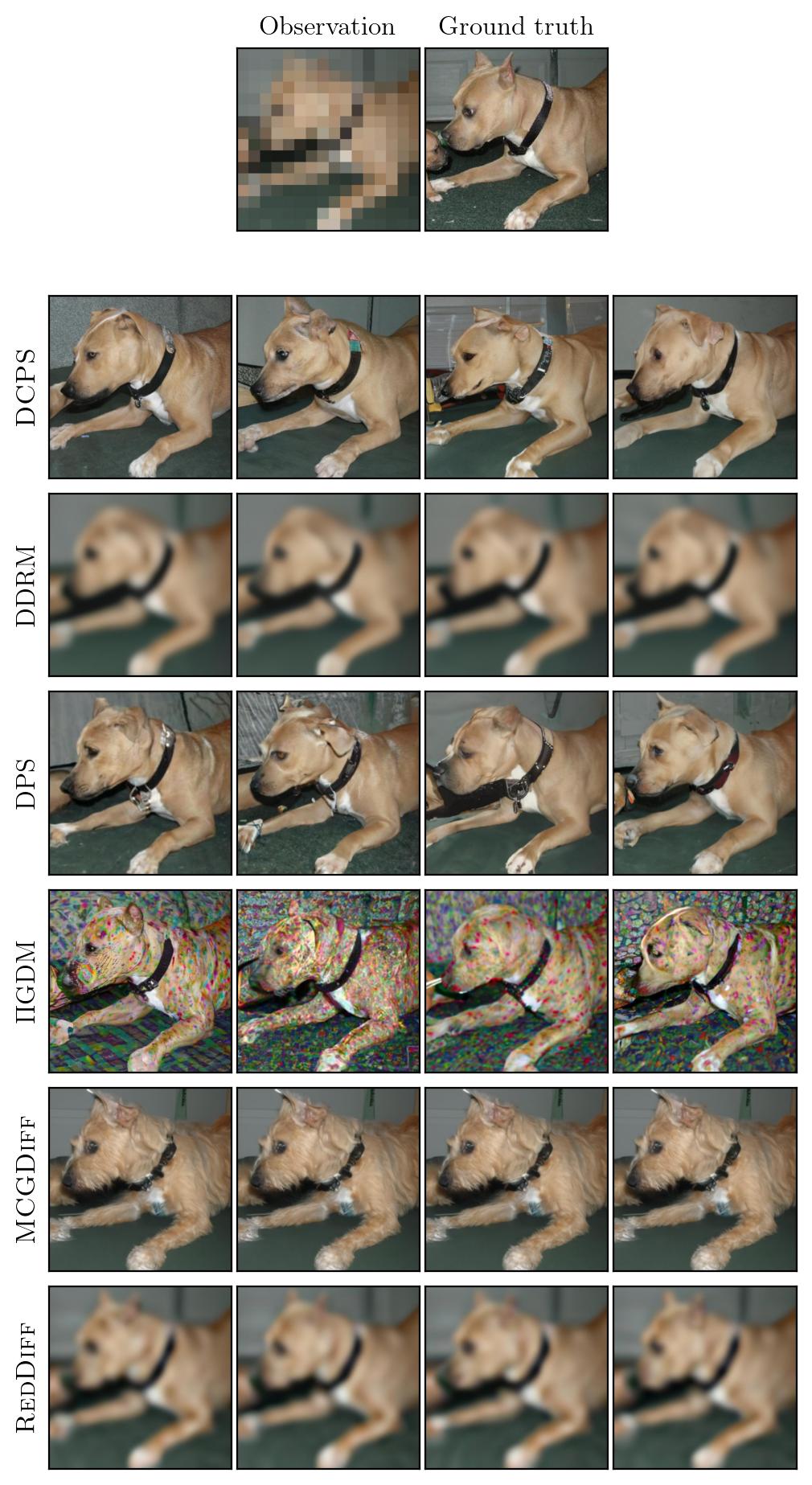}
        \includegraphics[width=0.42\textwidth]{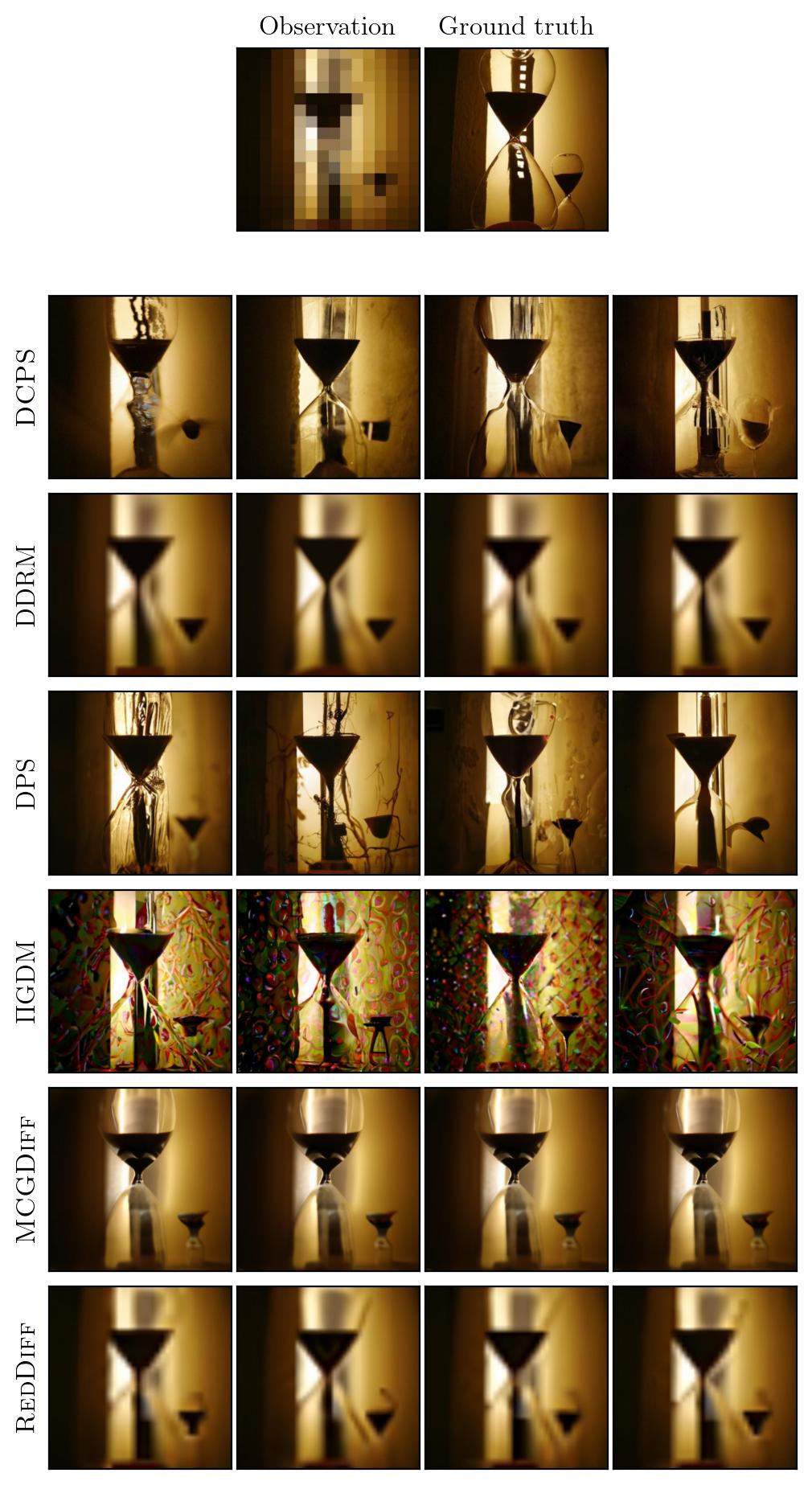}
    }
    \subfigure{
        \includegraphics[width=0.42\textwidth]{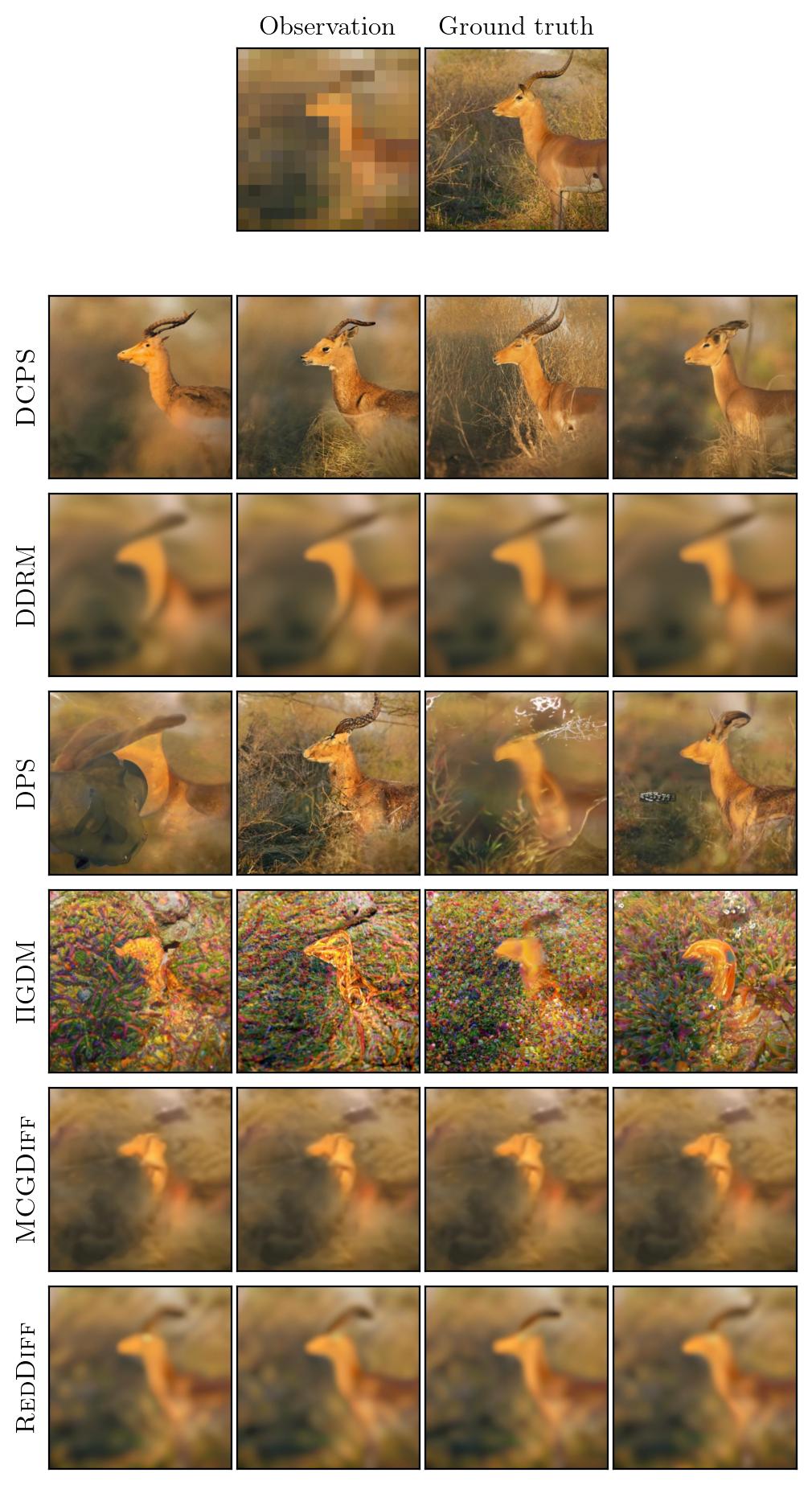}
    }
    \caption{SR $16 \times$ task on \imagenet.}
\end{figure}

\begin{figure}[htb]
    \centering
    \subfigure{
        \includegraphics[width=0.42\textwidth]{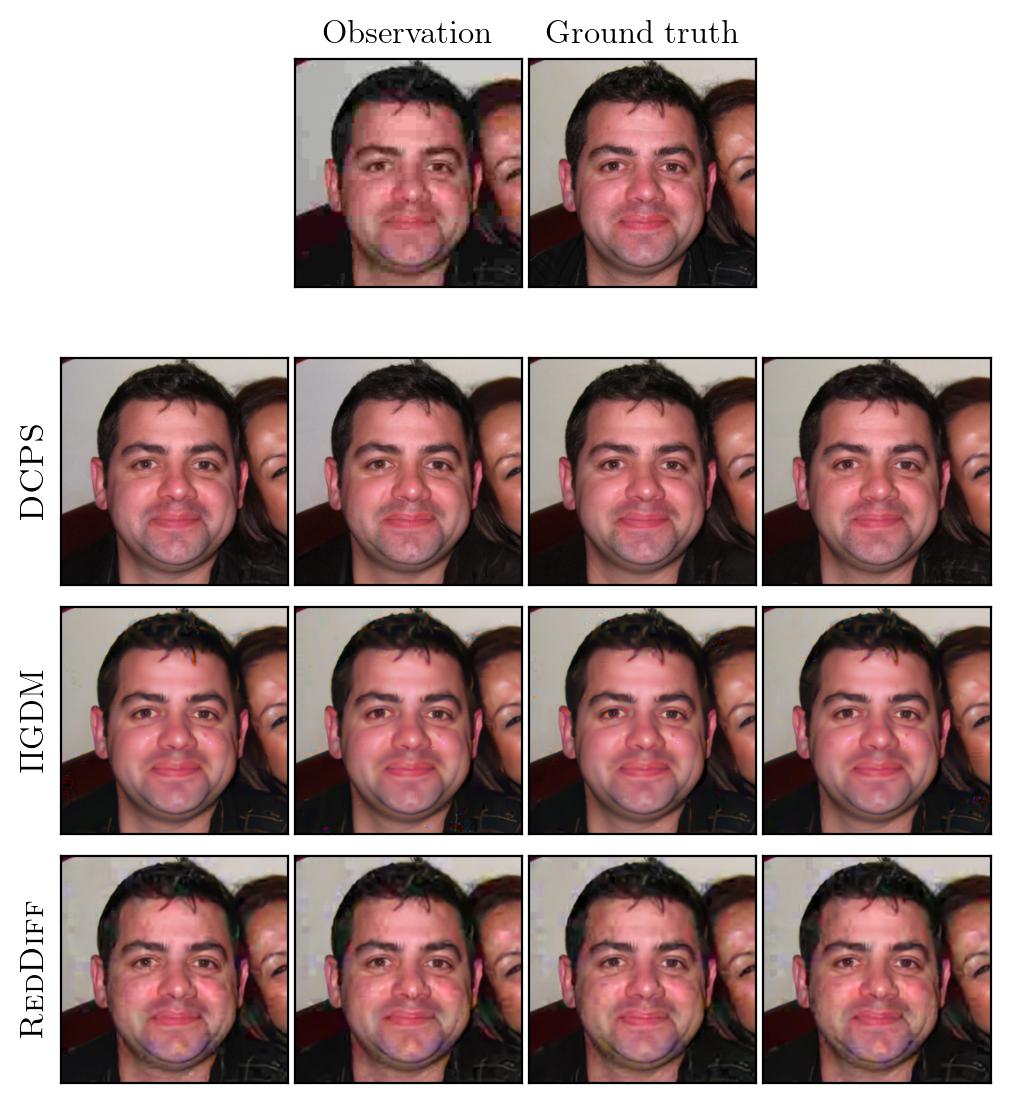}
        \includegraphics[width=0.42\textwidth]{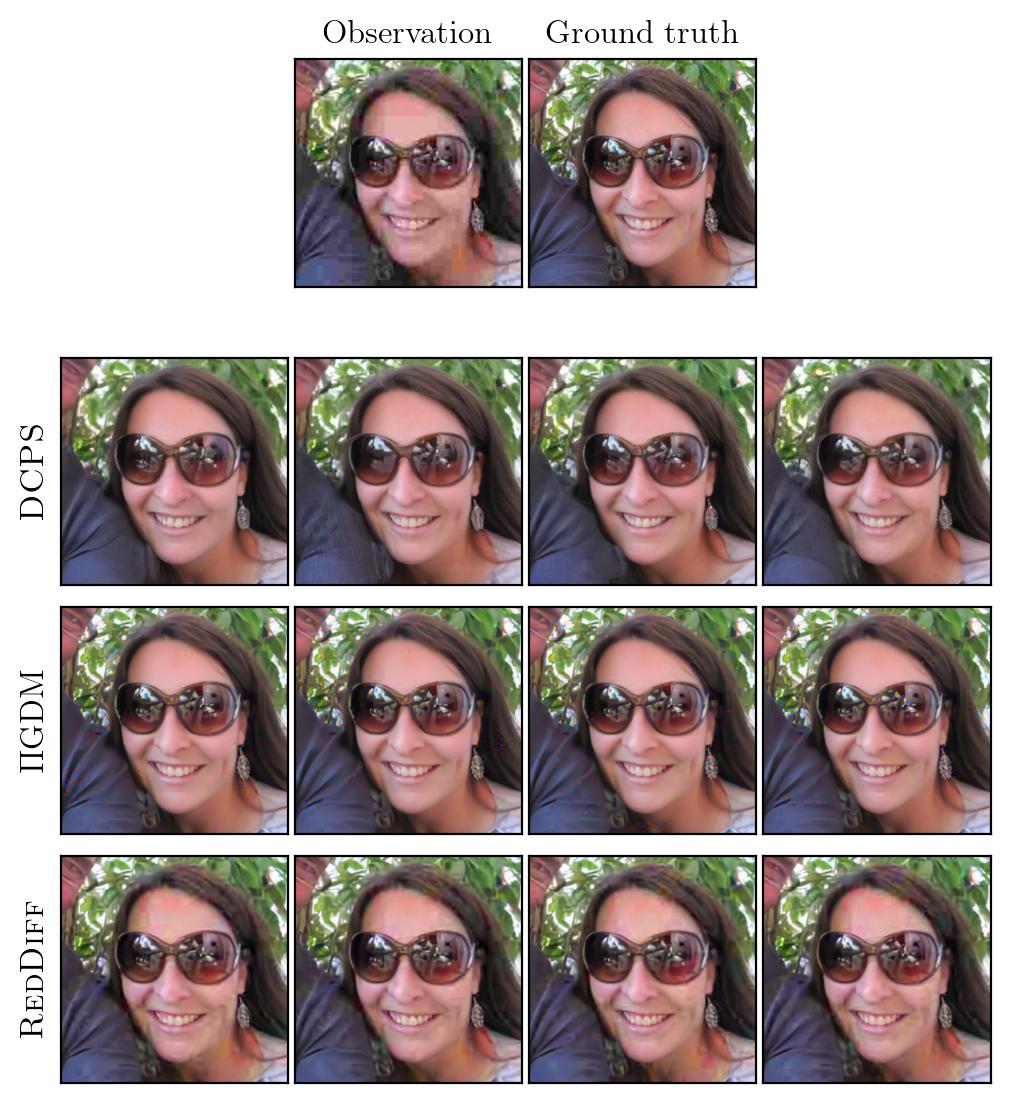}
    }
    \caption{JPEG task with QF=8 on \ffhq.}
    \subfigure{
        \includegraphics[width=0.42\textwidth]{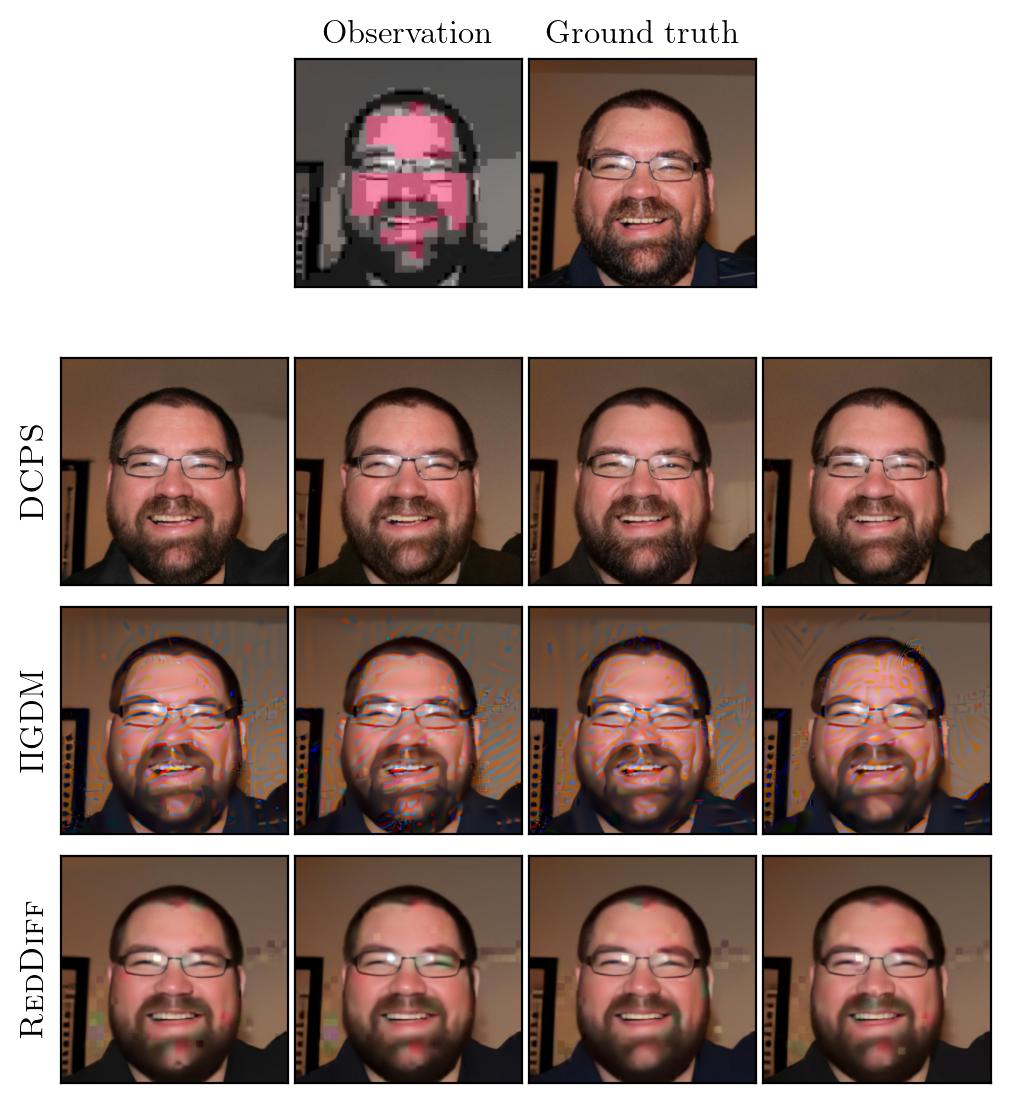}
        \includegraphics[width=0.42\textwidth]{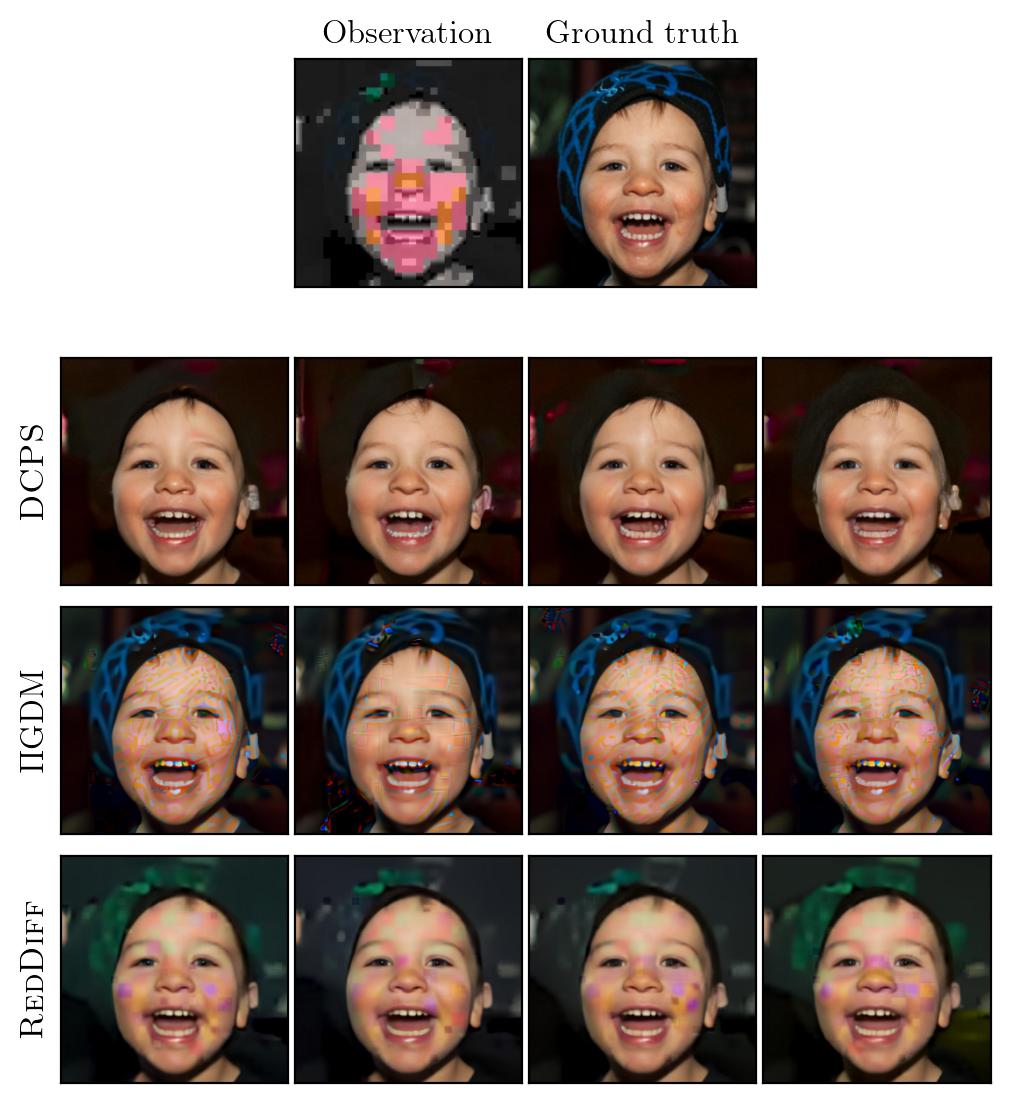}
    }
    \caption{JPEG task with QF=2 on \ffhq.}
\end{figure}

\begin{figure}[htb]
    \centering
    \subfigure{
        \includegraphics[width=0.42\textwidth]{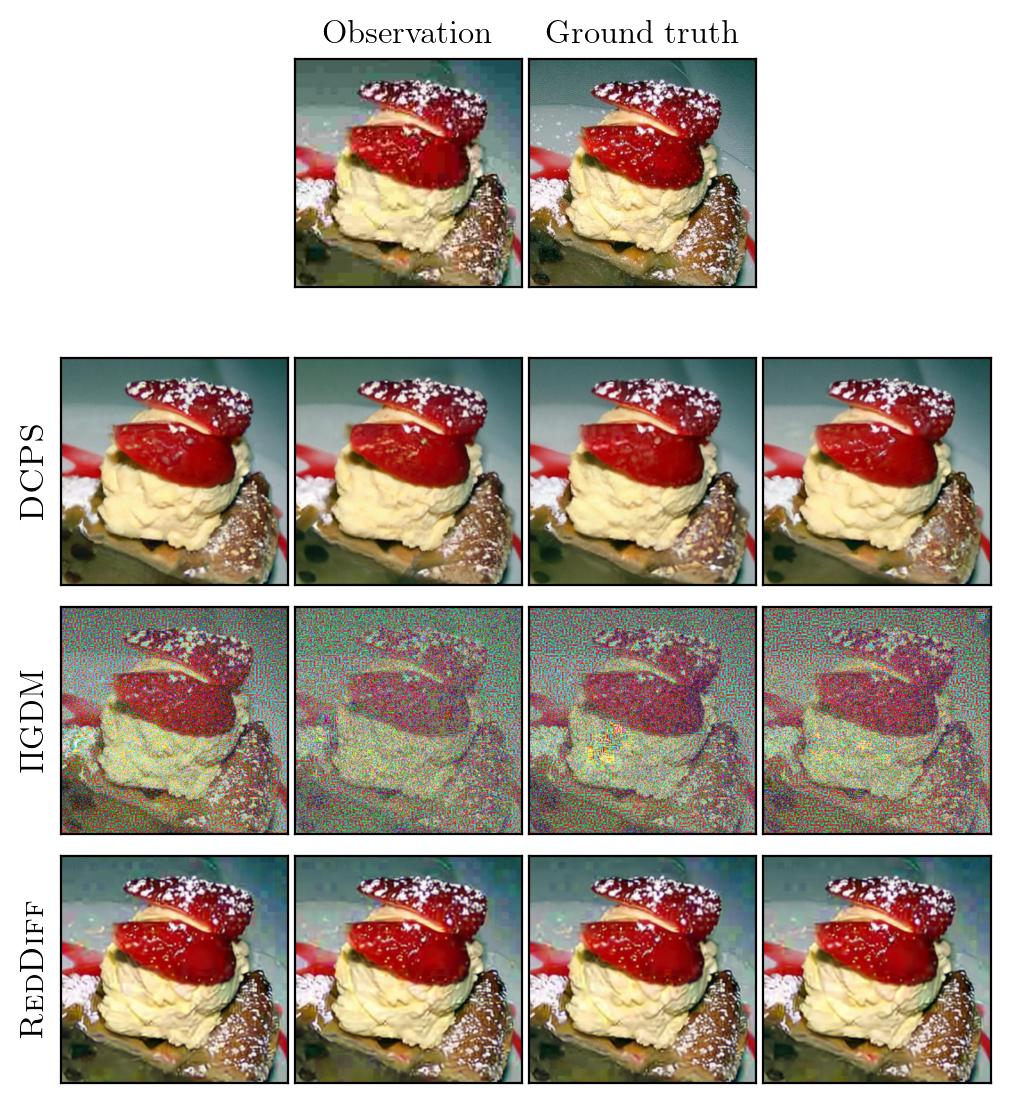}
        \includegraphics[width=0.42\textwidth]{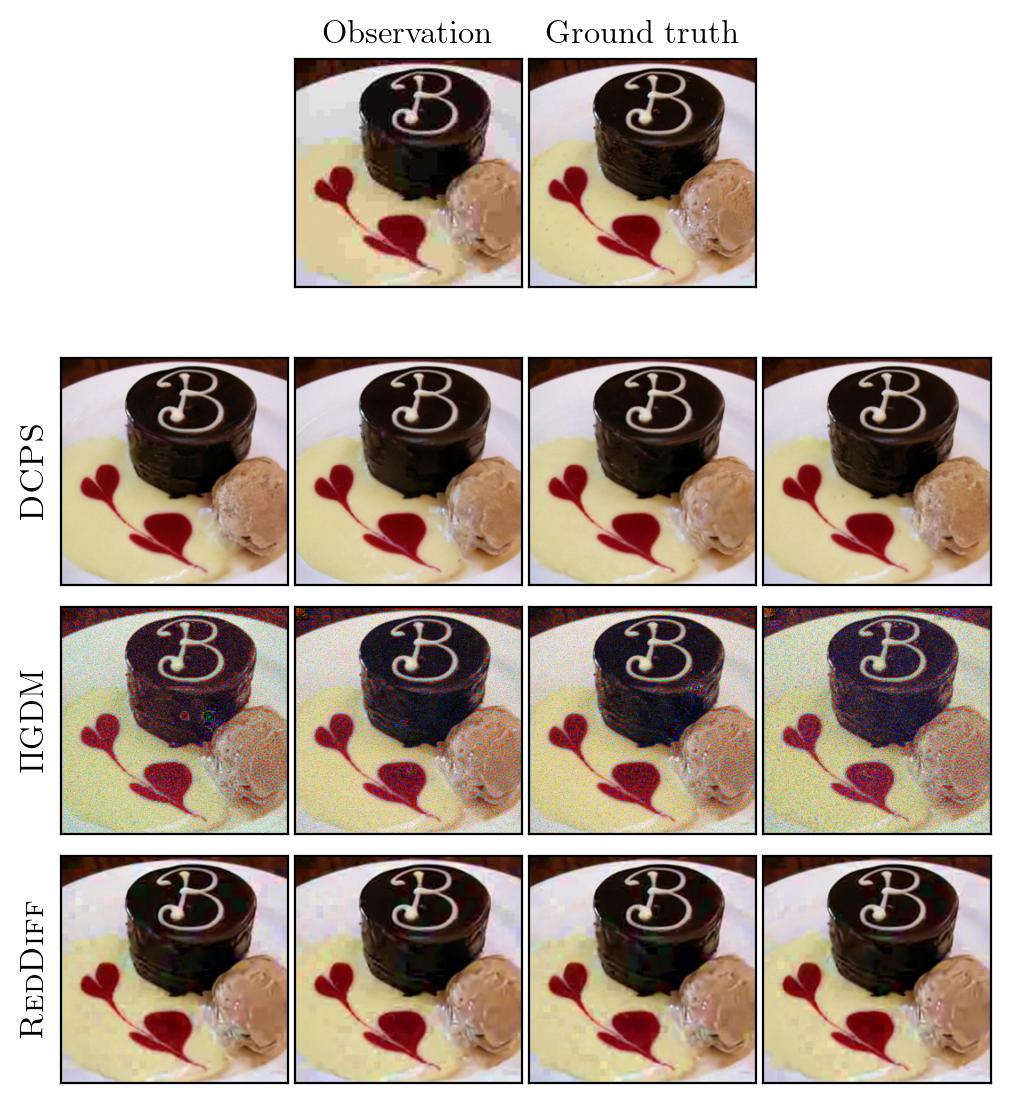}
    }
    \caption{JPEG task with QF=8 on \imagenet.}
    \subfigure{
        \includegraphics[width=0.42\textwidth]{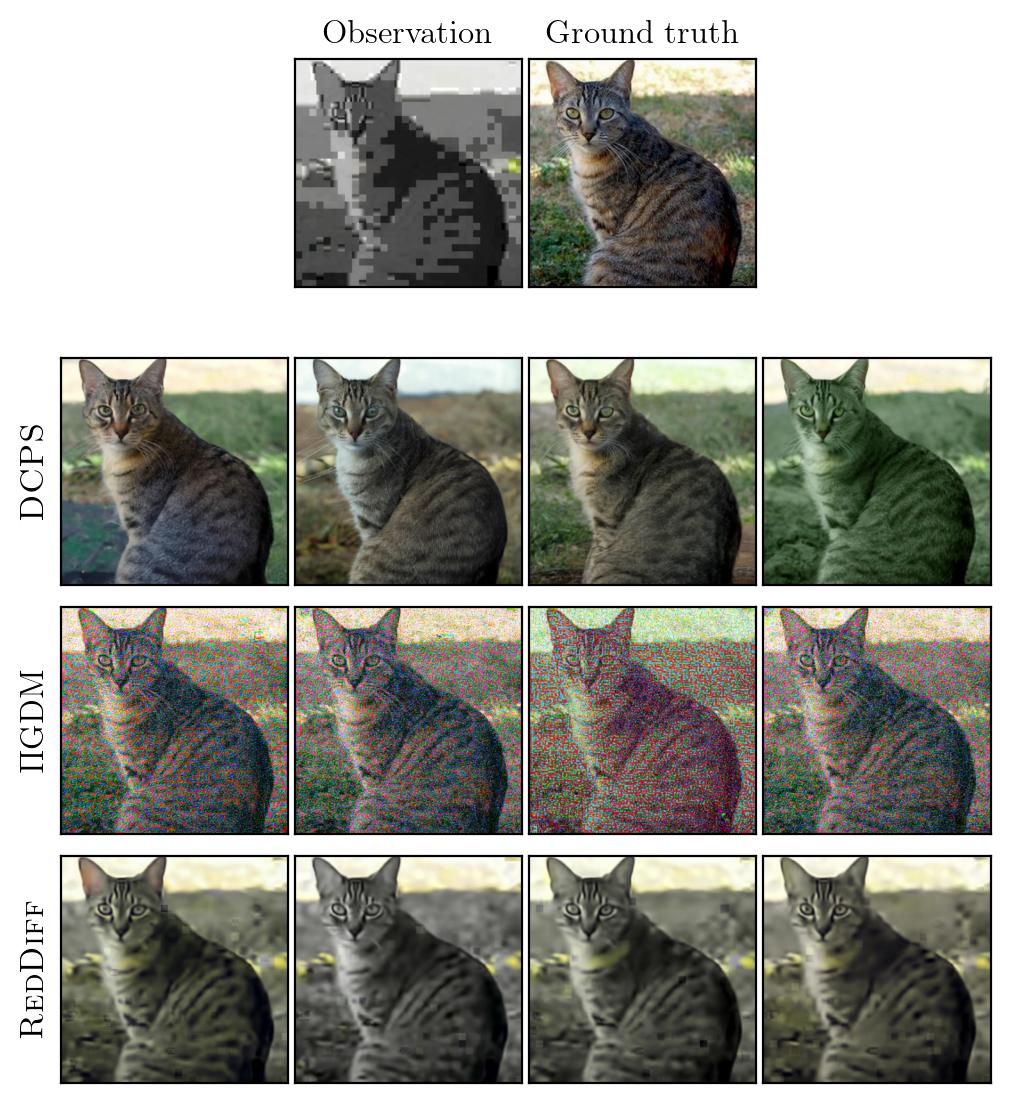}
        \includegraphics[width=0.42\textwidth]{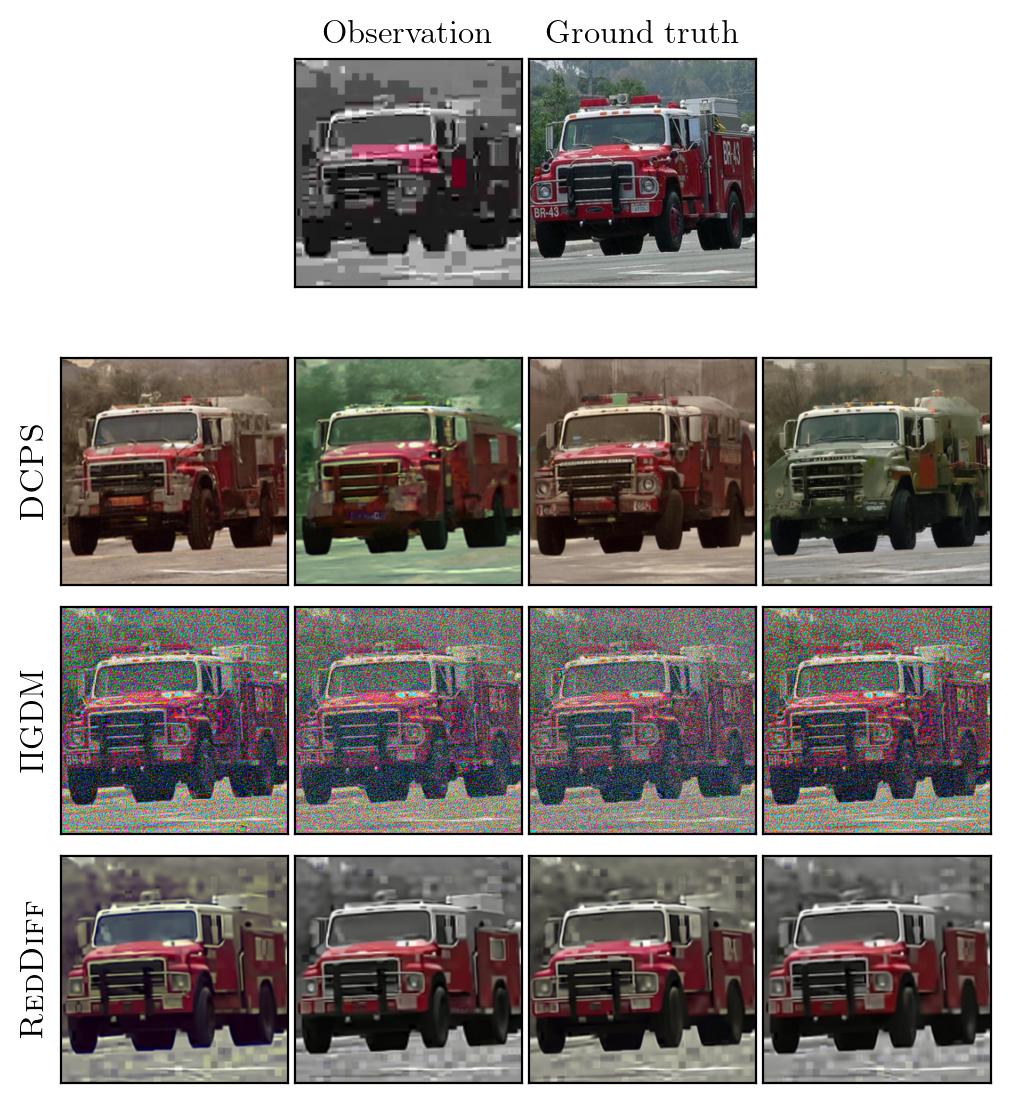}
    }
    \caption{JPEG task with QF=2 on \imagenet.}
\end{figure}

 \end{appendix}


 \end{document}